%% file: main.tex
\documentclass[journal]{IEEEtran} % Comment this line out

\usepackage{amsmath}  % assumes amsmath package installed
\usepackage{amssymb}  % assumes amsmath package installed
\usepackage{amsthm}
\usepackage{mathtools}

\usepackage{paralist}

\usepackage{verbatim}

\usepackage{color}
\usepackage{graphicx}
\usepackage{adjustbox}

\usepackage{mathptmx}
\usepackage{times}

\usepackage{multirow}

\usepackage{fancyhdr}

\usepackage[ruled,vlined,linesnumbered]{algorithm2e}
\usepackage[noend]{algorithmic}
\SetKwInput{KwInput}{Input}                % Set the Input
\SetKwInput{KwOutput}{Output}              % set the Output

\usepackage[us]{datetime}
\usepackage{caption}
\usepackage{subfloat}

\usepackage[color]{changebar_fix}
\usepackage[normalem]{ulem}

\newcommand{\hrulealg}[0]{\vspace{1mm} \hrule \vspace{1mm}}

\usepackage[usenames,dvipsnames,table]{xcolor}
\usepackage{colortbl}

\newcommand{\Hl}[2][\empty]{%
\ifx#1\empty
\else
\sethlcolor{#1}%
\fi
\hl{#2}}
\usepackage{soul}
\soulregister\Hl{7}
\soulregister\ref7
\soulregister\autoref7
\soulregister\eqref7
\soulregister\cite7
\soulregister\pageref7
\soulregister\ac7
\soulregister\acp7
\soulregister\acs7
\soulregister\acl7
\soulregister\subsection7
\soulregister\SI7
\soulregister\url7

\usepackage{mdframed}

\newmdenv[
   backgroundcolor=yellow,
   linewidth=0pt,
   innertopmargin=1pt,
   innerbottommargin=1pt
]{highlighted}

\usepackage{hyperref} %pdf with links and toc on the left
\hypersetup{
  colorlinks,
  citecolor=black,
  filecolor=black,
  linkcolor=black,
  urlcolor=black,
  pdfauthor={},
  pdfsubject={},
  pdftitle={}
}

\newcommand{\PREPRINTYEAR}{2025}

\newcommand{\DOI}{10.1109/TRO.2025.3547296} % you will not get a DOI until the paper is actually published, so update this when you get it and reupload the new preprint to all systems

\usepackage[placement=top,vshift=-2,firstpage=true]{background}
\SetBgScale{1.0}
\SetBgContents{\parbox{0.95\textwidth}{\small \begin{center} The manuscript will appear in IEEE Transactions on Robotics under DOI: \DOI\end{center} \vspace{-0.2cm}  \copyright{}~\PREPRINTYEAR~IEEE.~Personal use of this material is permitted. Permission from IEEE must be obtained for all other uses, in any current or future media, including reprinting/republishing this material for advertising or promotional purposes, creating new collective works, for resale or redistribution to servers or lists, or reuse of any copyrighted component of this work in other works.} }
\SetBgColor{black}
\SetBgAngle{0}
\SetBgOpacity{1.0}

\newcommand{\aref}[1]{\hyperref[#1]{Appendix~\ref*{#1}}}

\usepackage{todonotes}

\usepackage{subcaption}

\usepackage{latexsym}
\usepackage{color}

\usepackage{cite}

\usepackage[inline]{enumitem} % inline enumerate
\setenumerate[0]{label=\roman*.}

\usepackage{siunitx}

\usepackage{ifthen}

\usepackage{booktabs}
\usepackage{tablefootnote}

\usepackage{subfiles}
\usepackage{bm}

\setlist[itemize]{noitemsep, nosep}

\usepackage{pifont}
\definecolor{color_green}{rgb}{0, .722, .243}
\definecolor{color_red}{rgb}{0.737,0.165,0}

\usepackage{stackengine} 

\usepackage{tikz}
\usetikzlibrary{calc}
\usepackage{tikz-dimline}
\usepackage{pgfplotstable}
\usepgfplotslibrary{statistics}
\usepgfplotslibrary{fillbetween}
\pgfplotsset{compat=1.8}

\colorlet{veccol}{green!45!black}
\colorlet{myred}{red!80!black}
\colorlet{myblue}{blue!80!black}
\colorlet{mypurple}{blue!50!red!100!}
\colorlet{amber}{red!90!yellow!100!blue!50}
\colorlet{mygreen}{green!40!black}
\tikzstyle{vector}=[->,very thick,veccol]
\usetikzlibrary{shapes.geometric,arrows.meta,angles,quotes}
\tikzstyle{thin arrow}=[dashed,thin,-{Latex[length=4,width=3]}]
\tikzstyle{arrow}=[{-Latex}]

\definecolor{color_blue}{rgb}{0.22, 0.2, 0.502}
\definecolor{color_red}{rgb}{0.737,0.165,0}
\definecolor{color_green}{rgb}{0, .522, .243}

\usepgfplotslibrary{fillbetween}
\usepgfplotslibrary{groupplots}
\def\centerarc[#1](#2)(#3:#4:#5)% Syntax: [draw options] (center) (initial angle:final angle:radius)
    { \draw[#1] ($(#2)+({#5*cos(#3)},{#5*sin(#3)})$) arc (#3:#4:#5); }

\pgfmathdeclarefunction{fpumod}{2}{%
        \pgfmathfloatdivide{#1}{#2}%
        \pgfmathfloatint{\pgfmathresult}%
        \pgfmathfloatmultiply{\pgfmathresult}{#2}%
        \pgfmathfloatsubtract{#1}{\pgfmathresult}%
        \pgfmathfloatifapproxequalrel{\pgfmathresult}{#2}{\def\pgfmathresult{5}}{}%
    }

\usepackage{tikzviolinplots}

\pgfplotsset{boxplot legend/.style={
    legend image code/.code={
        \draw[#1] (0cm,-0.1cm) rectangle (0.4cm,0.1cm)
        (0.2cm,-0.1cm) -- (0.2cm,-0.2cm) (0.05cm,-0.2cm) -- (0.35cm,-0.2cm)
        (0.2cm,0.1cm) -- (0.2cm,0.2cm) (0.05cm,0.2cm) -- (0.35cm,0.2cm);
     \path (0cm,0.24cm) (0cm,-0.24cm);
    },
}}

\tikzset{
    scale circle/.code={
        \pgfgetlastxy{\x@coord}{\y@coord}
        \pgfmathsetmacro\xscale{\pgfkeysvalueof{/tikz/x radius}*1/cos(\y@coord/\pgfplotsunitylength)}
        \tikzset{/tikz/x radius=\xscale}
    },
    scale bar/.code={
        \tikzset{
            insert path={
                node [
                    below,
                    font=\scriptsize
                ] {0} 
                +(0,-2pt) -- +(0,0) -- ++(20pt,0) 
                node [
                    below,
                    font=\scriptsize,
                    inner xsep=1pt,
                    label={[inner sep=0pt,font=\scriptsize]right:m}] {#1}
                -- +(0,-2pt)
            }
        }
    },
    scale bar/.default=10
}

\usepackage{scalerel}
\usetikzlibrary{svg.path}
\definecolor{orcidlogocol}{HTML}{A6CE39}
\tikzset{
  orcidlogo/.pic={
    \fill[orcidlogocol] svg{M256,128c0,70.7-57.3,128-128,128C57.3,256,0,198.7,0,128C0,57.3,57.3,0,128,0C198.7,0,256,57.3,256,128z};
    \fill[white] svg{M86.3,186.2H70.9V79.1h15.4v48.4V186.2z}
    svg{M108.9,79.1h41.6c39.6,0,57,28.3,57,53.6c0,27.5-21.5,53.6-56.8,53.6h-41.8V79.1z M124.3,172.4h24.5c34.9,0,42.9-26.5,42.9-39.7c0-21.5-13.7-39.7-43.7-39.7h-23.7V172.4z}
    svg{M88.7,56.8c0,5.5-4.5,10.1-10.1,10.1c-5.6,0-10.1-4.6-10.1-10.1c0-5.6,4.5-10.1,10.1-10.1C84.2,46.7,88.7,51.3,88.7,56.8z};
  }
}
\newcommand\orcidicon[1]{\href{https://orcid.org/#1}{\mbox{\scalerel*{
        \begin{tikzpicture}[yscale=-1,transform shape]
          \pic{orcidlogo};
        \end{tikzpicture}
}{|}}}}

\SetArgSty{textnormal}

\newtheoremstyle{wobrackets}%
  {3pt}% (space above)
  {3pt}% (space below)
  {\itshape}% (body font)
  {}% (indent amount)
  {\bfseries}% {theorem head font}
  {}% {punctuation after theorem head}
  {.3em}% {space after theorem head}
  {\thmname{#1} \thmnumber{#2}: \thmnote{\normalfont#3}}% {theorem head spec}

\theoremstyle{wobrackets}
\newtheorem{theoremwobrackets}{Theorem}

\newtheoremstyle{remark}%
  {3pt}% (space above)
  {3pt}% (space below)
  {\itshape}% (body font)
  {}% (indent amount)
  {\bfseries}% {theorem head font}
  {}% {punctuation after theorem head}
  {.2em}% {space after theorem head}
  {\thmname{#1}: \thmnote{\normalfont#3}}% {theorem head spec}

\theoremstyle{remark}
\newtheorem*{remark}{Remark}

\DeclareMathOperator*{\minimize}{minimize}

\title{%
\Title
}

\usepackage[printonlyused]{acronym}
\acrodef{uav}[UAV]{Unmanned Aerial Vehicle}
\acrodef{lsap}[LSAP]{Linear Sum Assignment Problem}
\acrodef{lbap}[LBAP]{Linear Bottleneck Assignment Problem}
\acrodef{lifo}[LIFO]{Last In First Out}
\acrodef{cbaa}[CBAA]{Consensus-Based Auction Algorithm}
\acrodef{cbba}[CBBA]{Consensus-Based Bundle Algorithm}
\acrodef{tofrp}[TOFREP]{Time-Optimal Formation Reshaping Problem}
\acrodef{catora}[CAT-ORA]{Collision-Aware Time-Optimal formation Reshaping Algorithm}

\newcommand{\Title}{CAT-ORA: Collision-Aware Time-Optimal Formation Reshaping for Efficient Robot Coordination in 3D Environments}

\author{
  Vit Kratky$^{*\orcidicon{0000-0002-1914-742X}}$,
  Robert Penicka$^{\orcidicon{0000-0001-8549-4932}}$,
  Jiri Horyna$^{\orcidicon{0000-0001-6614-0928}}$,
  Petr Stibinger$^{\orcidicon{0000-0002-7662-9230}}$,
Tomas Baca$^{\orcidicon{0000-0001-9649-8277}}$,
  Matej Petrlik$^{\orcidicon{0000-0002-5337-9558}}$,
  Petr Stepan$^{\orcidicon{0000-0002-7444-3264}}$,
  Martin Saska$^{\orcidicon{0000-0001-7106-3816}}$
\thanks{Authors are with the Department of Cybernetics, Faculty of Electrical Engineering, Czech Technical University in Prague, Technicka 2, Prague 6, Czech Republic.\\
{$^*$Corresponding author, {\tt\scriptsize\{\href{mailto:vit.kratky@fel.cvut.cz}{vit.kratky}|\href{mailto:penicrob@fel.cvut.cz}{penicrob}|\href{mailto:horynjir@fel.cvut.cz}{horynjir}|\href{mailto:stibipet@fel.cvut.cz}{stibipet}

|\href{mailto:bacatoma@fel.cvut.cz}{bacatoma}|\href{mailto:matej.petrlik@fel.cvut.cz}{matej.petrlik}|\href{mailto:stepan@fel.cvut.cz}{stepan}|\href{mailto:martin.saska@fel.cvut.cz}{martin.saska}\}@fel.cvut.cz}}%

{This work was partially funded by the CTU grant no. SGS23/177/OHK3/3T/13, by the Czech Science Foundation (GAČR) grant no. 23-06162M, and by the European Union under the project Robotics and advanced industrial production (reg. no. CZ.02.01.01/00/22\_008/0004590).}
}
}
\begin{document}

\maketitle

% %%{ Abstract
\begin{abstract}
In this paper, we introduce an algorithm designed to address the problem of time-optimal formation reshaping in three-dimensional environments while preventing collisions between agents. 
The utility of the proposed approach is particularly evident in mobile robotics, where agents benefit from being organized and navigated in formation for a variety of real-world applications requiring frequent alterations in formation shape for efficient navigation or task completion.
  Given the constrained operational time inherent to battery-powered mobile robots, the time needed to complete the formation reshaping process is crucial for their efficient operation, especially in case of multi-rotor \acp{uav}.
  The proposed \ac{catora} builds upon the Hungarian algorithm for the solution of the robot-to-goal assignment implementing the inter-agent collision avoidance through direct constraints on mutually exclusive robot-goal pairs combined with a trajectory generation approach minimizing the duration of the reshaping process. 
  Theoretical validations confirm the optimality of \ac{catora}, with its efficacy further showcased through simulations, and a real-world outdoor experiment involving 19 \acp{uav}.
  Thorough numerical analysis shows the potential of \acs{catora} to decrease the time required to perform complex formation reshaping tasks by up to $49\%$, and $12\%$ on average compared to commonly used methods in randomly generated scenarios.   

\end{abstract}

%%}

%%{ SUPPLEMENTARY MATERIAL 

\section*{Supplementary material}

\textbf{Video:} {\url{https://mrs.felk.cvut.cz/tro2025catora}} 

\textbf{Open-source code:} {\url{https://github.com/ctu-mrs/catora}}

%%}

%%{ SECTION: INTRODUCTION

\section{Introduction}
\label{sec:introduction}
\IEEEPARstart{T}{eams} of autonomous mobile robots have found practical applications in various real-world scenarios, including search and rescue operations\cite{cerberusDarpa, petrlik2023UAVsSurfaceCooperative, best2024explorerDarpa}, environmental monitoring\cite{JI2022103967}, \cite{ercolani2023multiRobotGasMonitoring}, precision agriculture\cite{ju2022agriculture}, and automated warehouse systems\cite{azadeh2019warehouseManagement}.
In most cases, these teams consist of robots working together to achieve a common objective while independently navigating through the environment and avoiding collisions. 
However, in certain applications, it is advantageous for mobile robots to be arranged in a specific formation to accomplish desired tasks, such as documenting historical buildings~\cite{petracek2023dronument}, monitoring wildfires~\cite{afghah2019wildFiresMonitoring}, or creating drone light shows~\cite{waibel2017drone, kim2018surveyLightShows}.
In these scenarios, the robots often need to adjust their positions relative to one another to achieve the required formation shape for the mission's execution.
% , or to enable efficient navigation in a cluttered environment. 
Considering the formation shape adaptation as part of a robotic mission, the time efficiency of this process becomes of great importance.
This applies specifically to vehicles with operational time significantly constrained by battery endurance, such as multi-rotor \aclp{uav} (\acsp{uav}), especially in time critical missions such as search and rescue and applications requiring highly dynamic performance, such as drone light shows. 

%%{ Figure: Intro
\begin{figure}[!t]
  \centering
  \begin{tikzpicture}
    \node[anchor=south west,inner sep=0] (a) at (0,0) {
      \includegraphics[width=1.0\columnwidth]{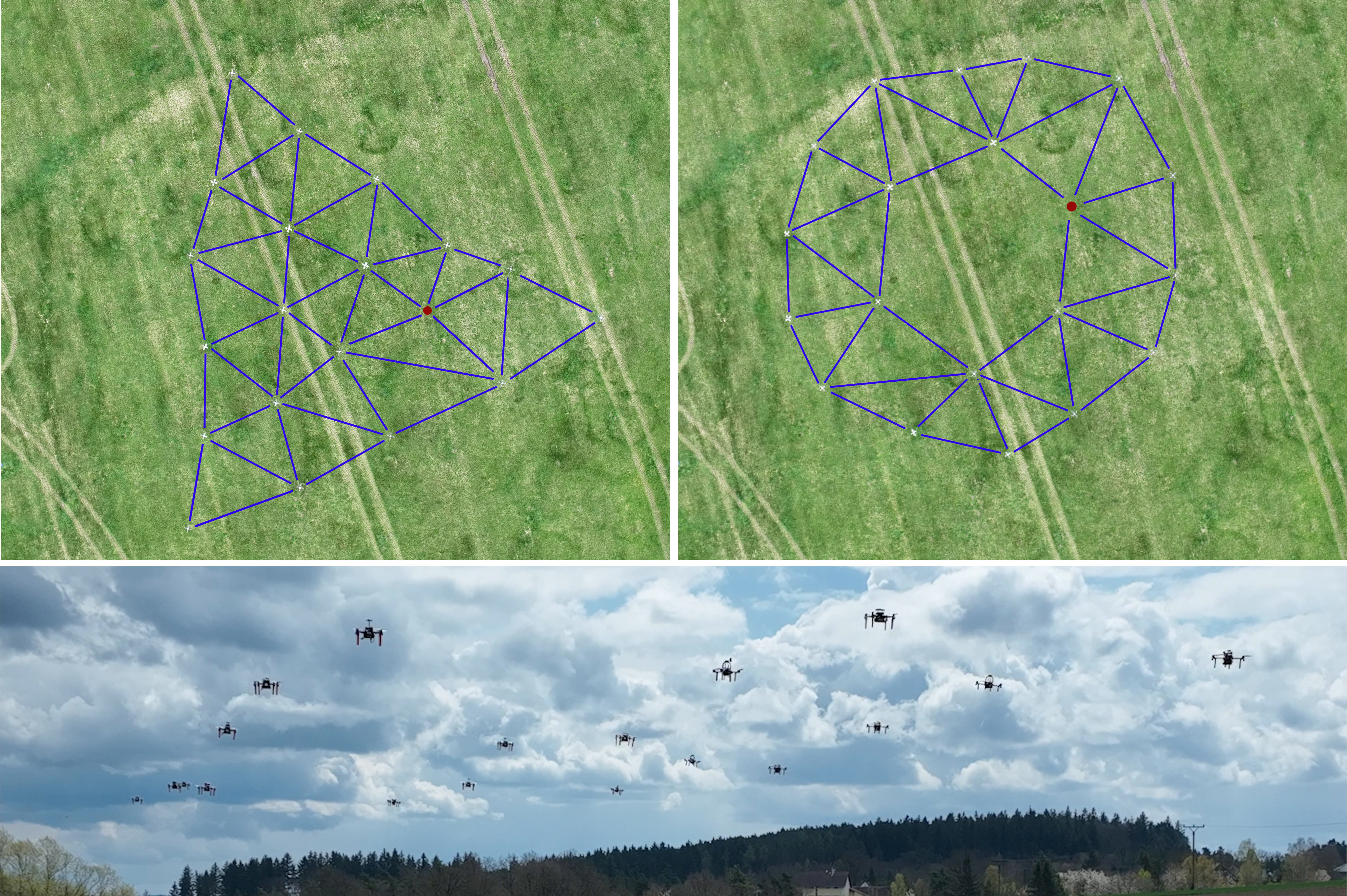}
    };

    \pgfmathsetmacro{\hzero}{0.0375}
    \pgfmathsetmacro{\vzero}{0.4208}
    \pgfmathsetmacro{\hshift}{0.503}
    \pgfmathsetmacro{\vshift}{-0.3745}
   
    \begin{scope}[x={(a.south east)},y={(a.north west)}]

        \node[fill=black, fill opacity=0.3, text=white, text opacity=1.0] at (\hzero, \vzero) {\textbf{(a)}};
        \node[fill=black, fill opacity=0.3, text=white, text opacity=1.0] at (\hzero+\hshift, \vzero) {\textbf{(b)}};
        \node[fill=black, fill opacity=0.3, text=white, text opacity=1.0] at (\hzero-0.001, \vzero+\vshift) {\textbf{(c)}};

      \draw (0.36,0.44) [scale bar=10];
      \draw (0.86,0.44) [scale bar=10];

    \end{scope}

  \end{tikzpicture}
  \caption{Deployment of the introduced \acl{catora} (\acs{catora}) in a small-scale drone visual performance with 19 \acp{uav}. The images show the transition of \acp{uav} guided by the \acs{catora} from a triangular shape (a) to a ring shape (b). This transition was performed within 7 seconds. The blue lines highlight the shape of the formation in top view images, while (c) captures the flying formation from the side. The red point represents a missing UAV that failed to start due to a HW failure.}
  \label{fig:intro}
  \vspace*{-0.2cm}
\end{figure}

%%}

This paper tackles the \ac{tofrp} with collision avoidance guarantees. 
The problem involves finding the assignment of robots to goals coupled with the generation of minimum-time collision-free trajectories.
From the robotics perspective, the formation reshaping problem is a specific instance of cooperative motion planning.
However, instead of having specific goals assigned to individual robots, the group of robots is given a set of unassigned goals to visit. 
The algorithms for the solution of assignment problems have been widely tackled in literature~\cite{kuhnHungarianAlgorithm, munkresHungarianAlgorithm, antonyshyn2023multipleRobotTaskPlanningSurvey, poudel2022TaskAssignmentAlgorithmsa, skaltsis2021taskAllocationSurvey, quinton2023marketApproachesTaskAllocationSurvey, chopra2017distributedHungarianAlgorithm, Burkard1998AssignmentP, carpaneto1981AlgorithmSolutionBottleneck, jang2018taskAssignmentUsingHeadonicGames, oh2016psobasedTaskAssignment, choi2009consensusBasedAuctionAlgorithm}.
However, since robots are physical entities sharing an environment, mutual collision avoidance has to be considered during the assignment process. 
This consideration implies that, in general, the individual cost of assigning two robot-goal pairs in a matching depends on the other assigned pairs, preventing a direct use of algorithms for the solution of general assignment problems.

Previous works in the field of formation reshaping vary in the level of decentralization, complexity, dimensions of the environment, optimization criteria, and applied methodology~\cite{turpin2014CaptConcurrentAssignment, agarwal2018SimultaneousOptimizationAssignmentsa, kloder2006PathPlanningPermutationinvarianta, GRAVELL2021104753, akella2020AssignmentAlgorithmsVariable, macalpine2015SCRAMScalableCollisionavoiding, wang2020ShapeFormationHomogeneous, khan2019LearningSafeUnlabeled}.  
Although the completion time of the reshaping process is a critical factor for algorithms deployed on robots with limited operational time, only a few works have taken the time criterion into account~\cite{GRAVELL2021104753, akella2020AssignmentAlgorithmsVariable, macalpine2015SCRAMScalableCollisionavoiding}.
However, none has addressed the minimization of completion time while simultaneously accounting for mutual collision avoidance among robots and the kinematic constraints associated with robots as physical entities.
Furthermore, these works lack guarantees regarding the solution completeness and quality, which is one aspect that limits the transfer of the algorithms to industrial applications where we observe a clear tendency to favour robotic systems with predictable and well-defined behavior guided by clear, understandable rules.

To this end, we address the problem of the time-efficient collision-free formation reshaping in 3D environments by introducing a centralized, deterministic \acl{catora} (\acs{catora}) directly optimizing the completion time (the so-called makespan) of the formation reshaping process while providing guarantees on a minimum mutual distance of involved agents (inter-agent collision avoidance).
The proposed approach comprises two key components designed together to allow us to provide theoretical guarantees of the overall complex robotic system's behavior: (i) an algorithm for optimal robot-to-goal assignment considering mutual collision avoidance among robots, and (ii) a computationally efficient trajectory generation approach minimizing the completion time of a set of trajectories.
\ac{catora} builds upon the Hungarian algorithm~\cite{kuhnHungarianAlgorithm, munkresHungarianAlgorithm} adapted to solve the robot-to-goal assignment as \ac{lbap}~\cite{PapadimitriouCombinatorialOptimization} effectively managing potential collisions between assigned robot-goal pairs.
The designed approach to the generation of a set of trajectories is based on a closed-form solution to the minimum-time single-trajectory generation problem~\cite{Penicka22UAVplanning} adapted to generate a set of trajectories minimizing their makespan, while keeping collision-free properties.

The optimality, efficiency, and other attributes of \ac{catora} have been confirmed through theoretical validation, statistical evaluation, and a real-world demonstration of a formation flight in a small-scale visual entertainment performance involving up to 19 \acp{uav} (see \autoref{fig:intro}).
The results demonstrate the capability of \ac{catora} to solve the introduced problem in real time (within a few milliseconds for instances of up to 32 robots) while significantly decreasing the time required to perform formation reshaping tasks (up to 49\% compared to \ac{lsap}-based solution \cite{kuhnHungarianAlgorithm}, \cite{munkresHungarianAlgorithm}) and providing collision avoidance guarantees.
This outcome holds significant value, especially for robots with limited operational time, and can be employed to enhance existing approaches or serve as a foundation for future research in formation reshaping, particularly concerning the autonomous deployment of cooperating multi-robot systems in real-world environments.

We consider the particular contributions of this paper to be:
\begin{itemize}
  \item We present a lower bound on a minimum mutual distance of trajectories for the solution of the robot-to-goal assignment as \ac{lbap} along with its theoretical proof.
  \item We introduce a deterministic, complete algorithm for solving the robot-to-goal assignment problem, minimizing the maximum length of the path among assigned robot-goal pairs while respecting constraints on mutually exclusive robot-goal pairs.
  \item We provide a closed-form solution for generating a set of trajectories connecting given start and goal positions and minimizing the makespan while preserving the guarantees of collision-free properties.
  \item We combine the contributions mentioned above to build \ac{catora}, the first known complete approach for the solution of the formation reshaping problem minimizing the makespan while providing collision avoidance guarantees in 3D environments. We provide verification of its properties through several proofs, numerical analysis, and a real-world experiment.
  \item We provide a quantitative and theoretical analysis of the \ac{catora} solution compared to the \ac{lsap}-based solution in terms of the makespan of a reshaping process, showing its superior performance.
\end{itemize}
\vspace*{0.3cm}

%%}

%%{ SECTION: RELATED WORK

\section{Related Work}
\label{sec:related_work}

Many recently published works in field of cooperative motion planning and  trajectory generation focus on the generation of collision-free trajectories in challenging scenarios while deliberately neglecting the assignment of goals to individual agents\cite{tordesillas2022mader,chen2023cooperativeMotionPlanning,park2023dlsc,kondo2024robustMader}.
The introduced time-optimal formation reshaping problem can be considered a robot-to-goal assignment problem closely coupled with a minimum-time trajectory generation.
One of the most frequently used methods for solving assignment problems is the Hungarian algorithm~\cite{munkresHungarianAlgorithm, kuhnHungarianAlgorithm}, capable of providing a solution minimizing the sum of individual associations costs with time complexity $O(n^3)$, where $n$ is the number of robots. 
Since the Hungarian algorithm was introduced, many algorithms tackling the assignment problem and its variants with different characteristics and performance have been developed~\cite{antonyshyn2023multipleRobotTaskPlanningSurvey, poudel2022TaskAssignmentAlgorithmsa, skaltsis2021taskAllocationSurvey, quinton2023marketApproachesTaskAllocationSurvey, chopra2017distributedHungarianAlgorithm}.
Notably, there have been innovative approaches grounded in game-theory principles~\cite{jang2018taskAssignmentUsingHeadonicGames}, swarm-intelligence~\cite{oh2016psobasedTaskAssignment}, reinforcement learning~\cite{zhao2019rlTaskAllocation}, and market-based methodologies~\cite{choi2009consensusBasedAuctionAlgorithm}. 
Despite these advancements, the Hungarian algorithm remains in use and has been utilized in many works as an efficient centralized algorithm with provable properties to solve the assignment problem.

First introduced in~\cite{fulkerson1953ProductionLineAssignmentProblem}, a specific case of the assignment problem minimizing the maximum individual cost (bottleneck) is called \acl{lbap} (\acs{lbap}).
Similarly to the original problem, \ac{lbap} can be solved using the Hungarian algorithm by modified cost substitutions~\cite{fulkerson1953ProductionLineAssignmentProblem}, applying threshold algorithms~\cite{Burkard1998AssignmentP} or shortest augmenting path algorithms~\cite{carpaneto1981AlgorithmSolutionBottleneck}.
An important adaptation of the Hungarian algorithm, used later in this work, is its dynamic variant~\cite{MillsTettey2007TheDH}.
Given the initial assignment, the dynamic variant enables solving an assignment problem with changed costs in approximately one hundredth of the computational requirements of the Hungarian algorithm starting from scratch.
Despite the variety of approaches developed to solve assignment problems, their direct use to formation reshaping is restricted by neglecting the collision resolution among robots.

Methods suitable for addressing the formation reshaping problem exhibit differences in terms of decentralization, guarantees on the quality of the solution, and employed methods.
The centralized approaches~\cite{turpin2014CaptConcurrentAssignment,agarwal2018SimultaneousOptimizationAssignmentsa, kloder2006PathPlanningPermutationinvarianta, akella2020AssignmentAlgorithmsVariable, macalpine2015SCRAMScalableCollisionavoiding, GRAVELL2021104753} mostly take advantage of the complete information for generating optimal solutions, but often impose an assumption on the collision-free environment.
In~\cite{turpin2014CaptConcurrentAssignment}, the authors propose a centralized approach based on the concurrent solution of assignment of goals and planning of trajectories (CAPT algorithm), which is further extended to a decentralized approach.
The CAPT incorporates a solution of \ac{lsap} minimizing the sum of squared traveled distances combined with constant-velocity and minimum-snap trajectories, which was proved to yield collision-free trajectories under the assumption on a minimum initial distance between agents.
The problem is extended to a variable goal formation (variable scale and translation) in~\cite{agarwal2018SimultaneousOptimizationAssignmentsa}, where the authors show that the problem of task assignment with variable goal formation can be transformed to \ac{lsap} via cost substitution. However, the approach is limited to 2D, and the mutual collisions are prevented by adapting the scale of the final formation.

Unlike the centralized solutions, the distributed approaches often suffer from incomplete information, leading to suboptimal solutions (e.g., approach applying local task swapping~\cite{wang2020ShapeFormationHomogeneous}) and limited guarantees on its quality (e.g., learning-based approach~\cite{khan2019LearningSafeUnlabeled}).
The robot-to-goal assignment problem is also solved in several works on distributed control of multi-rotor \ac{uav} formation in obstacle-free regions~\cite{lusk2020DistributedPipelineScalable}, as well as in complex environments~\cite{alonsoMora2019DistributedMultirobotFormation, quan2022FormationFlightDense}.
These works apply distributed task assignment algorithms~\cite{burger2012distributedSimplexAlgorithm, choi2009consensusBasedAuctionAlgorithm} to assign the robots to local goals during alignment to the target formation.
Although~\cite{lusk2020DistributedPipelineScalable, alonsoMora2019DistributedMultirobotFormation, quan2022FormationFlightDense} are proposed primarily for multi-rotor helicopters, they utilize the sum of squared distances as the minimization criterion for the assignment problem.
Such choice provides certain guarantees on the mutual distance of trajectories if solved optimally~\cite{turpin2014CaptConcurrentAssignment}. However, it does not reflect the problem being solved since minimizing squared traveled distances for in-flight multi-rotor \acp{uav} is neither optimal from the point of view of duration, energy consumption, nor any other appropriate criterion.

The time criterion was considered only in a few works dealing with formation reshaping.
In~\cite{GRAVELL2021104753}, the authors aim to minimize the total time in motion and build the solution of an assignment problem on duration of time-optimal trajectories. The algorithm relies on collision resolution via a combination of time delays and altitude adaptation, which limits its application to 3D environments.
The algorithm presented in~\cite{akella2020AssignmentAlgorithmsVariable} directly approaches the minimization of the makespan by defining the problem as \ac{lbap}.
The proposed solution considers a variable goal formation but is limited to 2D and ignores inter-agent collisions.
Another algorithm considering the minimization of the makespan~\cite{macalpine2015SCRAMScalableCollisionavoiding} also solves the assignment as a variant of \ac{lbap}, but it considers constant-velocity trajectories only.
The authors provide proof of collision avoidance guarantees; however, these are only valid in 2D environments with initial and goal configurations constrained to the grid.
Although some related works show impressive results, the oversight regarding the mutual collisions or minimum-time objective of the robot-to-goal assignment limits their efficient use in real-world applications.

%%}

%%{ SECTION: Problem definition

\section{Problem Definition}
\label{sec:problem_definition}

The \acl{tofrp} (\acs{tofrp}), tackled in this manuscript, is defined as follows.
Given the set of initial configurations of $n$ unlabeled robots $\mathbb{S} = \{\mathbf{s}_1, \mathbf{s}_2, \dots, \mathbf{s}_n\}$ and set of $n$ goal configurations $\mathbb{G} = \{\mathbf{g}_1, \mathbf{g}_2, \dots, \mathbf{g}_n\}$, find a set of collision-free trajectories $\mathbb{T}$ that guide the robots from $\mathbb{S}$ to $\mathbb{G}$ while minimizing the makespan of the reshaping process.

Let us define the makespan of reshaping the formation $F$ given the assignment $\phi: \mathbb{S} \rightarrow \mathbb{G} $ as
\begin{equation}\label{eq:makespan}
  M_r(\mathbb{S}, \mathbb{G}, \phi) = \max_{(i,j) \in \phi} \text{tf}(\text{T}(\mathbf{s}_i, \mathbf{g}_j)),
\end{equation}
where $\text{tf}(\text{T}(\mathbf{a}, \mathbf{b}))$ represents the time required to reach position $\mathbf{b}$ from position $\mathbf{a}$ following trajectory $\text{T}(\mathbf{a}, \mathbf{b})$.
Then, the \ac{tofrp} is defined as
\begin{equation}\label{eq:time_optimal_formation_reshaping}
  \begin{split}
    \minimize_{\phi \in \Phi, \text{T} \in \mathcal{T}}\,\,& M_r(\mathbb{S}, \mathbb{G}, \phi),\\
    \text{subject to}\,\,& \min_{0 \leq t \leq t_e} ||\text{T}(\mathbf{s}_i, \mathbf{g}_j, t) - \text{T}(\mathbf{s}_k, \mathbf{g}_l, t)|| \geq \Delta, \\
    &\forall (i,j) \in \phi, (k,l) \in \phi, (i,j) \neq (k,l),
  \end{split}
\end{equation}
where $T(\cdot, t)$ represents a point on a trajectory $T(\cdot)$ corresponding to time $t$, $t_e = \max(\text{tf}(T(\mathbf{s}_i, \mathbf{g}_j)), \text{tf}(T(\mathbf{s}_k, \mathbf{g}_l)))$ is a maximum duration of examined trajectories, $\Delta$ stands for the minimum acceptable mutual distance of robots, $\Phi$ is the set of all possible assignments from $\mathbb{S}$ to $\mathbb{G}$, and $\mathcal{T}$ is a class of arbitrary trajectory generation functions.

In the following sections, we introduce \ac{catora}, an optimal algorithm for the solution of problem~\eqref{eq:time_optimal_formation_reshaping} under the following assumptions:
\begin{enumerate}[leftmargin=0.85cm]
  \item [(A1)] Both the robots and the goals are unlabeled (any robot can be assigned to an arbitrary goal location).
  \item [(A2)] The robots are stationary in the initial and goal configurations.
  \item [(A3)] The motion of the robots between the initial and goal configuration is limited to straight paths with mutually equivalent time parametrization.
  \item [(A4)] The robots are considered to be spheres with radius R for the collision avoidance resolution.
  \item [(A5)] The minimum distance between the pairs of initial configurations and the pairs of goal configurations $\delta = \min_{(i,j) \in \{1, \dots, n\}^2,\, i \neq j}\min \left(||\mathbf{s}_i -  \mathbf{s}_j||, ||\mathbf{g}_i -  \mathbf{g}_j||\right)$ fulfills the condition $\delta \geq \eta \Delta$ with $\eta \geq \sqrt{2}$ being a constant parameter.
  \item [(A6)] The convex hull of $\mathbb{S} \cup \mathbb{G}$ is free of obstacles apart from the robots themselves.

\end{enumerate}

The assumptions (A1) -- (A6) are necessary to guarantee the optimality of the proposed algorithm to the solution of \ac{tofrp}, as defined in~\eqref{eq:time_optimal_formation_reshaping}.
However, in \autoref{sec:experimental_results}, we show that the assumption (A2) is not strict and that the algorithm can also be used for reshaping moving formations, and further that the optimal solution considering assumption (A3) stays close to the theoretical lower bound of the optimal solution not considering (A3).
The assumptions (A5) and especially (A6) impose significant limitations, but both (A5) and (A6) may be easily satisfied in most of the real-world scenarios discussed in~\autoref{sec:introduction}, making the proposed solution practical for real-world applications.

%%}

%%{ PRELIMINARIES

\section{Overview of the maximum matching in bipartite graphs and the Hungarian method}
\label{sec:terminology}

In this section, we overview key terms and definitions from graph theory applied in a further description of the proposed methodology and briefly describe the Hungarian method employed in the proposed algorithm.

\subsection{Maximum matching in bipartite graphs}

Key terms related to maximum matching problem in bipartite graphs used in following sections are:
\begin{itemize}
  \item \textit{Bipartite graph}: graph $G = \{V, E\} = \{V_x, V_y, E\}$, where the set of vertices $V$ can be partitioned in two disjoint subsets $V_x, V_y$, such that the set of edges $E$ does not contain any edge connecting vertices from the same partition.
  \item \textit{Matching}: subset of edges $E_M \subset E$, such that every vertex in $V$ is incident to at most one edge in $E_M$.
  \item \textit{Cardinality of the matching}: number of edges in a matching $C_M = |E_M|$. The matching containing the maximum possible number of edges is called \textit{maximum cardinality matching}. If $C_M = |V_x| = |V_y|$, the matching is called \textit{perfect}. 
  \item \textit{Matched edge}: edge $e_{ij}$ is called matched if it is a part of the matching, unmatched otherwise. 
  \item \textit{Matched vertex}: vertex $v$ is matched if it is incident to an edge in matching $E_M$, and unmatched otherwise.
  \item \textit{Alternating path}: path in a graph that starts with an unmatched vertex and alternates between edges that do not and do belong to the matching.
  \item \textit{Augmenting path}: an alternating path that ends with an unmatched vertex.
  % \item \textit{Vertex labeling}: An alternating path that ends with an unmatched vertex.
  \item \textit{Minimum weight bipartite matching problem}: given bipartite graph $G = \{V_x, V_y, E\}$ and weight function $w: E \rightarrow \mathbb{R}$, find a maximum cardinality matching $E_M$, such that $\sum_{e_{ij} \in E_M} w(e_{ij})$ is minimum.
  \item \textit{Dual problem of minimum weight bipartite matching problem}: given bipartite graph $G = \{V, E\} = \{V_x, V_y, E\}$, weight function $w: E \rightarrow \mathbb{R}$, and vertex labeling function $l_f: V \rightarrow \mathbb{R}$, find a feasible labeling of a maximum cost $c(l_f) = \sum_{v_{x,i} \in V_x} l_f(v_{x,i}) + \sum_{v_{y,j} \in V_y} l_f(v_{y,j})$, where \textit{feasible labeling} is a choice of labels such that $l_f(v_{x,i}) + l_f(v_{y,j}) \leq w(e_{ij})$.

\end{itemize}

\noindent For simplicity of description and without loss of generality, we assume that the bipartite graph $G = \{V_x, V_y, E\}$ is complete and balanced, i.e., $|V_x| = |V_y|$ in the remainder of the paper.

\subsection{Hungarian algorithm}\label{sec:hungarian_algorithm}
The Hungarian algorithm~\cite{kuhnHungarianAlgorithm, munkresHungarianAlgorithm} is widely applied for the solution of the assignment problem (which can also be represented as a minimum weight bipartite matching problem) with proven complexity $O(n^3)$, where $n$ is a number of matched entities.
The input of the Hungarian algorithm is a square biadjacency matrix $\mathbf{M}_d$ representing a weighted bipartite graph $G$ with weight function $w: E \rightarrow \mathbb{R}$.
The algorithm exploits the properties of the dual of minimum weight bipartite matching problem by using dual variables $u_i = l_f(v_{x,i}),\,\, v_j = l_f(v_{y,j}),\,\, i,j \in \{0, 1, \dots, N\}$.
These variables are updated during the run of the algorithm and used to determine the admissibility of edge $e_{ij}$ given by condition
\begin{equation}\label{eq:admissibility_condition}
  u_i + v_j = w(e_{ij}).
\end{equation}

The Hungarian algorithm starts with an empty matching $\phi$ and repeatedly searches for augmenting paths in an equality subgraph formed by edges fulfilling condition~\eqref{eq:admissibility_condition}.
The search for an augmenting path is realized by building so-called Hungarian trees that are rooted in unmatched nodes.
If the Hungarian tree formed by alternating paths in a graph $G$ contains an augmenting path, the current matching is updated by flipping the matched and unmatched edges along the found path. 
This process always increases the cardinality of current matching by one in a single step of the algorithm.
If the augmenting path is not found in a current equality subgraph, the values of dual variables are updated such that the dual task remains feasible and new edges are introduced into the equality subgraph. 
Then, the search for an augmenting path continues.
The incremental increase of the cardinality of the matching ensures that the algorithm reaches a perfect matching for $\mathbf{M}_d \in \mathbb{R}^{n\times n}$ in $n$ steps of a successful search for an augmenting path.
We refer to~\cite{PapadimitriouCombinatorialOptimization, munkresHungarianAlgorithm, kuhnHungarianAlgorithm} for a detailed description of the algorithm and proofs of its properties.

%%}

%%{ SECTION: Algorithm overview

\section{\acl{catora} - overview}
\label{sec:algorithm_overview}

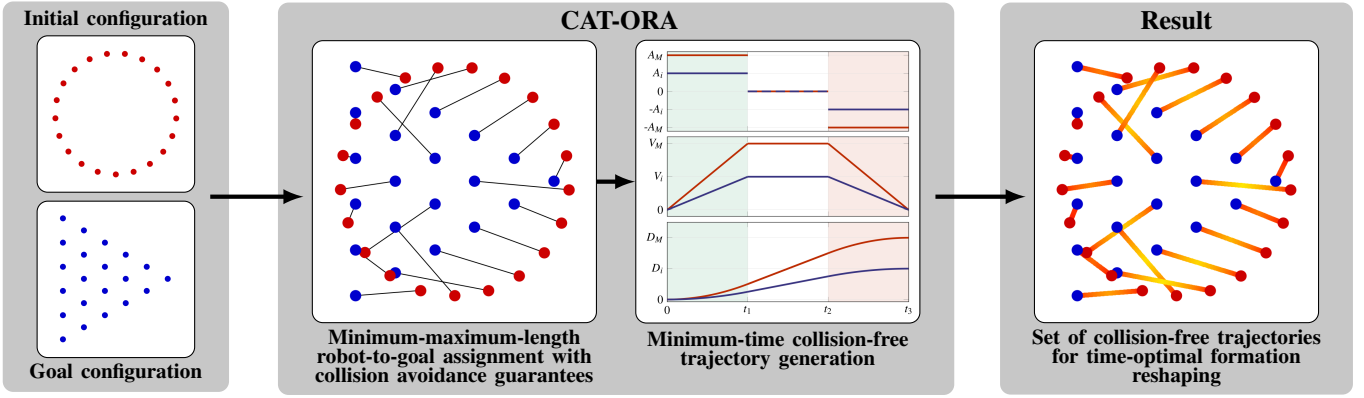
\begin{figure*}[thpb]
  \centering
  \input{fig/catora/catora.tex}
  \vspace*{-0.2cm}
  \caption{Block diagram of the proposed \acl{catora} (\acs{catora}). The colors of the trajectories in the image on the right encode the velocity profile of particular trajectories, with red being equal to zero velocity and yellow to maximum velocity.}
  \label{fig:catora_overview}
\vspace*{-0.6cm}
\end{figure*}

The introduced \ac{tofrp}~\eqref{eq:time_optimal_formation_reshaping} consists of two problems:
(i) the optimal assignment of initial configurations to goal configurations and (ii) the generation of collision-free minimum-time trajectories.
In further description, we assume that these two problems are completely separable, and that
\begin{equation}\label{eq:assumption_separation}
  \begin{split}
    \max_{(i,j) \in \phi_a} \text{tf}(\text{T}(\mathbf{s}_i, \mathbf{g}_j)) \geq  \max_{(i,j) \in \phi_b} \text{tf}(\text{T}(\mathbf{s}_i, \mathbf{g}_j)) \implies \\
 \max_{(i,j) \in \phi_a} ||\mathbf{s}_i - \mathbf{g}_j|| \geq  \max_{(i,j) \in \phi_b} ||\mathbf{s}_i - \mathbf{g}_j||\\
  \end{split}
\end{equation}
holds for all assignments $\phi_a, \phi_b$ from $\mathbb{S}$ to $\mathbb{G}$.
This means that the assignment minimizing the makespan~\eqref{eq:makespan} corresponds to the assignment minimizing the maximum distance $d_{max}$ between the assigned initial and goal configurations
\begin{equation}\label{eq:max_length}
  d_{max} = \max_{(i,j) \in \phi} ||\mathbf{s}_i - \mathbf{g}_j||.
\end{equation}
This allows us to design \acl{catora} (\acs{catora}) such that the robot-to-goal assignment and generation of collision-free minimum-time trajectories are tackled in a decoupled way (see \autoref{fig:catora_overview} for block diagram of \acs{catora}).
The proof that the proposed decoupled approach does not influence the optimal solution and that~\eqref{eq:assumption_separation} is fulfilled within the proposed approach is provided in~\autoref{sec:proof_of_solution_independence}.

\subsection{Minimum-weight robot-to-goal assignment}
The task of assigning the goal configurations to particular robots can be defined as an integer linear program
\begin{equation}\label{eq:lsap}
  \begin{split}
    \minimize &\sum_{i = 1}^n \sum_{j = 1}^n w(e_{ij}) x_{ij}, \\
    \text{subject to} &\sum_{i = 1}^n x_{ij} = 1\,\, \forall j \in \{1, \dots, n\},\\
    &\sum_{j = 1}^n x_{ij} = 1\,\, \forall i \in \{1, \dots, n\},\\
    & x_{ij} \in \{0, 1\}  \,\, \forall i, j \in \{1, \dots, n\},
  \end{split}
\end{equation}
where $w(e_{ij})$ is the cost of assignment of the goal configuration $\mathbf{g}_j$ to initial configuration $\mathbf{s}_i$, and $x_{ij} = 1$ if $\mathbf{s}_i$ is assigned to $\mathbf{g}_j$, $x_{ij} = 0$ otherwise.
The problem~\eqref{eq:lsap} is often referred to as \acl{lsap} (\acs{lsap}) which can be efficiently solved by the Hungarian algorithm~\cite{kuhnHungarianAlgorithm, munkresHungarianAlgorithm}.
Using the squared Euclidean distances $||\mathbf{s}_i - \mathbf{g}_j||^2$ as costs $w(e_{ij})$, the solution of \eqref{eq:lsap} was proved to guarantee the collision-free property of constant-velocity trajectories when $\delta \geq \sqrt{2}R$~\cite{turpin2014CaptConcurrentAssignment}, where $R$ is the safety radius of robots.

In compliance with~\eqref{eq:assumption_separation}, problem \eqref{eq:lsap} must be reformulated to minimize the length of the longest trajectory in the assignment for solving \ac{tofrp}~\eqref{eq:time_optimal_formation_reshaping}:
\begin{equation}\label{eq:lbap}
  \begin{split}
    \minimize \max_{i,j \in \{0, \dots, n\}} w(e_{ij}) &x_{ij}, \\
    \text{subject to} &\sum_{i = 1}^n x_{ij} = 1\,\, \forall j \in \{1, \dots, n\},\\
    &\sum_{j = 1}^n x_{ij} = 1\,\, \forall i \in \{1, \dots, n\},\\
    & x_{ij} \in \{0, 1\}  \,\, \forall i, j \in \{1, \dots, n\},
  \end{split}
\end{equation}
known as \acl{lbap} (\acs{lbap})~\cite{fulkerson1953ProductionLineAssignmentProblem}.
The specificity of the robot-to-goal assignment problem requires augmenting~\eqref{eq:lbap} by including constraints on mutually colliding paths
\begin{equation}\label{eq:lbap_with_mutual_constraints}
    \sum_{e=1}^{|C_m|} x_{\text{idx}(C_{m,e})} = 1\,\, \forall C_m \in \mathbf{C},
\end{equation}
where $\mathbf{C}$ is a set of constraints represented by sets of mutually colliding edges, and $\text{idx}(\cdot)$ represents the indices of the corresponding edge.

Solving~\eqref{eq:lbap} augmented by~\eqref{eq:lbap_with_mutual_constraints} using standard optimization methods would require to compute the whole set of mutual collision constraints prior to the solution of the problem, which would require to check collisions among $\frac{n^2(n-1)^2}{2}$ pairs of edges, making it computationally intractable for large $n$.
In this work, we propose a novel algorithm that combines the Hungarian algorithm~\cite{kuhnHungarianAlgorithm, munkresHungarianAlgorithm} and its dynamic variant~\cite{MillsTettey2007TheDH} with fast collision checking.
The collision checking is built on the analysis of theoretical guarantees on a minimum mutual distance of trajectories for an assignment provided as a solution of \ac{lbap} (detailed in~\autoref{sec:lbap_theoretical_analysis}). 
A thorough description of the robot-to-goal assignment component of \ac{catora} is provided in~\autoref{sec:task_assignment_lbap_algorithm}.

\subsection{Minimum makespan collision-aware trajectory planning}

The generation of collision-free trajectories between pairs of matched initial and goal configurations that minimize the makespan of the formation reshaping process requires considering the generation of individual minimum-time trajectories.
In compliance with assumption (A3), we consider a model with single-dimension point-mass dynamics $\ddot{p} = a$, with constraints on acceleration control inputs $-a_{max} \leq a \leq a_{max}$, and limits on velocity $v = \dot{p}$, $0 \leq v \leq v_{max}$.
Although the individual minimum-time trajectories using this model would minimize the makespan, they do not preserve the guarantees on mutual collision avoidance.
Exploiting the fact that the minimized makespan is influenced only by the length of the longest trajectory, we have proposed an approach for generating mutually collision-free minimum-time trajectories, preserving the theoretical guarantees on minimum mutual distance.
The proposed approach, which is based on a closed-form solution of the minimum-time trajectory generation problem, is detailed in~\autoref{sec:trajectory_generation}, along with the proof of theoretical guarantees.
%%}

%%{ SECTION: Theoretical guarantees of LBAP

\section{Theoretical guarantees of LBAP solution}\label{sec:lbap_theoretical_analysis}

The solution of \ac{lsap} using squared Euclidean distances as costs has been proved to guarantee minimum distance between trajectories $d_{min}$ equal to
\begin{equation}
  d_{min} = \frac{\sqrt{2}}{2} \delta,
\end{equation}
where $\delta = \min_{(i,j) \in \{1, \dots, n\}^2,\, i \neq j}\min \left(||\mathbf{s}_i -  \mathbf{s}_j||, ||\mathbf{g}_i -  \mathbf{g}_j||\right)$ is the minimum distance between any two initial and goal configurations.
The detailed description of the proof is provided in\cite{turpin2014CaptConcurrentAssignment}.
In the following sections, similar properties are derived and proved for the application of \ac{lbap} to solve the same problem while minimizing the maximum distance between the assigned initial and goal configurations~\eqref{eq:max_length}, thus minimizing the makespan~\eqref{eq:assumption_separation}.

\subsection{Minimum mutual distance of two trajectories}
For the analysis of the guarantees on the minimum distance between trajectories, we consider the following scenario.
Without loss of generality, we can assume fixed initial and goal positions $\mathbf{s}_i$, $\mathbf{s}_j$, $\mathbf{g}_i$ with $||\mathbf{s}_j - \mathbf{g}_i||=d$, $||\mathbf{s}_i - \mathbf{g}_i||=Md,\, M \in [0, 1)$ and an arbitrarily positioned goal position $\mathbf{g}_j$ such that
 $||\mathbf{s}_j - \mathbf{g}_i|| \geq ||\mathbf{s}_i - \mathbf{g}_j||$ (see \autoref{fig:trapezoid_general_case}).
For $M \geq 1$, the \ac{lbap} solution coincides with the solution to \ac{lsap}, thus implicitly providing the same guarantees on minimum mutual distance.

\begin{figure}[thpb]
  \centering
  \input{fig/trapezoids/trapezoid_general_case.tex}
  \caption{An example problem consisting of two initial positions $\mathbf{s}_i$, $\mathbf{s}_j$ and two goal locations $\mathbf{g}_i$, $\mathbf{g}_j$. Without loss of generality, the distance $||\mathbf{s}_j - \mathbf{g}_i||$ is assumed to be equal to $d$ and $||\mathbf{s}_i - \mathbf{g}_i|| = Md$, where $M \in [0, 1)$.}
  \label{fig:trapezoid_general_case}
\end{figure}
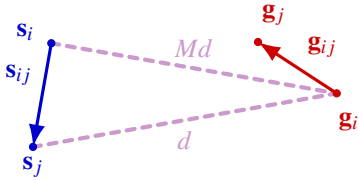

Considering the trajectories with constant velocity, the position of robot $\mathbf{x}_i(t)$ at time $t$ following the trajectory from initial position $\mathbf{s}_i$ to goal position $\mathbf{g}_i$ can be described as
\begin{equation}\label{eq:parametrization_const_velocity}
  \mathbf{x}_i(t) = (1-\alpha)\mathbf{s}_i + \alpha \mathbf{g}_i,
\end{equation}
with $\alpha = \frac{t}{t_d}$ uniformly sampled on $[0, 1]$, where $t_d$ is the duration of the trajectory.
Hence, the mutual distance between robots following trajectories $T_i$ and $T_j$ of the same duration $t_d$ from $\mathbf{s}_i$ to $\mathbf{g}_i$ with velocity $v_i$ and $\mathbf{s}_j$ to $\mathbf{g}_j$ with velocity $v_j \neq v_i$, respectively, can be expressed as
\begin{equation}\label{eq:mutual_dist}
  \begin{split}
    ||\mathbf{x}_j(t) - \mathbf{x}_i(t)|| = &||(1-\alpha)\mathbf{s}_j + \alpha \mathbf{g}_j - (1-\alpha)\mathbf{s}_i - \alpha \mathbf{g}_i||\\
    = & ||(1-\alpha) (\mathbf{s}_j - \mathbf{s}_i) + \alpha(\mathbf{g}_j - \mathbf{g}_i)||.
  \end{split}
\end{equation}

Using the notation introduced at the beginning of this section and notations
\begin{align}
  \mathbf{s}_{ij} &= \mathbf{s}_j - \mathbf{s}_i,\\
  \mathbf{g}_{ij} &= \mathbf{g}_j - \mathbf{g}_i,
\end{align}
the equation~\eqref{eq:mutual_dist} can be written in the following form
\begin{equation}
  ||\mathbf{x}_j(t) - \mathbf{x}_i(t)|| = ||(1-\alpha)\mathbf{s}_{ij} + \alpha \mathbf{g}_{ij}||.
\end{equation}
Then, the squared distance is given by
\begin{equation}\label{eq:squared_dist}
  \begin{split}
    ||\mathbf{x}_j(t) - \mathbf{x}_i(t)||^2 =& ||(1-\alpha)\mathbf{s}_{ij} + \alpha \mathbf{g}_{ij}||^2 = ||\mathbf{s}_{ij} + \alpha (\mathbf{g}_{ij} - \mathbf{s}_{ij})||^2 \\
    =& (\mathbf{s}_{ij} + \alpha (\mathbf{g}_{ij} - \mathbf{s}_{ij}))^T (\mathbf{s}_{ij} + \alpha (\mathbf{g}_{ij} - \mathbf{s}_{ij})).
  \end{split}
\end{equation}
\begin{remark}
In the remainder of this manuscript, we exploit the monotonicity of the quadratic function in the positive domain to replace the search for minimum distance by search for minimum squared distance.
\end{remark}

Following the theorem in~\cite{turpin2014CaptConcurrentAssignment}, for notational convenience, we define:
\begin{equation}\label{eq:start_and_goals_substitutions}
  \begin{split}
    a &\equiv \mathbf{s}_{ij}^T\mathbf{s}_{ij},\\
    b &\equiv \mathbf{s}_{ij}^T\mathbf{g}_{ij},\\
    c &\equiv \mathbf{g}_{ij}^T\mathbf{g}_{ij}.
  \end{split}
\end{equation}
This enables us to simplify~\eqref{eq:squared_dist} to
\begin{equation} \label{eq:squared_dist_alpha}
  ||\mathbf{x}_j(t) - \mathbf{x}_i(t)||^2 = \alpha^2 (a-2b+c) -2\alpha(a-b) + a.
\end{equation}
From~\eqref{eq:squared_dist_alpha}, the value of $\alpha$ minimizing the distance between trajectories of robots with indices $i, j$, can be found as
\begin{equation}\label{eq:alpha_min}
  \alpha_{ij}^* = \frac{a-b}{a-2b+c}.
\end{equation}
By substituting the value of $\alpha_{ij}^*$ in~\eqref{eq:squared_dist_alpha}, the minimum squared distance between trajectories $T_i, T_j$ is given by
\begin{equation}\label{eq:squared_dist_x}
  ||\mathbf{x}_i - \mathbf{x}_j||_{min}^2 = \left\{\begin{array}{cl} \frac{ac-b^2}{a-2b+c} & \text{if } 0 < \alpha_{ij}^* < 1,\\
    \delta_{ij}, & \text{otherwise,}
  \end{array} \right.
\end{equation}
where 
\begin{equation}\label{eq:min_dist_starts_and_goals}
  \delta_{ij} = \min(||\mathbf{s}_{ij}||^2, ||\mathbf{g}_{ij}||^2).
\end{equation}
The minimum squared distance~\eqref{eq:squared_dist_x} was already proved to be greater than $\frac{1}{2}\delta_{ij}$ for $b \geq 0$ in\cite{turpin2014CaptConcurrentAssignment}, which is guaranteed for solutions provided by \ac{lsap} using quadratic costs since
\begin{equation}\label{eq:lsap_condition} 
  ||\mathbf{s}_{i} - \mathbf{g}_{i}||^2 + ||\mathbf{s}_{j} - \mathbf{g}_{j}||^2 \leq ||\mathbf{s}_{j} - \mathbf{g}_{i}||^2 + ||\mathbf{s}_{i} - \mathbf{g}_{j}||^2
\end{equation}
holds for all pairs $\mathbf{s}_{i}, \mathbf{g}_{i}$, and $\mathbf{s}_{j}, \mathbf{g}_{j}$ being part of an optimal assignment. 
Considering following equality which holds for general vectors $\mathbf{x}, \mathbf{y}$
\begin{equation}
    ||\mathbf{x} - \mathbf{y}||^2 = (\mathbf{x} - \mathbf{y})^T(\mathbf{x} - \mathbf{y}) = \mathbf{x}^T\mathbf{x} - 2\mathbf{x}^T\mathbf{y} + \mathbf{y}^T\mathbf{y},
\end{equation}
condition~\eqref{eq:lsap_condition} can be rewritten to
\begin{equation} 
    \mathbf{s}_i^T\mathbf{g}_i + \mathbf{s}_j^T\mathbf{g}_j - \mathbf{s}_i^T\mathbf{g}_j - \mathbf{s}_j^T\mathbf{g}_i = \mathbf{s}_{ij}^T\mathbf{g}_{ij} = b \geq 0.
\end{equation}

In contrast to the application of \ac{lsap}, the \ac{lbap} solution does not directly provide any guarantee on values $a, b, c$, (defined in~\eqref{eq:start_and_goals_substitutions}), and thus the worst-case minimum distance between trajectories is zero.
Given the guarantees $||\mathbf{x}_i - \mathbf{x}_j||_{min}^2 \geq \frac{1}{2}\delta_{ij}$ for $b \geq 0$~\cite{turpin2014CaptConcurrentAssignment}, we further focus on analyzing the guarantees of specific case of the \ac{lbap} solutions with $b < 0$. 

Without loss of generality, we assume $c = ka, k \geq 1$, and $||\mathbf{s}_{ij}||$ to be constant, leading to $a=a_0$ with some constant $a_0 \in \mathbb{R}^+$.
Considering $b = \sqrt{a}\sqrt{c} \cos\lambda$, where $\lambda$ is an angle between vectors $\mathbf{s}_{ij}$ and $\mathbf{g}_{ij}$, and the constraints $a \geq 0$, $a \geq b$, $c \geq b$ enforced by constraints on $\alpha \in [0, 1]$ and $b < 0$, we can rewrite the first part of~\eqref{eq:squared_dist_x} to 
\begin{equation}\label{eq:squared_dist_x_rewritten}
  ||\mathbf{x}_i - \mathbf{x}_j||_{min}^2 = a\frac{k(1-\cos^2(\lambda))}{k-2\sqrt{k}\cos(\lambda) +1}.
\end{equation}
Since $b < 0 \implies \cos{\lambda} < 0$ and $a$ is constant, the gradient of \eqref{eq:squared_dist_x_rewritten} with respect to $k$ is non-negative for all admissible values of $k$.
Therefore, the squared distance is minimal if $a=c$, meaning that the distance between the two start configurations and distance between two goal configurations equal.
Using the substitution~\eqref{eq:start_and_goals_substitutions}, this result can be also interpreted as $||\mathbf{s}_{ij}|| = ||\mathbf{g}_{ij}||$.
For $a=c$, the equation~\eqref{eq:squared_dist_x} is simplified to
\begin{equation}\label{eq:min_dist_simplified}
  ||\mathbf{x}_i - \mathbf{x}_j||_{min}^2 = \left\{\begin{array}{cl} \frac{a+b}{2} & \text{if } 0 < \alpha_{ij}^* < 1,\\
    \delta_{ij}, & \text{otherwise.}
  \end{array} \right.
\end{equation}
Since $a$ is constant and $a > 0$, the minimum distance is achieved for minimum $b$, such that $a=c$.

\begin{remark}
The distance between two trajectories following line segments is minimal when the trajectories intersect, which means that they lie in the same plane.
For each pair of line segments $(\mathbf{q}, \mathbf{r})$ in three-dimensional space, it holds that either $\mathbf{q} || \mathbf{r}$, and thus $\mathbf{q}$ and $\mathbf{r}$ lie in the same plane, or we can find a plane $P$ such that $\mathbf{q} \in P$ and $\mathbf{r} || P$.
The projection $\mathbf{r}^\prime$ of $\mathbf{r}$ into a parallel plane preserves the dimension of $\mathbf{r}$ and 
\begin{equation}
  \begin{split}
    \hspace{-0.02cm}&||(1-\kappa)(\mathbf{q}(t_0) - \mathbf{r}(t_0)) + \kappa(\mathbf{q}(t_f)-\mathbf{r}(t_f))|| \geq\\
    \hspace{-0.02cm}&||(1-\kappa)(\mathbf{q}(t_0) - \mathbf{r^\prime}(t_0)) + \kappa(\mathbf{q}(t_f)-\mathbf{r^\prime}(t_f))||\,\, \forall \kappa \in [0, 1], 
  \end{split}
\end{equation}
where $\mathbf{p}(t_0), \mathbf{r}(t_0)$ and $\mathbf{p}(t_f), \mathbf{r}(t_f)$ stand for the start and end points of the line segments, respectively, and $\kappa$ is an independent variable.
This allows us to solve the rest of the problem in two-dimensional space without the loss of generality.
\end{remark}

Based on the definition~\eqref{eq:start_and_goals_substitutions}, the value of $b$ is given by
\begin{equation}\label{eq:b_min}
  b = \mathbf{s}_{ij}^T\mathbf{g}_{ij} = ||\mathbf{s}_{ij}||||\mathbf{g}_{ij}||\cos(\beta + \gamma),
\end{equation}
where $\beta = \angle \mathbf{g}_i\mathbf{s}_j\mathbf{s}_i$ and $\gamma = \angle \mathbf{g}_j\mathbf{g}_i\mathbf{s}_j$ (see~\autoref{fig:trapezoid_general_case_with_angles}).
\begin{figure}[thpb]
  \centering
  \input{fig/trapezoids/trapezoid_general_case_with_angles.tex}
  \caption{Illustration of the general case of an assignment problem with fixed points $\mathbf{s}_i, \mathbf{s}_j, \mathbf{g}_i$ and variable point $\mathbf{g}_j$. }
  \label{fig:trapezoid_general_case_with_angles}
\end{figure}
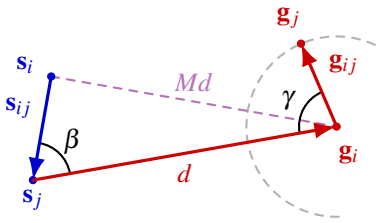
Considering~\eqref{eq:b_min} and the limitations on the values of $\beta$ and $\gamma$ coming from $M \in [0,1)$, the distance is minimized for $\beta + \gamma = \pm\pi$, resulting in the intersection of $\mathbf{s}_i - \mathbf{g}_i$ and $\mathbf{s}_j - \mathbf{g}_j$, leading to a minimum mutual distance equal to zero.  
This implies that there are no theoretical guarantees on the limits for the minimum mutual distance of robots following trajectories $\text{T}_i, \text{T}_j$ for a general case.
Therefore, we further focus on analysis of the guarantees on minimum mutual distance depending on the value of $M$.

Given the condition
\begin{equation}\label{eq:case_leq_m}
  \frac{\max(||\mathbf{s}_i - \mathbf{g}_i||, ||\mathbf{s}_j - \mathbf{g}_j||)}{\max(||\mathbf{s}_i - \mathbf{g}_j||, ||\mathbf{s}_j - \mathbf{g}_i||)} \leq M,
\end{equation}
we can state the following theorem:
\begin{theoremwobrackets}[]\label{th:minimum_dist_lbap}
  If $\frac{\max(||\mathbf{s}_i - \mathbf{g}_i||, ||\mathbf{s}_j - \mathbf{g}_j||)}{\max(||\mathbf{s}_i - \mathbf{g}_j||, ||\mathbf{s}_j - \mathbf{g}_i||)} \leq M,\,\,M \in [0,1)$, then minimum mutual distance $d_{ij,min} \geq \sqrt{1-M^2} \delta_{ij}$.
\end{theoremwobrackets}
Considering the assumption in~\autoref{th:minimum_dist_lbap} and~\eqref{eq:b_min}, the mutual distance is minimized when
\begin{equation}
  ||\mathbf{s}_i - \mathbf{g}_i|| = ||\mathbf{s}_j - \mathbf{g}_j|| = M \max(||\mathbf{s}_j - \mathbf{g}_i||, ||\mathbf{s}_i - \mathbf{g}_j||),
\end{equation}
which maximizes $\beta + \gamma$ in range $(0, \pi)$.
This corresponds to a situation in which the positions $\mathbf{s}_i, \mathbf{s}_j, \mathbf{g}_i, \mathbf{g}_j$ form vertices of an isosceles trapezoid.
The detailed analysis of minimum mutual distance of robots following trajectories formed by diagonals of an isosceles trapezoid is provided in~\aref{ap:min_dist_trapezoid}, along with proof of~\autoref{th:minimum_dist_lbap}.

As a consequence of~\autoref{th:minimum_dist_lbap}, a pair of constant-velocity trajectories for which~\eqref{eq:case_leq_m} holds is guaranteed to be collision-free under the condition
\begin{equation}\label{eq:collision_free_condition}
  \delta_{ij} \geq \frac{\sqrt{1-M^2}}{1-M^2} \Delta.
\end{equation}
% \begin{remark}
Since we have analyzed the worst-case scenario, the resulting condition~\eqref{eq:collision_free_condition} forms a lower bound on the minimum distance between a pair of trajectories that can be applied for an efficient mutual collisions check of robots following constant-velocity trajectories.

%%}

%%{ Algorithm for solution of LBAP

\section{Algorithm for solution of \ac{lbap} with guarantees on minimum distance and collision-free trajectories}\label{sec:task_assignment_lbap_algorithm}

As previously stated in~\autoref{sec:algorithm_overview}, neither the Hungarian algorithm nor its adaptations for \ac{lbap} can be used to directly solve the \ac{lbap} with constraints on mutually colliding trajectories~\eqref{eq:lbap_with_mutual_constraints}.
  Using the results obtained in~\autoref{sec:lbap_theoretical_analysis}, we introduce an optimal algorithm for the solution of~\eqref{eq:lbap} with additional constraints~\eqref{eq:lbap_with_mutual_constraints}. The algorithm which forms the first component of the proposed \ac{catora} is outlined in~\autoref{alg:overall_lbap}, illustrated in~\autoref{fig:overall_lbap}, and detailed in the following sections.

\begin{figure}[htpb]
  \centering
  \input{fig/blap_algorithm_scheme.tex}
  \vspace*{-0.3cm}
  \caption{Simplified diagram illustrating succession of individual steps of the algorithm for robot-to-goal assignment considering mutual collision constraints. The green and red arrows indicate the branching based on positive and negative results, respectively. The detailed description of individual steps is provided in~\autoref{sec:task_assignment_lbap_algorithm}.}
  \label{fig:overall_lbap}
\end{figure}
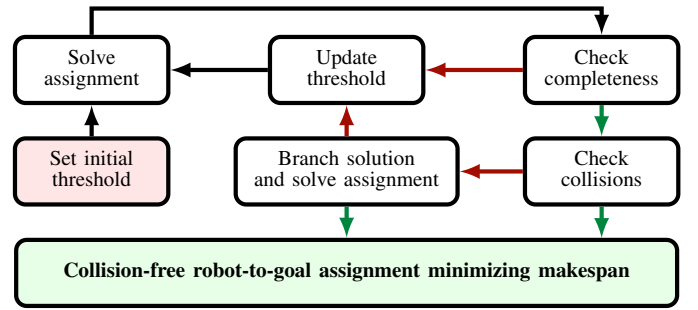

%%{ THE OVERALL ALGORITHM
\begin{algorithm}[!ht]
  \DontPrintSemicolon
  \KwInput{ sets of initial and goal configurations $\mathbb{S}$, $\mathbb{G}$}
  \KwOutput{complete, collision-free assignment $\phi$ from $\mathbb{S}$ to $\mathbb{G}$, minimizing the length of the trajectories}

  \hrulealg
  % \KwData{Testing set $x$}
  $\mathbf{M_d}$, $\mathbf{S_d}$, $\mathbf{G_d}$ $:=$ preprocessData($\mathbb{S}, \mathbb{G}$) \\
  $t_{lb} := $ getThresholdLowerBound($\mathbf{M_d}$) \\
  $\mathbf{T}, t_c := $ initializeThresholds($\mathbf{M_d}$, $t_{lb}$) \\
  $\mathbf{B} := $ initializeBoundedMatrix($\mathbf{M_d}$,$t_c$) \\
  $\mathbf{u}, \mathbf{v} :=$ initializeDualVariables($\mathbf{M_d}$, $\mathbf{B}$) \label{algl:initialize_dual_variables} \\
  $\phi := $ findInitialAssignment($\mathbf{M_d}$, $\mathbf{B}$)\\
  $done$ := false \\

  \While{$\mathbf{not}$ $done$} {

    $\phi :=$ internalHungarian($\phi, \mathbf{M_d}, \mathbf{B}, \mathbf{u}, \mathbf{v}$) \label{algl:internal_hungarian} \\

    $valid := $ isComplete($\phi$) \label{algl:is_complete} \\

    \If{$valid$} {
      \mbox{$c\_edges := $ getCollidingEdges($\phi, \mathbf{M_d}, \mathbf{S_d}, \mathbf{G_d}, \mathbb{S}, \mathbb{G}$)} \\
    \If{$c\_edges$ = $None$} {
    $done$ := true \label{algl:done}\\
    } \Else {
    $\phi :=$ branchSolution($\phi, \mathbf{M_d}, \mathbf{B}, \mathbf{u}, \mathbf{v}$) \label{algl:branch_solution} \\
    $valid := $ isComplete($\phi$) \\
    }
    }

    \If{$\mathbf{not}$ $valid$} {
    $t_c := $ updateThreshold($\mathbf{T}$) \label{algl:update_threshold}\\
    $\mathbf{B}, \mathbf{e}_u := $ updateBoundedMatrix($t_c$) \\
    \mbox{$\mathbf{u}, \mathbf{v} := $ updateMatchingAndDuals($\phi, \mathbf{M_d}, \mathbf{e}_u, \mathbf{u}, \mathbf{v}$)}\label{algl:update_matching_and_duals}\\
    }
  }
  \caption{Algorithm for robot-to-goal assignment considering mutual collision constraints}
  \label{alg:overall_lbap}
\end{algorithm}
%%}

\subsection{Algorithm for robot-to-goal assignment considering mutual collision constraints}\label{sec:lbap_algorithm}
To simplify the description of the proposed robot-to-goal assignment algorithm (\autoref{alg:overall_lbap}), we assume sets of initial and goal configurations $\mathbb{S}$ and $\mathbb{G}$ to be of the same size $|\mathbb{S}| = |\mathbb{G}| = N$, even though this is not strictly required.
The algorithm begins with data preprocessing to get the weighted biadjacency matrix $\mathbf{M}_d \in \mathbb{R}^{N\times N},\,\, m_{ij} = ||\mathbf{s}_i - \mathbf{g}_j||^2$ and the distance matrices $\mathbf{S}_d \in \mathbb{R}^{N\times N},\,\, s_{ij} = ||\mathbf{s}_i - \mathbf{s}_j||^2$, $\mathbf{G}_d \in \mathbb{R}^{N\times N},\,\, g_{ij} = ||\mathbf{g}_i - \mathbf{g}_j||^2$ that store the squared distances of particular start and goal locations for efficient collision checking.

Further steps initialize several variables.
First, the lower bound $t_{lb}$ for a threshold of elements in $\mathbf{M}_d$ that  determines whether edges $e_{ij}$ can be part of the solution is found as an element of $\mathbf{M}_d$ 
\begin{equation}\label{eq:threshold_lower_bound}
  t_{lb} = \max(R_{min}, C_{min}),
\end{equation}
where
\begin{align}
  R_{min} &= \max_{i \in 1, \dots, N} \min_{j \in 1,\dots, N} m_{ij},\\
  C_{min} &= \max_{j \in 1, \dots, N} \min_{i \in 1,\dots, N} m_{ij}.
\end{align}
Next, the list of thresholds $\mathbf{T}$ is formed as a sorted list of elements $m_{ij}$ in $\mathbf{M}_d$ which are greater than $t_{lb}$.
The current threshold $t_c = t_{lb}$ is also applied in initialization of a bounding matrix $\mathbf{B} \in \{0,1\}^{N \times N}$, where
\begin{equation}
  b_{ij} = \left\{
    \begin{array}{ccc}
      0 & \text{if} & m_{ij} \leq t_{c}, \\
      1 & \text{if} & m_{ij} > t_c.\\
    \end{array}\right.
\end{equation}
The bounding matrix $\mathbf{B}$ is used and updated throughout the whole algorithm to limit the maximum cost of an admissible edge and also to exclude the restricted edges, being part of the collision, from the assignment.

As the next step, the vectors of row and column dual variables $\mathbf{u} = \{u_1, \dots, u_N\}$ and $\mathbf{v} = \{v_1, \dots, v_N\}$ are initialized according to the following rule:
\begin{align}
  v_j &= \min_{i \in \{q | b_{qj} = 0,\,\, q \in \{1, \dots, N\}\}} m_{ij},\,\, \forall j \in \{1, \dots, N\}, \\
  u_i &= \min_{j \in \{q | b_{iq} = 0,\,\, q \in \{1, \dots, N\}\}} m_{ij} - v_j,\,\, \forall i \in \{1, \dots, N\}.
\end{align}
This initialization ensures that at least one admissible edge is present in each row and column at the beginning of the algorithm.
The final step preceding the main loop of the algorithm finds an initial assignment by a sequential search for an arbitrary admissible edge that lies in a yet unassigned row and column.
This step is not necessary since the algorithm can start with a valid matching of arbitrary cardinality (including the empty matching), but it decreases the number of required steps in the initial phase of the algorithm.

With the completed initialization, the main loop of the algorithm begins with the \textit{internalHungarian()} procedure (line~\ref{algl:internal_hungarian} of~\autoref{alg:overall_lbap}) detailed in~\autoref{alg:internal_hungarian}.
% The following description is adopted from~\cite{MillsTettey2007TheDH}.
This procedure starts by searching for an augmenting path through growing the Hungarian trees rooted at the unmatched nodes in a current equality subgraph.
If an augmenting path $P$ is found, the matching at step $k$, $\phi_k$ is updated by path $P$ as
\begin{equation}
  \phi_{k+1} = (P - \phi_k) \cup (\phi_k - P).
\end{equation}
Otherwise, the dual variables are updated using the set of nodes encountered in the grown Hungarian trees according to the formula
\begin{align}
  u_i &= \left\{
    \begin{array}{ccc}
      u_i - \theta & \text{if} &i \in \mathbb{H}_r, \\
      u_i + \theta & \text{if} &i \notin \mathbb{H}_r, \\
    \end{array}\right. \forall i \in \{0,\dots,N\},\\
    v_j &= \left\{
      \begin{array}{ccc}
        v_j + \theta & \text{if} &j \in \mathbb{H}_c, \\
        v_j - \theta & \text{if} &j \notin \mathbb{H}_c, \\
      \end{array}\right. \forall j \in \{0,\dots,N\},
\end{align}
where
\begin{equation}\label{eq:theta}
  \theta = \frac{1}{2} \min_{i \notin \mathbb{H}_r, j \in \mathbb{H}_c, b_{ij} = 0} m_{ij} - u_i - v_j, \\
\end{equation}
and $\mathbb{H}_r, \mathbb{H}_c$ are sets of nodes' indices encountered within the Hungarian trees corresponding to the rows and columns of $\mathbf{M}_d$, respectively.

Up to this part, the \textit{internalHungarian()} procedure (\autoref{alg:internal_hungarian}) matches the internal part of the original Hungarian algorithm with the only difference in~\eqref{eq:theta} which excludes the edges restricted by the bounding matrix $\mathbf{B}$.
However, this modification can result in an undefined value of $\theta$, indicating that the assignment problem does not have a solution with the current threshold $t_c$ (line~\ref{algl:is_update_feasible} in~\autoref{alg:internal_hungarian}).
In such a case, the \textit{internalHungarian()} procedure is aborted while keeping the incomplete assignment $\phi$ and updating the values of dual variables $\mathbf{u}, \mathbf{v}$ for later processing inside the main loop of~\autoref{alg:overall_lbap}.

%%{ THE INTERNAL HUNGARIAN ALGORITHM
\begin{algorithm}
  \DontPrintSemicolon
  \caption{internalHungarian($\phi, \mathbf{M_d}, \mathbf{B}, \mathbf{u}, \mathbf{v}$)}
  \label{alg:internal_hungarian}
  \KwInput{matching $\phi$, matrix of squared distances $\mathbf{M_d}$, bounding matrix $\mathbf{B}$ marking the elements exceeding current threshold $t_c$, row and column dual variables $\mathbf{u}, \mathbf{v}$}
  \KwOutput{updated assignment $\phi$ with non-decreased cardinality, updated row and column dual variables $\mathbf{u}$, $\mathbf{v}$}
  \hrulealg
    \While{$\mathbf{not} \text{ isComplete}(\phi)$} {
    $T_h :=$ growHungarianTrees$(\phi,\mathbf{B},\mathbf{u}, \mathbf{v})$\\
    $P :=$ findAugmentingPath$(T_h)$\\
    \If {$P$ = $None$} {
      \If {isUpdateFeasible$(T_h, \mathbf{B})$} { \label{algl:is_update_feasible} 
      updateDualVariables$(T_h, \mathbf{u}, \mathbf{v})$\\
    } \Else {
    \textbf{break}\\
    }
    } \Else {
      $\phi :=$ augmentPath$()$\\
    }

    }

\end{algorithm}
%%}

Once the matching from the \textit{internalHungarian()} procedure is obtained, its completeness is verified (line~\ref{algl:is_complete} of~\autoref{alg:overall_lbap}).
If the matching is not complete, meaning that its cardinality $\text{card}(\phi) < N$, the threshold $t_c$ and bounding matrix $\mathbf{B}$ are updated.
As a result of the Hungarian algorithm on an incomplete graph, the last found matching $\phi$ has the maximum cardinality on a graph excluding edges bounded by $\mathbf{B}$.
Thus, the matching cannot be completed without adding at least $K = N - \text{card}(\phi)$ new edges.
Based on this observation, the current threshold $t_c$ is updated to the lowest value in $\mathbf{T}$ that decreases the number of bounded elements in $\mathbf{B}$ by at least $K$.

The change of the elements in matrix $\mathbf{B}$ corresponds to the modifications of values in the original cost matrix $\mathbf{M}_d$, which requires updating the dual variables to maintain the dual task feasible.
For this purpose, we have adapted the method for updating dual variables in a dynamic (cost-changing) variant of the task assignment problem proposed in~\cite{MillsTettey2007TheDH}.
The \textit{updateMatchingAndDuals()} procedure applied within the proposed algorithm (line~\ref{algl:update_matching_and_duals} of~\autoref{alg:overall_lbap}) is outlined in~\autoref{alg:update_duals}.
After the adaptation of dual variables, the algorithm proceeds to the next run of the \textit{internalHungarian()} procedure (line~\ref{algl:internal_hungarian} of~\autoref{alg:overall_lbap}), starting with the matching of cardinality $\text{card}(\phi_k) \geq \text{card}(\phi_{k-1})$ and a decreased number of bounded elements.

%%{ THE UPDATE DUALS ALGORITHM
\begin{algorithm}
  \DontPrintSemicolon
  \caption{updateMatchingAndDuals($\phi, \mathbf{M_d}, \mathbf{e}_u, \mathbf{u}, \mathbf{v}$)}
  \label{alg:update_duals}
  \KwInput{matching $\phi$, matrix of squared distances $\mathbf{M_d}$, set of updated edges $\mathbf{e}_u$, row and column dual variables $\mathbf{u}, \mathbf{v}$}
  \KwOutput{updated assignment $\phi$, updated row and column dual variables $\mathbf{u}$, $\mathbf{v}$}
  \hrulealg
    \For {$e_{ij} \in \mathbf{e}_u$ } {
    \If {$m_{ij} < u_i + v_j$} {
    $u_i = \min_{k \in \{1, \dots, N\}} m_{ik} - v_k $ \\
    \If { $e_{ij} \notin \phi$} {
    $\phi := \phi \backslash \{e_{ik}, k \in \{1, \dots, N\}\}$ \\
    }
    } \Else {
    $\phi := \phi \backslash e_{ij}$ \\
    }
    }

\end{algorithm}
%%}

If the matching found by \textit{internalHungarian()} procedure is perfect, it is tested for the existence of colliding edges using a combination of the results derived in~\autoref{sec:lbap_theoretical_analysis} for evaluation of the majority of the potential collisions, and the precise collision checking using~\eqref{eq:squared_dist_x}.
The collision check is done over all pairs of edges in the perfect matching $\phi$.
The collision check of edges $e_{ij}, e_{kl} \in \phi$ starts with the evaluation of
\begin{equation}\label{eq:collision_check_initial_condition}
  \begin{split}
    \text{colide}(e_{ij}, e_{kl}) = &(m_{ij} + m_{kl}) \geq (m_{il} + m_{kj}) \land \\
    &\max(m_{ij}, m_{kl}) > M^2 \max(m_{il}, m_{kj}),
  \end{split}
\end{equation}
where the value of $M$ is set based on the value of $\delta$ and the minimum allowed mutual distance $\Delta$ using~\eqref{eq:collision_free_condition}.
The first part of the condition rejects the risk of potential collision by detecting the equivalence with the \ac{lsap} solution with proven guarantees on the minimum distance of trajectories~\cite{turpin2014CaptConcurrentAssignment} while the second part eliminates the collisions using~\autoref{th:minimum_dist_lbap}.

Since the condition from~\autoref{th:minimum_dist_lbap} represents the lower bound on a minimum mutual distance, we further apply the exact computation of a minimum distance to avoid false positive detections of collisions if $\text{collide}(e_{ij}, e_{kl}) = \text{true}$.
Thus, if the condition~\eqref{eq:collision_free_condition} is not met, equation~\eqref{eq:squared_dist_x} is applied for an exact computation of the minimum mutual distance of the trajectories being compared to the minimum acceptable distance $\Delta$.
In the case that there is no pair of colliding edges in a perfect matching $\phi$, the algorithm terminates and returns $\phi$ as a complete assignment from $\mathbb{S}$ to $\mathbb{G}$, minimizing the maximum length of the trajectory while fulfilling the condition on collision-free assignment with constant-velocity trajectories (line~\ref{algl:done} of~\autoref{alg:overall_lbap}).

If a collision is detected, the \textit{branchSolution()} procedure (\autoref{alg:branch_solution}) is started to ensure that the algorithm explores all possibly collision-free matchings for a current threshold $t_c$ before increasing its value and making new elements of $\mathbf{M}_d$ feasible (line~\ref{algl:branch_solution} of~\autoref{alg:overall_lbap}).
The proposed method is based on the depth-first search algorithm performed on a binary tree graph formed by nodes defined by matching $\phi$, pair of colliding edges $\mathbf{e}_c$, and vectors of row and column dual variables $\mathbf{u}, \mathbf{v}$.
Note that, the \textit{branchSolution()} method is used to find any collision-free solution with current threshold $t_{c}$ that defines the optimal value.   
Thus, the non-optimality of the depth-first search does not influence the optimality of the presented algorithm.    

The binary tree, rooted at a node corresponding to initial perfect matching, is iteratively built during the depth-first search by expanding the parent node according to the following expansion rule.
The parent node $N_p = \{\phi_p, e_{ij}, \mathbf{u}_p, \mathbf{v}_p\}$ with a maximum matching $\phi_p$, restricted edge $e_{ij}$, and dual variables $\mathbf{u}_p, \mathbf{v}_p$ is, in the case of detected colliding edges $\mathbf{e}_c = \{e_{kl}, e_{op}\}$, expanded into two child nodes derived from the task assignment problem of the parent node by adding a single bounded edge and updating dual variables correspondingly.
Thus, the expansion of a node $N_p$ results in new nodes given by
\begin{equation}\label{eq:expansion_rule}
  \begin{split}
    N_{c1} &= \{\phi_p, e_{kl}, \mathbf{u}_p, \mathbf{v}_p\},\\
    N_{c2} &= \{\phi_p, e_{op}, \mathbf{u}_p, \mathbf{v}_p\}.\\
  \end{split}
\end{equation}
Note that, the matching and dual variables of particular nodes are always updated during the \textit{internalHungarian()} procedure before the child nodes are derived from them.

In every iteration of the \textit{branchSolution()} procedure, a node $\mathbf{n}_c = \{\phi_c, e_{ij}, \mathbf{u}_c, \mathbf{v}_c\}$ is dequeued from the \ac{lifo} queue, the bounding matrix $\mathbf{B}$ is updated with the newly restricted edge $e_{ij}$ and the corresponding dual variables and matching are updated using \textit{updateMatchingAndDuals()} (\autoref{alg:update_duals}).
After that, the \textit{internalHungarian()} is run to find a perfect matching for an updated assignment problem (line~\ref{algl:branch_solution_ih} of~\autoref{alg:branch_solution}).
Since the newly restricted edge $e_{ij}$ is always part of an initial parent matching, the cardinality of the matching $\phi$ after an update is always $\text{card}(\phi) = N-1$.
As mentioned earlier, given the matching of cardinality $N-1$ and the corresponding dual variables, the internal Hungarian algorithm terminates after a single step with either a perfect matching (if an augmenting path exists) or an incomplete matching.
This fact is important for keeping the computational complexity of the proposed algorithm low.

If a perfect matching is not found by the \textit{internalHungarian()} method, the solution to an assignment problem with a set of bounded edges given by $\mathbf{B}$ does not exist.
Consequently, it is easy to show that this situation cannot be improved by restricting additional edges. 
Hence, we cannot get a valid solution by expanding such a node and can proceed to the next iteration.
If the computed solution is a perfect matching, it has to be examined whether it is collision-free.
In case a pair of colliding edges is not found, the perfect matching is returned as a valid collision-free solution to the main loop of~\autoref{alg:overall_lbap}.
Otherwise, the node is expanded according to the expansion rule~\eqref{eq:expansion_rule}, inserted into the queue, and the algorithm proceeds to the next iteration (line~\ref{algl:branch_solution_expand} of~\autoref{alg:branch_solution}).
In case all branches of the tree were explored without finding a valid, complete solution, the procedure returns to the main loop of~\autoref{alg:overall_lbap}, where the value of the current threshold is updated (line~\ref{algl:update_threshold}), and the search for a solution continues.
In the main loop, the algorithm repeats the above-described steps until a valid, complete solution is found.
The valid solution is guaranteed to exist under the assumptions specified in~\autoref{sec:problem_definition}.

Since the introduced algorithm iteratively increases the threshold on bounded edges, the number of bounded edges decreases and the problem becomes less restricted.
In a worst-case scenario, the algorithm reaches a point where none of the edges are bounded, meaning that also none of the edges from solution $\phi_{LSAP}$ minimizing the sum of squared costs are bounded. 
Then, $\phi_{LSAP}$ is an output of the internal Hungarian algorithm.
Since the $\phi_{LSAP}$ solution is guaranteed to be collision-free with given assumptions, a valid collision-free matching is always found.

%%{ THE BRANCH AND SOLVE ALGORIT}HM
\begin{algorithm}
  \DontPrintSemicolon
  \caption{branchSolution($\mathbf{e}_c, \phi, \mathbf{M_d}, \mathbf{B}, \mathbf{u}, \mathbf{v}, \mathbb{S}, \mathbb{G}, \mathbf{S}_d, \mathbf{G}_d$)}
  \label{alg:branch_solution}
  \KwInput{colliding edges $\mathbf{e}_c$, matching $\phi$, matrix of squared distances $\mathbf{M_d}$, bounding matrix $\mathbf{B}$, row and column dual variables $\mathbf{u}, \mathbf{v}$, initial and goal configurations $\mathbb{S}$, $\mathbb{G}$, matrices with squared distances of initial and goal configurations $\mathbf{S}_d, \mathbf{G}_d$}
  \KwOutput{complete collision-free assignment $\phi_{new}$ if it exists, original assignment otherwise}
  \hrulealg
    $\mathbf{O}_l := \O$ \tcp*{Last in first out queue}
    Node $root := \{\phi, \mathbf{e}_c, \mathbf{u}, \mathbf{v}\}$\\
    $\mathbf{O}_l \leftarrow$ expand$(root, \phi, \mathbf{e}_c)$\\
    \While {$\mathbf{O}_l \neq \O$} {
    $\mathbf{n}_c :=$ dequeue$(\mathbf{O}_l)$\\
    updateRestrictedNodes$(\mathbf{M_d}, \mathbf{B}, \mathbf{n}_c)$\\
    updateMatchingAndDuals$(\mathbf{n}_c.\phi, \mathbf{M_d}, \mathbf{n}_c.e_{ij}, \mathbf{u}, \mathbf{v})$\\
    $\phi_{new} :=$ internalHungarian$(\mathbf{n}_c.\phi, \mathbf{M_d}, \mathbf{B}, \mathbf{u}, \mathbf{v})$ \label{algl:branch_solution_ih}\\
    \If {isComplete$(\phi_{new})$} {

      $\mathbf{e}_c :=$  getCollidingEdges$(\phi_{new},\mathbf{M_d}, \mathbf{S_d}, \mathbf{G_d}, \mathbb{S}, \mathbb{G}$)\\
    \If{$\mathbf{e}_c$ = $None$} {
      \Return $\phi_{new}$ \tcp*{solution found}
    } \Else {
      $\mathbf{O}_l \leftarrow $expand$(\mathbf{n}_c, \phi_{new}, \mathbf{e}_c)$} \label{algl:branch_solution_expand}
    }
    }
    \Return $\phi$ \tcp*{solution not found}
\end{algorithm}
%%}

%%}

%%{ SECTION: Trajectory generation

% ==============================================================================
\section{Minimum-makespan trajectory generation}\label{sec:trajectory_generation}
% ==============================================================================

The algorithm designed for the solution of the \ac{lbap} introduced in~\autoref{sec:lbap_algorithm} guarantees to solve part of \ac{tofrp} by finding an assignment minimizing the length of the longest path with additional guarantees on the collision-free property of the constant-velocity trajectories.
In this section, we describe a second component of \ac{catora} that allows us to generate a set of trajectories connecting the pairs of assigned positions, while minimizing the makespan of the reshaping process and preserving the conditions on the collision-free property derived in~\autoref{sec:lbap_theoretical_analysis}.
% In the following sections, we consider a model with single-dimension point-mass dynamics $\ddot{p} = a$, with constraints on control inputs $-a_{max} < a < a_{max}$, and bounded velocity $v = \dot{p}$, $0 < v < v_{max}$.

\subsection{Minimum-time trajectory generation}
The time-optimal control of a model with single-dimension point-mass dynamics and constraints on maximum velocity results in a control policy of form
\begin{equation}\label{eq:time_optimal_control_policy}
  a^*(t) = \left\{
    \begin{array}{ccc}
      a_{max} & \text{if} & t \leq t_1,\\
      0 & \text{if} & t_1 < t \leq t_2,\\
      -a_{max} & \text{if} & t_2 < t \leq t_3,
    \end{array}\right.
\end{equation}
where $t_3 = t^*$ is the overall minimized time of trajectory following~\cite{GRAVELL2021104753}.
The control policy~\eqref{eq:time_optimal_control_policy} leads to trajectories that are described by equations
\begin{equation}\label{eq:optimal_control_kinematics}
  \begin{split}
    p_1 &= p_0 + v_0t_1 + \frac{1}{2} a_{max}t_1^2,\\
    v_1 &= v_0 + a_{max}t_1,\\
    p_2 &= p_1 + v_1(t_2 - t_1),\\
    v_2 &= v_1,\\
    p_3 &= p_2 + v_2(t_3 - t_2) - \frac{1}{2}a_{max}(t_3 - t_2)^2,\\
    v_3 &= v_2 - a_{max}(t_3 - t_2),
  \end{split}
\end{equation}
with velocities $v_i$ and positions $p_i, i \in \{0, 1, 2, 3\}$.
With the known initial and final conditions of $p_0, v_0, p_3, v_3$, the maximum acceleration $a_{max}$, and the assumption on reachability of the maximum velocity $v_1 = v_{max}$, the number of equations~\eqref{eq:optimal_control_kinematics} matches the number of unknown variables ($t_1, t_2, t_3, p_1, p_2, v_2$), and~\eqref{eq:optimal_control_kinematics} has a closed-form solution.
By the addition of an assumption that $v_0 = v_3 = 0$ and its consequence $t_1 = t_3 - t_2$,~\eqref{eq:optimal_control_kinematics} can be modified to
\begin{equation}\label{eq:optimal_control_kinematics_v_const}
  \begin{split}
    p_1 &= p_0 + \frac{1}{2} a_{max}t_1^2,\\
    v_1 &= a_{max}t_1,\\
    p_2 &= p_1 + v_1(t_2 - t_1),\\
    v_2 &= v_1,\\
    p_3 &= p_2 + v_2(t_3 - t_2) - \frac{1}{2}a_{max}(t_3 - t_2)^2,\\
    0 &= v_2 - a_{max}(t_3 - t_2),\\
    t_1 &= t_3 - t_2.
  \end{split}
\end{equation}
With an additional seventh equation, the modified set of equations~\eqref{eq:optimal_control_kinematics_v_const} allows for relaxing the condition $v_1 = v_{max}$ to $v_1 \leq v_{max}$, thus providing a single closed-form solution valid even when maximum velocity cannot be reached, and the optimal control policy reduces to bang-bang control~\cite{Penicka22UAVplanning}.

%%}

%%{ Minimum time formation reshaping

\subsection{Trajectories for minimum-time formation reshaping}
By applying time optimal control policy~\eqref{eq:time_optimal_control_policy}, the time necessary for following the trajectory is directly proportional to the length of the trajectory.
Thus, the makespan of the formation reshaping process is equal to the time $t_m$ of following the trajectory $T_m$ corresponding to the longest path $P_m$ obtained in the assignment $\phi$
\begin{equation}\label{eq:max_path_length}
  ||P_m|| = D_m = \max_{(i,j) \in \phi} ||\mathbf{s}_i - \mathbf{g}_j||.
\end{equation}
From the solution of~\eqref{eq:optimal_control_kinematics_v_const} for the longest path $P_m$, we obtain the duration $t_1^m, t_2^m$ and $t_3^m$ of acceleration, constant speed, and deceleration segments, respectively.
Considering the duration of particular segments $t_1, t_2, t_3$ to be constant and equal to $t_1^m, t_2^m, t_3^m$ for all trajectories, \eqref{eq:optimal_control_kinematics_v_const} can be applied for the generation of the rest of trajectories with defined $t_1, t_2, t_3$, but varying $|a_i| \leq a_{max}$ and $v_1 \leq v_{max}$.
Such an approach results in trajectories defined by parametrization:
  \begin{equation}\label{eq:min_time_parametrization}
    \hspace*{-0.563cm}\mathbf{x}_i(s) = \begin{cases}
      \mathbf{s}_i + \left(\frac{A_mt_3^2 }{2D_m}s^2\right) (\mathbf{g}_j - \mathbf{s}_i), & \hspace*{-1.8cm}\mathrm{if}\ 0 \leq s \leq \frac{t_1}{t_3} \\
      \mathbf{s}_i + \left(\frac{V_mt_3 }{D_m}s-\frac{V_m^2}{2A_mD_m} \right) (\mathbf{g}_j - \mathbf{s}_i), & \hspace*{-1.8cm}\mathrm{if}\ \frac{t_1}{t_3} < s \leq \frac{t_2}{t_3} \\
      \mathbf{s}_i + \left(\frac{A_mt_3^2}{2D_m}\left(2s - s^2\right)-\frac{V_m^4+A_m^2D_m^2}{2A_mD_mV_m^2}\right)(\mathbf{g}_j - \mathbf{s}_i),& \\
      & \hspace*{-1.8cm} \mathrm{if}\  \frac{t_2}{t_3} < s \leq 1 \\
    \end{cases}
  \end{equation}
where $V_m$ and $A_m$ stand for maximum applied velocity and maximum applied acceleration of $T_m$, respectively, and $s = \frac{t}{t_m}$ with $t$ being time elapsed from start of the trajectory. 
Trajectories generated according to parametrization~\eqref{eq:min_time_parametrization} are illustrated in~\autoref{fig:minimum_time_trajectories}.

\begin{figure}[htpb]
  \centering
  \input{fig/graphs/minimum_time_trajectories.tex}
  \vspace*{-0.5cm}
  \caption{Acceleration and velocity profiles and position progress along the path for trajectory $T_j$ with length $D_j = D_m$ generated using time-optimal control policy~\eqref{eq:time_optimal_control_policy} (red), and for trajectory $T_i$ with length $D_i \leq D_m$ generated according to parametrization~\eqref{eq:min_time_parametrization} (blue). The background color distinguishes the acceleration (green), constant speed (white), and deceleration (red) segments of the trajectories.}
  \label{fig:minimum_time_trajectories}
\end{figure}
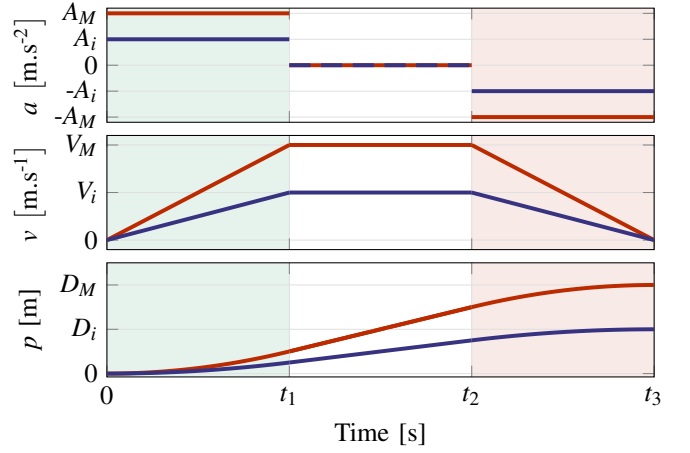

Let us define the progress ratio at time $t > 0$ for a pair of trajectories $T_i, T_j$ of lengths $D_i > 0$, $D_j > 0$, as
\begin{equation}
  PR(T_i, T_j, t) = \frac{p_i(t)}{p_j(t)},
\end{equation}
where $p_i(t), p_j(t)$ are the distances traveled along the trajectories $T_i, T_j$ till time $t$.

\begin{theoremwobrackets}\label{th:progress_ratios}
  If the progress ratios $PR_a = PR(T_{i,a},T_{j,a})$ and $PR_b = PR(T_{i,b},T_{j,b})$ of two parametrizations $a$ and $b$ of a single pair of paths are constant for all $t$ and $PR_a = PR_b$, the following relation holds:
  \begin{equation}
    \min_t||T_{i,a}(t) - T_{j,a}(t)|| = \min_t||T_{i,b}(t) - T_{j,b}(t)||.
  \end{equation}
\end{theoremwobrackets}
\begin{proof}
  The assumption $PR_a = PR_b = const.$ in~\autoref{th:progress_ratios} can be reformulated to equation
  \begin{equation}\label{eq:pr_2}
    \frac{p_{i,a}(s_a t_a)}{p_{j,a}(s_a t_a)} = \frac{p_{i,b}(s_b t_b)}{p_{j,b}(s_b t_b)},\,\, \forall s_a, s_b \in (0, 1],
  \end{equation}
  where $t_a$ and $t_b$ are the duration of trajectories $T_{i,a},T_{j,a}$ and $T_{i,b},T_{j,b}$, respectively, and $s_a,s_b$ are independent variables.
  Equation~\eqref{eq:pr_2} can be simplified to
  \begin{equation}\label{eq:pr_proof}
    p_{j,b}(s_b) = p_{j,a}(s_a) \frac{p_{i,a}(s_a)}{p_{i,b}(s_b)}.
  \end{equation}
  From~\eqref{eq:pr_proof}, it follows that for an arbitrary $p_{i,a}(s_a)$ with associated point $p_{j,a}(s_a)$, it holds that
  \begin{equation}
    p_{i,b}(s_b) = p_{i,a}(s_a) \Leftrightarrow p_{j,a}(s_a) = p_{j,b}(s_b).
  \end{equation}
  Thus, the set of corresponding points on particular trajectories is equal for both pairs of trajectories, and their minimum distances are equal.
\end{proof}

By applying~\autoref{th:progress_ratios}, it can be shown that the progress ratio of trajectories with a constant velocity~\eqref{eq:parametrization_const_velocity} and minimum-time trajectories defined by~\eqref{eq:min_time_parametrization} are equal and constant for all parts of parametrization~\eqref{eq:min_time_parametrization}.
It follows that the minimum-time trajectories defined by~\eqref{eq:min_time_parametrization} have the same properties that were derived for trajectories with constant velocity.
Thus, the pair of minimum-time trajectories are guaranteed to be collision-free if $\delta \geq \frac{\sqrt{1-M^2}}{1-M^2} \Delta$ and the \ac{catora} consisting of application of~\autoref{alg:overall_lbap} together with the trajectory generation approach~\eqref{eq:min_time_parametrization} provides a set of minimum-time collision-free trajectories as a solution to \ac{tofrp}.

\begin{remark}
  The presented methodology and \autoref{th:progress_ratios} are not limited to time-optimal control of a simplified single-dimension point-mass dynamics as presented in this section.
  The set of trajectories with constant progress ratios (thus collision-free for assignment based on \autoref{alg:overall_lbap}) can be generated for an arbitrary single trajectory generation approach with arbitrary complex motion model, including, e.g., focus on minimum-energy trajectory generation.
\end{remark}
%%}

%%{ SECTION: THeoretical and Statistical Results

% ==============================================================================
\section{Theoretical and statistical analysis}\label{sec:theoretical_results}
% ==============================================================================

In this section, we provide proof of the optimality of the proposed algorithm and state and prove several theorems that highlight the significant benefits of the proposed approach, and support the adequacy of the stated assumptions.

\subsection{The independence of robot-to-goal assignment on trajectory generation approach}\label{sec:proof_of_solution_independence}

The proposed decoupled solution to \ac{tofrp} is optimal under an assumption that the robot-to-goal assignment problem and minimum-time trajectory generation are separable~\eqref{eq:assumption_separation}.
If~\eqref{eq:assumption_separation} holds for all assignments and trajectories generated as described in~\autoref{sec:trajectory_generation}, then the duration of trajectories can be replaced by Euclidean distances in the computation of robot-to-goal assignment without compromising the optimality of the solution.

The proof of~\eqref{eq:assumption_separation} follows directly from~\eqref{eq:max_path_length}, \eqref{eq:min_time_parametrization} and properties of applied control policy~\eqref{eq:time_optimal_control_policy}.
Based on~\eqref{eq:max_path_length}, \eqref{eq:min_time_parametrization}, the duration of all trajectories for a single assignment depends only on the length of the longest path in the assignment.
Given the assumptions on stationary initial and goal configurations (A2), the generated trajectories have identical initial and final velocities. 
In such a case, the duration of trajectories generated using time-optimal control policy~\eqref{eq:time_optimal_control_policy} is a monotonic, increasing function of path length, ensuring the validity of assumption~\eqref{eq:assumption_separation}. 
 
\subsection{Optimality of the robot-to-goal assignment of \ac{catora}}\label{sec:completeness_and_optimality}

\begin{proof}[\unskip\nopunct]

The proof of the optimality of the proposed algorithm is based on the following observations: 
\begin{enumerate}[leftmargin=0.85cm]
  \item [(B1)] The Hungarian algorithm, and thus also the \textit{internalHungarian()} procedure, is optimal (proved in~\cite{munkresHungarianAlgorithm}).
  \item [(B2)] The dynamic variant of the Hungarian algorithm is optimal (proved in~\cite{MillsTettey2007TheDH}).
  \item [(B3)] The bounding of elements $m_{ij} > t_c$ of a cost matrix $\mathbf{M}_d$ by bounding matrix $\mathbf{B}$ is equivalent to substituting constant $Q = \sum_{m_{ij} \in \mathbf{M}_d}[m_{ij} \leq t_c] m_{ij}$ for all elements $m_{ij} > t_c$.
  \item [(B4)] If an element $m_{ij}$ used in updating dual variables~\eqref{eq:theta} is bounded, all unbounded elements of $M_d$ are already admissible, and the cardinality of the current matching cannot be increased without using bounded elements (comes from the properties of the Hungarian algorithm).
  \item [(B5)] By the addition of $k$ edges, the cardinality of the maximum matching $\phi$ can be increased by at most $k$. Thus, the matching $\phi$ cannot be completed without the addition of at least $k = N-\text{card}(\phi)$ new edges (comes from the properties of the Hungarian algorithm).
  \item [(B6)] The robot-to-goal assignment algorithm checks all complete solutions for collisions using a combination of an efficient analytical method and an exact analytical method (both derived in \autoref{sec:lbap_theoretical_analysis}). This approach ensures that any complete solution provided by the algorithm is collision-free with no false positives.
\end{enumerate}

  The proof of completeness follows directly from (B1) and (B2).  
By omitting procedures that do not change any variables in the main loop of~\autoref{alg:overall_lbap}, the algorithm reduces to a dynamic variant of the Hungarian algorithm, solving the assignment problem with iterative change of costs caused by updates of threshold $t_c$.    
Based on (B3), bounded elements only influence the update of duals once they are smaller than $t_c$.
The threshold is updated until a valid collision-free solution is found, eventually ending with $b_{ij} = 0\,\,\forall (i,j) \in \{1,\dots,N\}^2$.
If no elements of $M_d$ are bounded, the solution exists according to assumption (A5).
Then, in compliance with (B2), the solution is found, proving the algorithm's completeness. 

The proof of optimality is built using the fact that bounded values cannot be part of the solution, and thus the optimal value of the solution is bounded by $t_c$.
  Therefore, given guarantees of exact collision checking (B6), it is sufficient to show that the threshold $t_c$ is increased only if no valid solution exists with the current threshold.
  In~\autoref{alg:overall_lbap}, the initial threshold~\eqref{eq:threshold_lower_bound} is set to the maximum of the minimum elements across particular rows and columns. Since the perfect matching must include at least one element from each row and column, the initial lower bound does not exceed the optimal value.   
  Further, we branch the proof into two cases: (i) \autoref{alg:overall_lbap} never detects a colliding edge or (ii) \autoref{alg:overall_lbap} detects a colliding edge.

  In case (i), the algorithm alternates between applying the \textit{internalHungarian()} procedure and updating the threshold $t_c$.
Based on (B1) and (B4), the \textit{internalHungarian()} procedure always finds a maximum matching with respect to the current bounded matrix.
  If a perfect matching is found with the current threshold, then the optimal solution has been achieved; otherwise, $t_c$ is updated.
According to (B5), increasing $t_c$ to the lowest value that decreases the number of bounded elements in $\mathbf{B}$ by $k = N - \text{card}(\phi)$, cannot increase the threshold above the value of the optimal solution.
  Then, for case (i), the procedure mirrors the dynamic Hungarian algorithm~\cite{MillsTettey2007TheDH} with the costs changed by updates of $t_c$. 
  Therefore, the optimality guarantees for case (i) follow from (B2) and the validity of threshold updates given by (B5). 

  In case (ii), we have to prove that the \textit{branchSolution()} procedure is complete.
  The constraint on mutually colliding edges restricts edges $e_r, e_s$ so that only one can be included in the solution.
The \textit{branchSolution()} procedure exploits this constraint by creating a binary search tree where each branch is derived from a parent node by restricting exactly one edge from the colliding pair.
This effectively splits the original problem into two instances: one where solutions exclude $e_r$, and the second where solutions exclude $e_s$, ensuring no valid solution is missed.
  During the search for the solution, each node represented by the instance of an assignment problem is evaluated using the \textit{internalHungarian()} and potentially \textit{getCollidingEdges()} procedure.
  For each node, the evaluation can yield three outcomes depending on corresponding set of restricted edges: (i) the found matching is not perfect; (ii) the found matching is perfect, but contains colliding edges; or (iii) the found matching is perfect and collision-free.  
  If a perfect matching is not found, it is guaranteed not to exist ((B1), (B2)), and thus this branch of the solution does not have to be explored further since the restriction of an additional edge cannot lead to an increase in maximum cardinality.
If a perfect matching is found and it contains a pair of colliding edges, the node is split into two and further explored.
If the found matching is perfect and does not contain any colliding edge, it is bounded by $t_c$, and is thus optimal with respect to~\eqref{eq:lbap}, \eqref{eq:lbap_with_mutual_constraints}. 

  Since we have proven the optimality of the algorithm for the solution of \ac{lbap} with mutual collision constraints, along with the independence of robot-to-goal assignment on minimum-time trajectory generation approach, it can be concluded that the \ac{catora} is an optimal algorithm for the solution of \ac{tofrp}.

\end{proof}

\subsection{Comparison of \ac{lbap} and \ac{lsap} in terms of maximum path length}

The superior performance of the \ac{lbap}-based approach to robot-to-goal assignment in terms of the length of the longest path is evident from the \ac{lsap} and \ac{lbap} problem formulation.
With \autoref{th:bound_max_len} introduced and proven in this section, we provide an insight into the significance of this phenomenon, and thus also the benefit of solving robot-to-goal assignment as \ac{lbap} instead of \ac{lsap}. 
The theorem shows that the \ac{lsap} solution can produce an up to $1.7$-times longer longest path compared to the \ac{lbap} solution already for small instances of 3 robots.            
This ratio further grows with the squared root of a number of robots, reaching a ratio of 10 for instances with 100 robots.      

\begin{theoremwobrackets}\label{th:bound_max_len}
  The upper bound on the ratio between the maximum length of the path in \ac{lsap} assignment $\phi_s$ and \ac{lbap} assignment $\phi_b$ is $\sqrt{N}$, where $N$ is the number of goals.
  The lower bound on this ratio equals 1.
  Thus, the following equation holds
  \begin{equation}
    1 \leq \frac{\max_{(i, j) \in \phi_s} ||\mathbf{s}_i - \mathbf{g}_j||}{\max_{(i, j) \in \phi_b} ||\mathbf{s}_i - \mathbf{g}_j||} \leq \sqrt{N}.
  \end{equation}
\end{theoremwobrackets}

\begin{proof}
  The proof of the first inequality directly follows from the optimization criterion and properties of LBAP~\eqref{eq:lbap}.
  The \ac{lbap} solution directly minimizes the maximum length of the path in the assignment, hence $\max_{(i, j) \in \phi_b} ||\mathbf{s}_i - \mathbf{g}_j|| \leq \max_{(i, j) \in \phi_a} ||\mathbf{s}_i - \mathbf{g}_j||$ holds for optimal solution  to \ac{lbap} represented by assignment $\phi_b$ and any valid solution optimizing arbitrary criteria (including \ac{lsap}) represented by assignment $\phi_a$, 
  thus, $1 \leq \frac{\max_{(i, j) \in \phi_s} ||\mathbf{s}_i - \mathbf{g}_j||}{\max_{(i, j) \in \phi_b} ||\mathbf{s}_i - \mathbf{g}_j||}$.
  The lower bound is achievable since if no edge can be removed from the \ac{lsap} solution without rendering the problem unfeasible, the solutions of \ac{lsap} and \ac{lbap} coincide and thus $\max_{(i, j) \in \phi_b} ||\mathbf{s}_i - \mathbf{g}_j|| = \max_{(i, j) \in \phi_s} ||\mathbf{s}_i - \mathbf{g}_j||$.

  The second inequality can be proved by finding a solution to the optimization problem
    \begin{align}
      \underset{\mathbb{S}, \mathbb{G}}{\mathrm{maximize}}\, &
    \frac{\max_{(i, j) \in \phi_s(\mathbb{S}, \mathbb{G})} ||\mathbf{s}_i - \mathbf{g}_j||}{\max_{(i, j) \in \phi_b(\mathbb{S}, \mathbb{G})} ||\mathbf{s}_i - \mathbf{g}_j||}\label{eq:max_ratio_opt_problem}\\
      \text{subject to} &\sum_{(i, j) \in \phi_s} ||\mathbf{s}_i - \mathbf{g}_j||^2 \leq \sum_{(i, j) \in \phi_b} ||\mathbf{s}_i - \mathbf{g}_j||^2,\label{eq:max_ratio_constraint}
    \end{align}
  where~\eqref{eq:max_ratio_constraint} is a necessary condition for a validity of \ac{lsap} solution coming from its formulation~\eqref{eq:lsap}.
  Given the independence of the expressions in the numerator and denominator, the maximum of~\eqref{eq:max_ratio_opt_problem} can be found by independent maximization of numerator and minimization of denominator while considering constraint~\eqref{eq:max_ratio_constraint}.
The constraint~\eqref{eq:max_ratio_constraint} can be simplified to its least constrained form
  \begin{equation}\label{eq:max_ratio_constraint_simplified}
    \left(\max_{(i, j) \in \phi_s} ||\mathbf{s}_i - \mathbf{g}_j||\right)^2 \leq N\left(\max_{(i, j) \in \phi_b} ||\mathbf{s}_i - \mathbf{g}_j||\right)^2,
  \end{equation}
  using following observations:
  \begin{enumerate}
    \item The paths shorter than $\max_{(i, j) \in \phi_s} ||\mathbf{s}_i - \mathbf{g}_j||$ do not influence the value of numerator, but increase the value of the left side of~\eqref{eq:max_ratio_constraint}. In the extreme case, this leads to a set of paths with zero length except for a single path in the set.   
    \item The paths shorter than $\max_{(i, j) \in \phi_b} ||\mathbf{s}_i - \mathbf{g}_j||$ do not influence the value of the denominator, but increase the value of the right side of~\eqref{eq:max_ratio_constraint}. In the extreme case, this leads to a set of paths with equal lengths.
  \end{enumerate}
  The least constrained form~\eqref{eq:max_ratio_constraint_simplified}
  can be further reformulated to
  \begin{equation}\label{eq:max_ratio_constraint_simplified_two}
    \frac{\max_{(i, j) \in \phi_s} ||\mathbf{s}_i - \mathbf{g}_j||}{\max_{(i, j) \in \phi_b} ||\mathbf{s}_i - \mathbf{g}_j||} \leq \sqrt{N},
  \end{equation}
  which directly forms the upper bound on the examined quantity.
  The achievability of the upper bound can be proved by construction. The representative example for which the equality in~\eqref{eq:max_ratio_constraint_simplified_two} for arbitrary $N$ holds is formed by a scenario where $\mathbb{S} = \{\mathbf{v}_0, \dots, \mathbf{v}_{N}\} \subset \mathbb{V}$, $\mathbb{G} = \{\mathbf{v}_1, \dots, \mathbf{v}_{N+1}\} \subset \mathbb{V}$ with $\mathbb{V} = \{\mathbf{v}_0, \dots, \mathbf{v}_{N+1}\}$ being set of vertices of a closed polygonal chain for which holds:
  \begin{equation} 
    \begin{split}
      ||\mathbf{v}_{i-1} - \mathbf{v}_i|| &= K,\,\, K \in \mathbb{R}^+,\,\, \forall i \in \{1, \dots, N+1\}, \\
      ||\mathbf{v}_{N+1} - \mathbf{v}_0|| &= \sqrt{N}K, \\
      ||\mathbf{v}_{i} - \mathbf{v}_j|| &\geq K\,\, \forall {i, j} \in \{1, \dots, N+1\}^2,\,\, i \neq j,\\
    \end{split}
  \end{equation}
  Given an environment with unconstrained dimensions, the introduced scenario can be constructed for arbitrary $N$.
\end{proof}

\subsection{Study on suboptimality of the solution to \ac{tofrp} neglecting assumption on straight paths}

The optimality of the \ac{catora} for the solution of \ac{tofrp} is proved in \autoref{sec:completeness_and_optimality}.
However, the assumption (A3) discriminates the use of collision resolution techniques, such as time delays and geometric modifications of paths~\cite{GRAVELL2021104753, honig2018TrajectoryPlanningQuadrotor}.
Although the use of such techniques mostly leads to a significant increase in the computational complexity of the algorithm, they can resolve some collisions that are unsolvable by the proposed algorithm without increasing the threshold $t_c$.   
Thus, neglecting the assumption (A3) can change the optimum value of \ac{tofrp}.

We compare the achieved optimum value of the proposed algorithm with a theoretical lower bound of \ac{tofrp} (see \autoref{theorem:lb_tofrep}) to analyze the gap between the optimum value of the solution while both considering and not considering assumption (A3).
The results presented in~\autoref{fig:lower_bound_comparison} show that the \ac{catora} yields an optimal solution equal to the theoretical lower bound (neglecting assumption (A3)) in over $95\%$ of instances in dense environments with an average suboptimality $s_{avg} = 1.0008$ and maximum suboptimality $s_{max} = 1.16$. 

\begin{theoremwobrackets}[]\label{theorem:lb_tofrep} 
  The lower bound on the solution of \ac{tofrp} without limitations imposed by assumption (A3) is given by the duration of the minimum-time trajectory that corresponds to the longest path in the robot-to-goal assignment, as obtained by the solution of \ac{lbap} without considering mutual collisions.
\end{theoremwobrackets}

\begin{proof}
  Considering an assumption on static initial and goal configuration, the optimum value of \ac{tofrp} equals the duration of minimum-time trajectory along the longest path among all reshaping paths. 
Thus, minimizing the length of the longest path among all reshaping paths optimizes the original problem. 
By applying any technique to resolve the collisions among trajectories, the optimum value remains the same or increases.
  Thus, the solution of \ac{lbap} together with the minimum-time trajectory generation forms a lower bound to \ac{tofrp}. 
\end{proof}

\begin{figure}[htpb]
  \centering
  \input{fig/graphs/length_suboptimality.tex}
  \caption{Quantitative analysis of the suboptimality of the \ac{catora} to \ac{tofrp} omitting assumption (A3). The presented results are generated using $10^5$ instances with a density of the environment $d_r=0.1$ for every number of robots in $[5, 175]$. $Q_x$ stands for corresponding quantiles and $\mu$ stands for the mean value. The curve of $Q_{1.0}$ is associated with the values on the right axis.}
  \label{fig:lower_bound_comparison}
\end{figure}
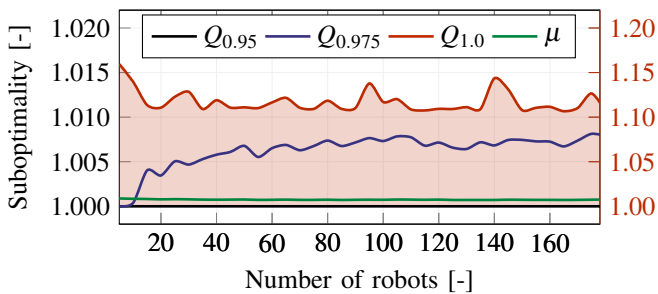

%%}

%%{ SECTION: Results

% ==============================================================================
\section{Numerical and experimental results}\label{sec:experimental_results}
% ==============================================================================
In this section, numerical and experimental results are presented to demonstrate the performance indicators of the proposed approach.
All results were evaluated in scenarios with varying numbers and densities of the robots randomly generated in a 3D environment.
The density of robots in the environment of volume $V_e$ is defined as 
\begin{equation} 
 d_r = \frac{\sum_{r=1}^N V_r}{V_e},
\end{equation}
where $V_r$ stands for the volume occupied by particular robots. 
All evaluations were performed on a computer with a 4-core Intel$^{(\text{R})}$ Core(TM) i7-10510U CPU with base frequency \SI{1.80}{GHz}.

\subsection{The effect on length of the path}

Although the presented approach is focused on minimizing the makespan of the formation reshaping process, the comparison based on the duration of the trajectories would depend on the choice of kinematic constraints. Thus, it would not yield fair results. 
Therefore, we compare the solutions provided by our algorithm for robot-to-goal assignment in terms of maximum length of the path with the solutions of \ac{lsap} used by several state-of-the-art works~\cite{turpin2014CaptConcurrentAssignment,lusk2020DistributedPipelineScalable, alonsoMora2019DistributedMultirobotFormation, quan2022FormationFlightDense}.
The results show that, on average, the \ac{catora} produces a set of paths with a maximum length 11\% shorter than the \ac{lsap} approach.
This highlights the significant benefit of using \ac{catora} instead of \ac{lsap}-based approaches, especially for battery-constrained robots or time-constrained applications. 
The detailed results for various numbers of robots and densities of the environment are presented in~\autoref{fig:lsap_vs_lbap}.
In compliance with~\autoref{th:bound_max_len}, a more significant effect is observed for instances with more robots.

\begin{figure}[htpb]
  \centering
  \input{fig/graphs/lsap_vs_lbap_length.tex}
  \vspace*{-0.5cm}
  \caption{The ratio between the maximum length of the set of paths produced by \ac{catora} $D_{m,CAT-ORA}$ and by the \ac{lsap} approach $D_{m,LSAP}$ for a varying number of robots and density of robots in the environment $d_r$. The results are generated using $10^5$ instances for each presented number of robots.}
  \label{fig:lsap_vs_lbap}
\end{figure}
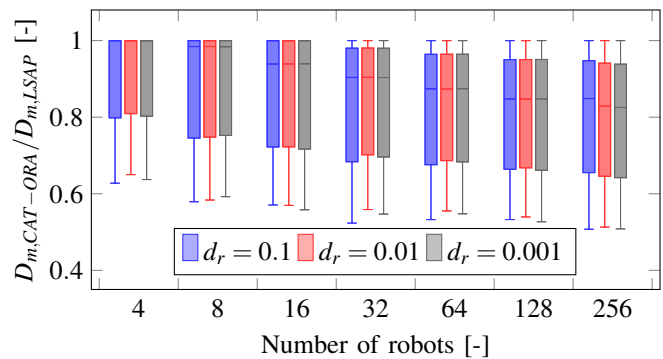

\subsection{Computational time}

The introduced procedures and checks guaranteeing the optimality of the \ac{catora} come at the cost of higher computational times in comparison to the original Hungarian algorithm.
A major increase in the computational burden may potentially come from the \textit{branchSolution()} method. 
However, reaching its theoretical asymptotic complexity would mean that all LBAP solution edges mutually collide. 
The probability of this situation is limited by the assumption (A5) and by solving the \ac{lbap} as an \ac{lsap} with restrictions on certain edges.
This brings the advantage that any pair of edges $e_{i,j}, e_{k,l}$ is guaranteed to be collision-free if $e_{i,l} \leq t_c$ and $e_{k,j} \leq t_c$.
Thus, in practice, the \textit{branchSolution()} method is responsible for $4.7\%$ of the total computational time on average among $10^5$ randomly generated instances with high density.

A detailed analysis has shown that the main part of the additional time required by \ac{catora} is not consumed by the collision resolution part, but by the search for a correct threshold for the feasible solution.
Since some algorithms for the solution of \ac{lbap} have lower theoretical complexity than those for the solution of \ac{lsap}, they can be used to increase the efficiency of a search for the threshold $t_c$.
However, their application in \autoref{alg:overall_lbap} is limited by the crucial role of dual variables that would require running the algorithm from its initial phase after the threshold is found.
Therefore, such an approach is efficient only for instances with a high number of robots and an inaccurate initial estimate of threshold $t_{lb}$.

The detailed comparison of computational times of the algorithm is shown in~\autoref{fig:res_computational_times}. 
Although the ratio between the maximum computational times of the Hungarian algorithm, applied for the solution of \ac{lsap} and \ac{catora} is significant, the absolute maximum difference in times does not exceed a few milliseconds for the instances with up to 32 robots. 
This keeps the computational demands sufficiently low for using \ac{catora} in applications that require real-time computations. 
The ratio between computational times decreases with an increasing number of robots since it mitigates the effect of the more demanding initialization phase. 

\begin{figure}[htpb]
  \centering
  \input{fig/graphs/lsap_vs_lbap_time.tex}
  \vspace*{-0.5cm}
  \caption{Comparison of computational demands of the \ac{lsap} approach and \ac{catora} approach for varying numbers of robots in an environment. The presented results are generated using $10^5$ instances with varying densities of the environment.}
  \label{fig:res_computational_times}
\end{figure}
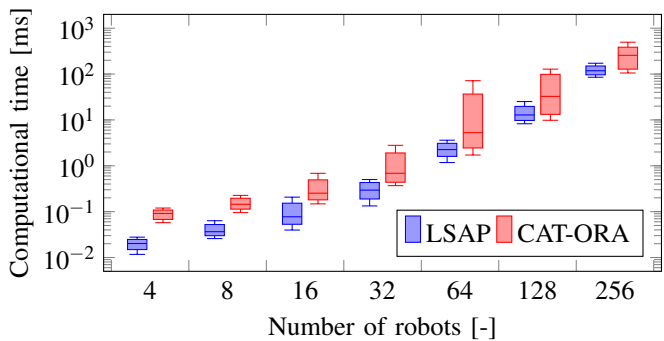

\subsection{Formation reshaping}

\begin{table*}[t]
\renewcommand{\arraystretch}{1.0}
  \begin{center}
    \caption{Comparison of different approaches for the solution of formation reshaping task. The values were computed from results obtained for $10^5$ instances of varying density $d_r \in [0.0001, 0.1]$, varying number of robots $N\in[10, 210]$, $v_{max} = \SI{4}{\meter\per\second}$, and $a_{max} = \SI{2}{\meter\per\second\squared}$. The $PDB$ value for statistics $x$ and method $m$ is computed as $PDB(x,m) = (x_m - x_{best})/x_{best}$, where $x_{best}$ is the best value of $x$ among all methods for a particular instance.      
    }\label{tab:solution_comparison}
    \newcommand{\rotateAngle}{00}
    \newcommand*{\OK}{\checkmark}
    \newcommand*{\hshift}{\hspace*{0.0cm}}
    % \rowcolors{1}{white}{gray!10!white}
    \begin{tabular}{l c c c c c c c c c}
      \toprule 
      \multirow{2.4}{*}{Approach} & \multirow{2.4}{*}{Success rate [\%]} & \multicolumn{2}{c}{~Makespan PDB [\%]} &  \multicolumn{2}{c}{~Max. length PDB [\%]} & \multicolumn{2}{c}{~Total length PDB [\%]} & \multicolumn{2}{c}{~Comp. time PDB [\%]} \\
      \cmidrule(l){3-4} \cmidrule(l){5-6} \cmidrule(l){7-8} \cmidrule(l){9-10}
      & & ~~mean & ~std. dev. & ~~~mean & std.dev & ~~~~mean & std.dev & ~~~mean & std.dev\\ 
      \midrule
      \ac{lsap}, min. time & \textbf{100.0} & ~~12.17 & ~8.54 & ~~~15.82 & 11.18 & ~~~~\textbf{0.00} & 0.00 & ~~~~~~\textbf{0.06} & ~~~1.69\\
      \ac{lbap}, min. time & 93.19 & ~~~\textbf{0.00} & ~0.00 & ~~~~\textbf{0.00} & ~0.00 & ~~~~3.11 & 2.82 & ~~~244.52 & 346.24\\
      \ac{catora} & \textbf{100.0} & ~~~0.06 & ~0.48 & ~~~~0.07 & ~0.63 & ~~~~3.11 & 2.83 & ~~~266.23 & 372.82\\
      \bottomrule
    \end{tabular}
  \end{center}
  \vspace*{-0.3cm}
\end{table*}

\begin{figure*}[htpb]
  \centering
  \input{fig/graphs/formation_reshaping.tex}
  \vspace*{-0.1cm}
  \caption{A qualitative comparison of the formation reshaping process applying \ac{catora} and approach applying \ac{lsap} solution coupled with minimum-time trajectory generation. The formation consists of 200 robots that are initially organized in a rectangular formation and are consequently required to adapt the shape of the formation to represent letters C, T, and U. The applied kinematic constraints are $v_{max} = \SI{4}{\meter\per\second}$, and $a_{max} = \SI{2}{\meter\per\second\squared}$. The height of each letter is \SI{100}{\metre} and the scale of the axes is equivalent. The gray lines represent the reshaping paths, and the colored points represent positions of robots at corresponding times. The color encodes the velocity of particular robots, with red being equal to zero velocity and yellow to $v_{max}$.}
  \label{fig:res_formation_reshaping}
  \vspace*{-0.5cm}
\end{figure*}
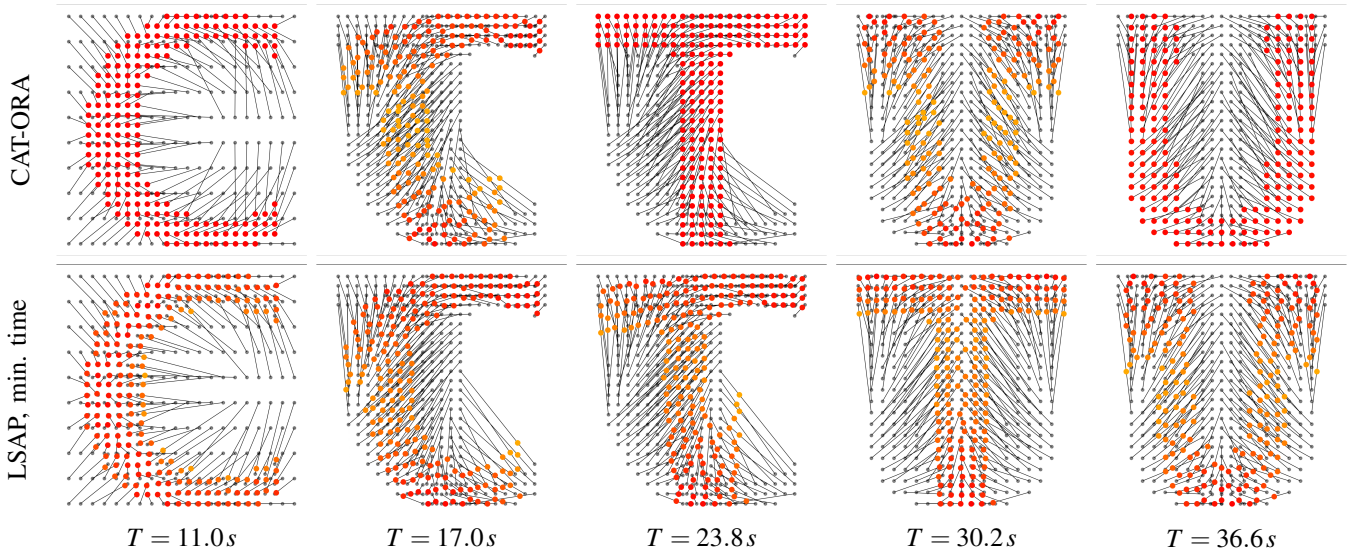

We benchmark the \ac{catora} by comparing the achieved results with the \ac{lsap} and \ac{lbap} algorithms coupled with minimum-time trajectories. 
Similarly to the detailed results in the previous section, the algorithms were evaluated on a set of $10^5$ instances representing formation reshaping tasks with various numbers and densities of robots in an environment. 
While the comparison results can be easily inferred from the characteristics of the individual algorithms, the presented results, as detailed in~\autoref{tab:solution_comparison}, quantitatively demonstrate the expected outcomes.
The \ac{lbap}-based consistently yields solutions with shorter maximum path lengths compared to other methods, resulting in a reduced makespan.
However, the generated trajectories lead to collisions in more than $6\%$ of instances. 
The \ac{catora} and \ac{lsap}-based approach provide collision-free trajectories for all instances. 
Yet, while the \ac{lsap}-based solution leads to an average increase of $12\%$ in makespan and $15\%$ in maximum path length compared to \ac{lbap}-based approach, the \ac{catora} only marginally extends the duration of reshaping process by an average of $0.06\%$ compared to \ac{lbap}-based approach.

The advantage of the \ac{catora} over the \ac{lsap}-based algorithm is showcased in a scenario requiring 200 robots initially arranged in a rectangular formation to sequentially adapt the formation shape to represent the letters C, T, and U.
Guided by the \ac{catora}, the entire formation reshaping task is completed in \SI{36.6}{\second}, which is \SI{5.4}{\second} faster than the solution provided by the \ac{lsap}-based approach, while the increase in computation time is $0.9\,\mathrm{s}$.
A detailed presentation of a specific formation reshaping instance is provided in~\autoref{fig:res_formation_reshaping}.

\subsection{Integration with cooperative motion planning algorithm}
The applicability of the proposed algorithm in the distributed scenarios is validated by integrating the robot-to-goal assignment algorithm of \ac{catora} with a distributed cooperative motion planning algorithm MADER\cite{tordesillas2022mader}.
The results, presented in~\autoref{tab:mader_comparison}, show that complete \ac{catora} reduces the makespan by $45\%$ on average when compared to MADER combined with the robot-to-goal assignment part of \ac{catora} only.
This is caused mainly by the MADER's approach to collision resolution, which reacts to the identified risk of collision in advance, including situations where the collision would not happen. Thus, hindering full exploitation of kinematic constraints to minimize the makespan. 
Such an approach is very reasonable and practical in many scenarios, but at the same time, it leads to unnecessary extension of the reshaping process, which underlines the advantage of centralized algorithms in scenarios focused on minimizing the makespan.

Further, we evaluate the performance of the MADER in combination with three different robot-to-goal assignment methods --- \ac{lsap}-based, \ac{lbap}-based and \ac{catora}-based.    
The MADER algorithm provides collision-free solutions in combination with all examined assignments. 
However, although the \ac{lbap}-based solution provides the shortest paths among all assignments, the results show that the needed collision avoidance maneuvers prolong the makespan of the formation reshaping, and on average \ac{lbap}-based solution provides worse performance than both collision-free assignments.
The best average performance over all instances was achieved with the assignment provided by \ac{catora}, which reduces the number of avoidance maneuvers by considering the collisions during the robot-to-goal assignment phase while minimizing the length of the longest path.
The results highlight the importance of the robot-to-goal assignment for efficient formation reshaping and consideration of the collisions during the robot-to-goal assignment, even in the case of the application of advanced collision-resolution techniques and algorithms.
The detailed results are provided in \autoref{tab:mader_comparison}.
\begin{remark}
  For the methods applying MADER, the individual instances were rotated according to the result of the assignment such that the longest assigned path points in the diagonal direction of $xy$ coordinate frame and thus the speed in this direction is not limited by the per-axis velocity and acceleration constraints applied by MADER.
\end{remark}

\begin{table}[t]
\renewcommand{\arraystretch}{1.0}
\setlength\tabcolsep{4.6pt}
  \begin{center}
    \caption{Comparison of the proposed approach with the combination of different approaches to robot-to-goal assignment and distributed cooperative motion planning algorithm MADER\cite{tordesillas2022mader}. The comparison was done on 100 instances with randomly generated sets of start and goal configurations with distinct robot-to-goal assignments for individual approaches, $v_{max} = \SI{4}{\meter\per\second}$, and $a_{max} = \SI{2}{\meter\per\second\squared}$, and number of robots limited to $N \in [3, 10]$ due to the computation limitations of the MADER when running on a single computer. 
    The $PDB$ values are computed separately for the minimum time and the MADER part of the table.
    }\label{tab:mader_comparison}
    \newcommand{\rotateAngle}{00}
    \newcommand*{\OK}{\checkmark}
    \newcommand*{\hshift}{\hspace*{0.0cm}}
    \addtolength{\tabcolsep}{-0.1em} 
    \begin{tabular}{l c c c c}
      \toprule 
      \multirow{2.4}{*}{Approach} & \multirow{2.4}{*}{\begin{tabular}{c}Success\\ rate [\%]\end{tabular}} & \multirow{2.4}{*}{\begin{tabular}{c}Makespan\\ mean [s]\end{tabular}} & \multicolumn{2}{c}{~Makespan PDB [\%]}\\
      \cmidrule(l){4-5}
      & & & ~~~mean & ~std. dev.\\ 
      \midrule
      \ac{lsap} + min. time & \textbf{100.0} & 4.89 & ~~~7.34 & ~4.82\\
      \ac{lbap} + min. time & 0.0 & \textbf{4.56} & ~~~\textbf{0.00} & ~\textbf{0.00}\\
      \ac{catora} & \textbf{100.0} & 4.65 & ~~~1.91 & ~1.37\\
      \midrule
      \ac{lsap} + MADER & \textbf{100.0} & 8.51 & ~~~5.76 & ~\textbf{7.45}\\
      \ac{lbap} + MADER & \textbf{100.0} & 8.74 & ~~~8.56 & ~9.88\\
      \ac{catora} + MADER & \textbf{100.0} & \textbf{8.44} & ~~~\textbf{4.89} & ~7.49\\
      \bottomrule
    \end{tabular}
    \addtolength{\tabcolsep}{+0.1em} 
    % \end{minipage}
    % }
  \end{center}
  \vspace*{-0.3cm}
\end{table}

\subsection{Real-world experiment}

In the real-world experiment, the \ac{catora} was applied in a scenario simulating a small-scale drone performance.
The scenario requires a set of robots to perform 19 transitions between formations of diverse shapes (both 2D and 3D) and sizes, while the center of the formation continuously moves through the environment.
Each formation $\mathbb{F}_i = \{\mathbf{r}_1,\dots,\mathbf{r}_N\}$ is defined by a set of desired relative positions to the center of the formation $\mathbf{r}_j \in \mathbb{R}^3$ defined in the orthogonal coordinate system $H$ that coincides with the position and orientation of the center of the formation.
Since the requirement on continuous movement contradicts the assumption (A2) on robots being stationary in the initial and goal configurations, \ac{catora} cannot be directly applied to compute trajectories between the robots' configurations defined in the world coordinate frame $W$.

However, the definition of relative positions in an orthogonal coordinate system ensures independence of mutual distances between desired relative positions on the motion of the formation.
Therefore, we apply \ac{catora} to compute the trajectories in the space of relative positions considering consecutive formations $\mathbb{F}_i, \mathbb{F}_{i+1}$ as initial and goal configurations, respectively. 
The generated trajectories then define the time evolution of $\mathbf{r}_j$, leading to continuous adaptation of the formation shape.
The trajectories in the world coordinate frame are then defined by $\mathbf{p}_j(t) =  \mathbf{T}_{H,W} \mathbf{r}_j(t), t \in (t_0, t_f)$, where $\mathbf{T}_{H,W}$ is a transformation matrix from formation frame $H$ to world coordinate frame $W$.
This approach shows that the assumption (A2) is not a strict requirement for the applicability of \ac{catora} to solve \ac{tofrp}.
However, the superposition of the generated trajectories to the trajectory of the center of the formation requires adapting the kinematic constraints for the generation of formation reshaping trajectories, such that the resulting trajectories $\mathbf{p}_j(t), j \in \{1, 2, \dots, N\}, t \in (t_0, t_f)$ do not violate the kinematic constraints.
Thus, (A3) is a necessary assumption for guaranteeing the optimality of the solution.

The real-world experiment was performed with 19 multi-rotor helicopters~\cite{HertJINTHW_paper, hert2022MRSModularUAV} using the MRS UAV system for low-level control and trajectory tracking~\cite{baca2021mrs}.
The time required for the whole performance was \SI{294.80}{\second} with the total time of reshaping \SI{214.8}{\second} and an overall computational time of \SI{60}{\milli\second}.    
Snapshots from the experiment are shown in~\autoref{fig:real_world_experiment} and~\autoref{fig:intro}.  

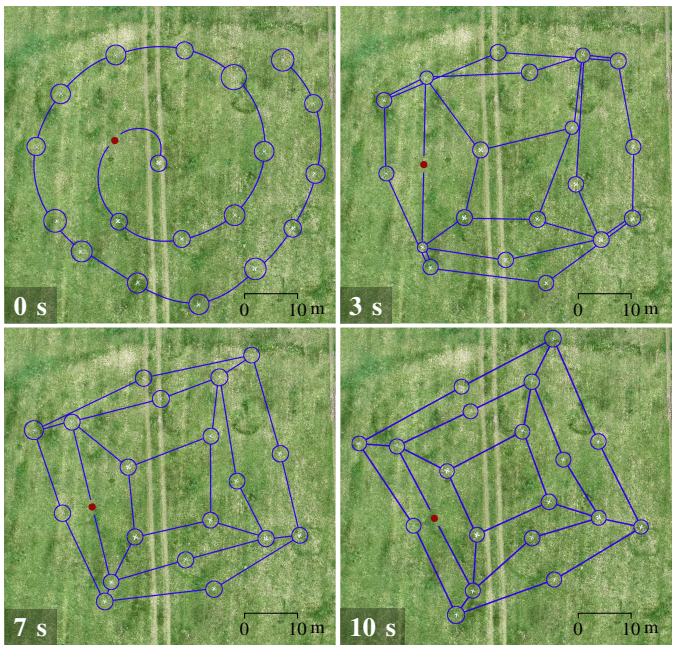
\begin{figure}[!t]
  \newcommand{\xcap}{0.99em}
  \newcommand{\ycap}{0.695em}
  \newcommand{\yshift}{12.16em}
  \newcommand{\xshift}{12.95em}
  \newcommand{\fillopa}{0.3}
  \centering
  \begin{tikzpicture}
    \node[anchor=south west,inner sep=0] (b) at (0,0) {\includegraphics[width=1.0\columnwidth]{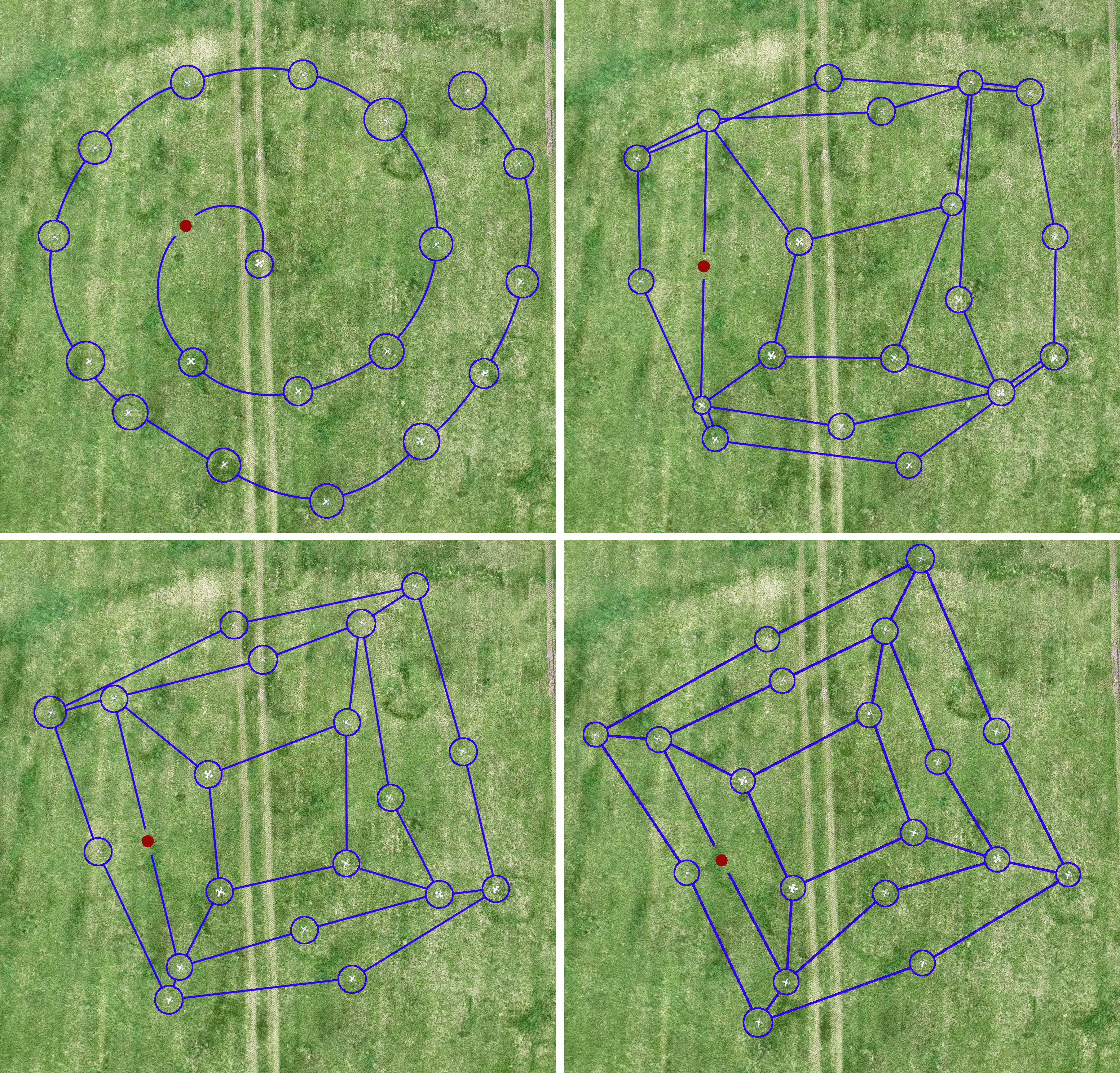}};%
  \begin{scope}[x={(b.south east)},y={(b.north west)}]
\node[fill=black, fill opacity=\fillopa, text=white, text opacity=1.0] at (\xcap, \ycap+\yshift) {\textbf{0 s}};
\node[fill=black, fill opacity=\fillopa, text=white, text opacity=1.0] at (\xcap+\xshift-0.27em, \ycap+\yshift) {\textbf{3 s}};
\node[fill=black, fill opacity=\fillopa, text=white, text opacity=1.0] at (\xcap, \ycap) {\textbf{7 s}};
\node[fill=black, fill opacity=\fillopa, text=white, text opacity=1.0] at (\xcap+\xshift, \ycap) {\textbf{10 s}};

    \draw (0.86,0.05) [scale bar=10];
    \draw (0.36,0.05) [scale bar=10];
    \draw (0.86,0.55) [scale bar=10];
    \draw (0.36,0.55) [scale bar=10];
    \end{scope}
  \end{tikzpicture}
  \caption{Snapshots from a real-world experiment showing the transition between a 3D spiral and a pyramid shape. The transition was completed within 10 seconds during continual rotation of the formation. The red point represents a missing UAV that failed to start due to a HW failure. Blue lines highlight the shape of the formation. Since the images show a 3D formation, the measuring scale is approximate.}
  \label{fig:real_world_experiment}
\end{figure}

%%}

%%{ SECTION: Discussion

% ==============================================================================
\section{Discussion}\label{sec:discussion}
% ==============================================================================

The proposed \ac{catora} provides a significant reduction (up to $49\%$) of the reshaping time at the cost of an increase in computational time.
Even though the computational time is, on average, approximately 3 times longer than for LSAP-based solutions, the low absolute computation times maintain the practicality of the proposed approach for real-time applications involving formations of tens of robots.

The primary drawback of the proposed method is its high theoretical asymptotic complexity which lies especially in the search for all potentially valid solutions during solution branching.
However, while this complexity does not allow us to provide guaranteed bounds on computation time, it rarely causes problems in practical use (as supported by data presented in~\autoref{fig:res_computational_times}). 
The only notable deviations in computation time outside the distribution presented in~\autoref{fig:res_computational_times} were observed for a few instances with more than 200 equally spaced robots, resulting in a large number of equivalent values in a distance matrix representing distances between individual start and goal configurations.

Although the algorithm is centralized, its expected applications are not limited to scenarios requiring offline computation of trajectories for a large number of robots or other elements (e.g., droplets on lab-on-chip devices), where the minimization of the makespan of the reshaping process is of interest. 
Thanks to low computational demands, the algorithm can be used also as a part of distributed systems where it can provide efficient initial robot-to-goal assignment (similarly as \ac{lsap}-based solution is used\cite{quan2022FormationFlightDense}) for on-demand or emergent formation reshaping tasks. 
These scenarios assume a relatively low number of robots (usually less than one hundred) and often show high interest in minimizing the makespan.
This combination of requirements makes the proposed algorithm well-suited for these real-time scenarios.

%%}

%%}

%%{ SECTION: Conclusion

% ==============================================================================
\section{Conclusion}\label{sec:conclusion}
% ==============================================================================

This paper introduces an algorithm named \acs{catora} (\acl{catora}) to address the time-optimal formation reshaping problem while considering mutual collision avoidance among robots.
It showcases superior performance in terms of the makespan of the formation reshaping process, while maintaining computational demands at a level suitable for real-time deployment, even in formations comprising up to one hundred robots.
The properties of the proposed algorithm have been evaluated by thorough numerical and theoretical analysis, including the proof of optimality, and the applicability of the algorithm in practical scenarios was demonstrated through simulations and real-world experiments.

Notably, the results highlight a significant advantage of the robot-to-goal assignment aspect within \ac{catora}.
It reduces the maximum length of the assigned path by up to 49\% compared to the \ac{lsap}-based methods utilized by state-of-the-art approaches in cooperative motion planning and formation control.
This finding holds particular significance for aerial vehicles with constrained operational time, as it enhances their performance during a real-world deployment. Moreover, this outcome has potential implications for future research on formation reshaping focused on deploying autonomous robots in general environments as it forms the lower bound on the optimal solution of the introduced problem in environments with obstacles.

%%}

% %%{ SECTION: Bibliography
\bibliographystyle{IEEEtran}
\bibliography{main}
% %%}

%%{ APPENDICES

\appendices

\section{Minimum distance on diagonals of an isosceles trapezoid}\label{ap:min_dist_trapezoid}

\begin{proof}[\unskip\nopunct]
For finding the minimum mutual distance of robots following constant-velocity trajectories on the diagonals of an isosceles trapezoid, the relation for a minimum value of $\alpha_{ij}^*$ derived for a general case in~\eqref{eq:alpha_min} can be used.
Given that $||\mathbf{s}_{ij}|| = ||\mathbf{g}_{ij}|| \Leftrightarrow a = c $,  \eqref{eq:alpha_min} can be simplified to
\begin{equation}
  \alpha_{ij}^* = \frac{a-b}{2a-2b} = \frac{1}{2}.
\end{equation}
As a consequence, the minimum distance between two trajectories following the diagonals of an isosceles trapezoid equals to distance between points cutting the diagonals in two line segments of equal length (see~\autoref{fig:isosceles_trapezoid}).

\begin{figure}[thpb]
  \centering
  \input{fig/trapezoids/isosceles_trapezoid.tex}
  \caption{Graphical illustration of the minimum distance $d_{ij,min}$ of constant-velocity trajectories lying on the diagonals of an isosceles trapezoid.}
  \label{fig:isosceles_trapezoid}
\end{figure}
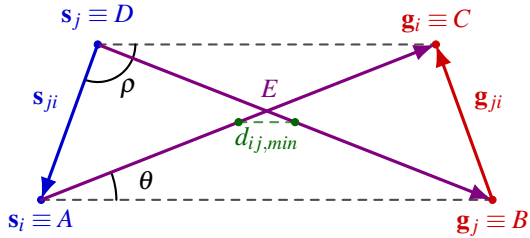

The theoretical guarantees on the minimum distance can be determined by the relation between minimum distance $d_{ij,min}$ and the value of $\delta_{ij}$~\eqref{eq:min_dist_starts_and_goals}.
Using substitutions
\begin{equation}
  \mathbf{s}_i \equiv A, \mathbf{g}_j \equiv B, \mathbf{g}_i \equiv C, \mathbf{s}_j \equiv D,
\end{equation}
the following relations hold for an isosceles trapezoid
\begin{equation}\label{eq:k_coeff}
  K = \frac{|CD|}{|AB|} = \frac{|CE|}{|AE|} = \frac{|DE|}{|BE|}.
\end{equation}

Applying the concept of triangle similarity, the distance $d_{ij,min}$ in an isosceles trapezoid can be expressed as
\begin{equation}\label{eq:d_min_isosceles}
  d_{ij,min} = \frac{|AE||AB|-\frac{1}{2}|AB||AC|}{|AE|}.
\end{equation}
Based on the assumption from~\autoref{th:minimum_dist_lbap} and~\eqref{eq:k_coeff}, we can determine relations
\begin{equation}\label{eq:isosceles_substitutions}
  \begin{split}
    |AB| =& \frac{1}{M} |AC|,\\
    |AE| =& \frac{1}{1+K} |AC|,
  \end{split}
\end{equation}
that can be applied to simplify~\eqref{eq:d_min_isosceles} to
\begin{equation}\label{eq:d_min_isosceles_simplified}
  d_{ij,min} = \frac{(1-K)|AC|}{2M}.
\end{equation}
Furthermore, we can utilize the properties of the right triangle and leverage the Law of Cosines to find a system of equations
\begin{equation}\label{eq:equations_for_k}
  \begin{split}
    \cos \theta =& \frac{|AB| - \frac{|AB|-K|AB|}{2}}{M|AB|} = \frac{1+K}{2M},\\
    \cos \theta =& \frac{|AB|^2 + M^2|AB|^2 - \delta_{ij}^2}{2M|AB|^2}.
  \end{split}
\end{equation}
From the system of equations~\eqref{eq:equations_for_k}, the relation for $K$ can be derived as
\begin{equation}
  K = \frac{M^2(|AC|^2-\delta_{ij}^2)}{|AC|^2}.
\end{equation}
By substituting this result into~\eqref{eq:d_min_isosceles_simplified}, we get the required relation for a minimum distance of two trajectories depending only on the minimum mutual distance of starts and goals $\delta_{ij}$, the ratio between the length of the trajectories $M$, and the length of the trajectory $|AC|$
\begin{equation}\label{eq:d_min_final}
  d_{ij,min} = \frac{|AC|^2(1-M^2) + M^2\delta_{ij}^2}{2M|AC|}.
\end{equation}

To find a minimum distance of trajectories independently on the length of $|AC|$, we find the derivative of~\eqref{eq:d_min_final} with respect to $|AC|$ and set it equal to zero
\begin{equation}\label{eq:partial}
  \frac{\partial d_{ij,min}}{\partial |AC|} = \frac{(1-M^2)|AC|^2-M^2\delta_{ij}^2}{2M|AC|^2} = 0.
\end{equation}
As a solution of~\eqref{eq:partial} with $\frac{\partial^2 d_{ij,min}}{\partial^2 |AC|} > 0$, we get the value of $|AC|$ minimizing the distance $d_{ij,min}$
\begin{equation}
  |AC|_{min}^* = \frac{\sqrt{1-M^2}M\delta_{ij}}{1-M^2}.
\end{equation}
By substituting $|AC|_{min}^*$ into~\eqref{eq:d_min_final} and simplifying this equation, we get a final relation for $d_{ij,min}$:
\begin{equation}\label{eq:d_min_final_simplified}
  d_{ij,min} = \sqrt{1-M^2}\delta_{ij},
\end{equation}
that proves~\autoref{th:minimum_dist_lbap}.
\end{proof}
%%}

%%{ Biographies
\newpage
\begin{IEEEbiography}[{\includegraphics
[width=1in,height=1.25in,clip,
keepaspectratio]{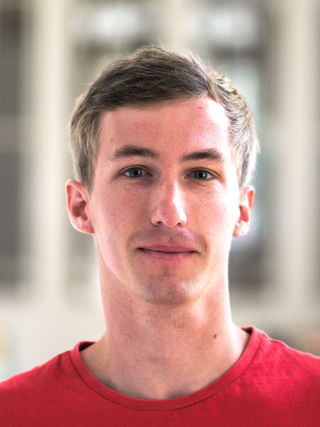}}]
{Vit Kratky} 
  received his Ph.D. degree in  Informatics, from the Czech Technical University in Prague, Czech Republic. He is a part of the \href{https://mrs.fel.cvut.cz}{Multi-robot Systems group} in CTU Prague, where he focuses on cooperative motion planning and autonomous navigation of unmanned aerial vehicles. Vit has co-authored 20 publications in conference and impacted journals, with over 700 citations indexed by Google Scholar and an h-index of 14. He has been involved in projects focused on inspection of man-made infrastructures, including \href{https://mrs.fel.cvut.cz/projects/dronument}{Dronument} and \href{https://aerial-core.eu/}{Aerial-Core} projects. Vit was also part of the CTU-CRAS-NORLAB team in the DARPA SubT competition. 
\end{IEEEbiography}

\begin{IEEEbiography}[{\includegraphics
[width=1in,height=1.25in,clip,
keepaspectratio]{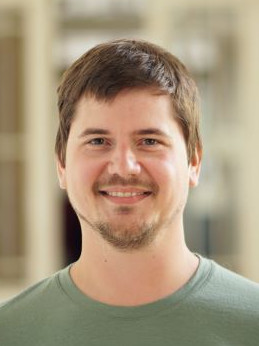}}]
{Robert Penicka} 
received his Ph.D. in Artificial Intelligence from the Czech Technical University (CTU) in Prague in 2020. He was a Postdoctoral Researcher at the University of Zürich (2020–2022) under Prof. D. Scaramuzza and is now a Research Fellow at CTU. His work focuses on UAV mission planning, trajectory optimization, and agile flight in cluttered environments. He was part of the CTU team that won the MBZIRC 2020 Grand Challenge. His doctoral thesis earned multiple awards, including the Dean’s Prize, the Werner von Siemens Award (2nd place), the Joseph Fourier Prize, and the Antonín Svoboda Award.
\end{IEEEbiography}

\begin{IEEEbiography}[{\includegraphics
[width=1in,height=1.25in,clip,
keepaspectratio]{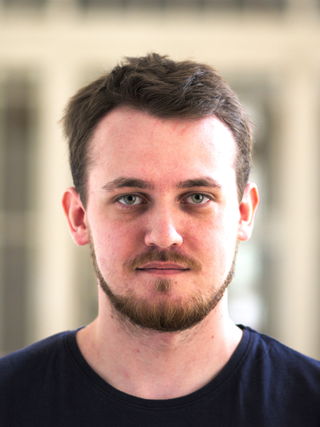}}]
{Jiri Horyna} 
received his M.Sc. degree from the Czech Technical
University in Prague, Czech Republic, where he is currently pursuing a
Ph.D. in the control and stabilization of swarms of unmanned aerial
vehicles. Since 2017, he has been a member of the \href{https://mrs.fel.cvut.cz}{Multi-robot Systems lab} at CTU Prague. He has co-authored 10 publications in
conferences and peer-reviewed journals, with over 150 citations
indexed by Scholar and an h-index of 6. His research primarily focuses
on multi-robot state estimation for groups of unmanned aerial vehicles.
\end{IEEEbiography}

\begin{IEEEbiography}[{\includegraphics
[width=1in,height=1.25in,clip,
keepaspectratio]{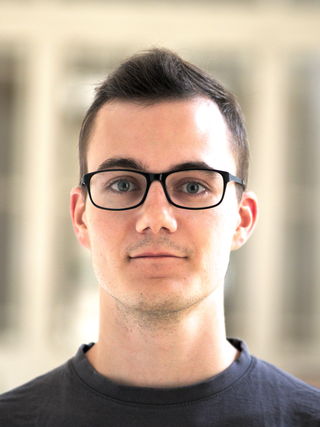}}]
{Petr Stibinger} 
received his M.Sc. degree in cybernetics and robotics
at the Czech Technical University in Prague, and is currently pursing
a Ph.D. with the \href{https://mrs.fel.cvut.cz}{Multi-robot Systems group} at CTU in Prague. 
His research focuses on data fusion and motion planning for robots
operating in hazardous environments. 
He is a co-author of 9 publications in conferences and impacted journals with $>200$ citations indexed by Scholar and h-index 6. He was a member of CTU-UPENN-NYU
team in the MBZIRC 2020.
\end{IEEEbiography}
\vfill 

\newpage

\begin{IEEEbiography}[{\includegraphics
[width=1in,height=1.25in,clip,
keepaspectratio]{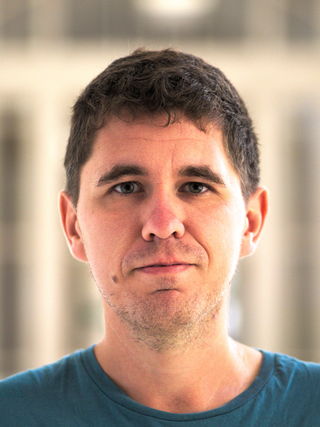}}]
{Tomas Baca} 
received his Ph.D. degree on distributed remote sensing at the Czech Technical University in Prague, Czech Republic. 
  He is a part of the \href{https://mrs.fel.cvut.cz}{Multi-robot Systems lab} in CTU Prague, focusing on distributed radiation sensing with Unmanned Aerial Vehicles and planning and control for Unmanned Aerial Vehicles. 
  Tomas is a co-author of $>50$ publications in conferences and impacted journals with $>3200$ citations indexed by Scholar and h-index 32. 
  He was a member of CTU-UPenn-UoL and CTU-UPENN-NYU teams in the MBZIRC 2017 and MBZIRC 2020 robotic competitions in Abu Dhabi, and of the CTU-CRAS-NORLAB team in the DARPA SubT competition.
\end{IEEEbiography}

\begin{IEEEbiography}[{\includegraphics
[width=1in,height=1.25in,clip,
keepaspectratio]{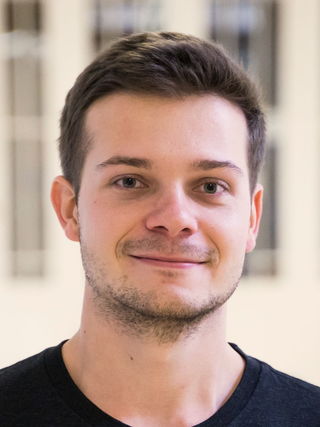}}]
{Matej Petrlik} 
received his Ph.D. in cybernetics and robotics from the Czech Technical University in Prague, Czech Republic. 
  His research is focused on state estimation for field deployment of autonomous UAVs in challenging environments. 
  He has been part of the \href{https://mrs.fel.cvut.cz}{Multi-robot Systems group} at CTU Prague since 2015. 
  He has co-authored 26 conference and journal publications, with over 1,300 citations indexed by Google Scholar and an h-index of 19. 
  He was also part of the CTU-UPenn-UoL team in the MBZIRC 2017, the CTU-UPENN-NYU team in the MBZIRC 2020, and the CTU-CRAS-NORLAB team in the DARPA SubT Challenge.
\end{IEEEbiography}

\vspace*{-0.32cm}
\begin{IEEEbiography}[{\includegraphics
[width=1in,height=1.25in,clip,
keepaspectratio]{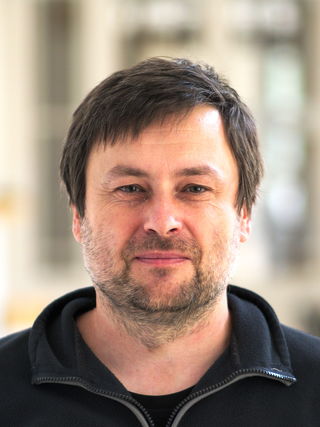}}]
{Petr Stepan} 
received his Ph.D. degree on sensor fusion for mapping at the Czech Technical University in Prague, Czech Republic. 
  He is a part of the \href{https://mrs.fel.cvut.cz}{Multi-robot Systems lab} in CTU Prague, where he focuses on sensor fusion, mapping, localization and planning for Unmanned Aerial Vehicles. 
  Petr has also been involved in industrial projects and the H2020 AerialCore project. 
  Petr is a co-author of $>40$ publications in conferences and impacted journals with $>700$ citations indexed by Scholar and h-index 11. He was a member of CTU-UPenn-UoL and CTU-UPENN-NYU teams in the MBZIRC 2017 and MBZIRC 2020 robotic competitions in Abu Dhabi.
\end{IEEEbiography}

\vspace{-1.5\baselineskip}
\begin{IEEEbiography}[{\includegraphics
[width=1in,height=1.25in,clip,
keepaspectratio]{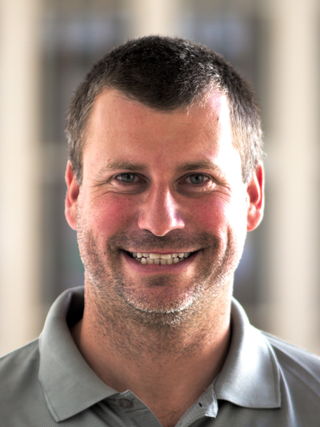}}]
{Martin Saska} 
received his Ph.D. degree at University of Wuerzburg, Germany, within the Ph.D. program of Elite Network of Bavaria. 
He founded and heads the \href{https://mrs.fel.cvut.cz}{Multi-robot Systems lab} at the Czech Technical University in Prague with more than 40 researchers. 
  He was a visiting scholar at University of Illinois at Urbana-Champaign and at University of Pennsylvania, USA. 
  He is a co-author of $>200$ publications in conferences and impacted journals, including IJRR, T-RO, AURO, JFR, ASC, EJC, with $>8100$ citations indexed by Scholar and h-index 50. 
  His team won multiple robotic challenges in MBZIRC 2017, MBZIRC 2020 and DARPA SubT competitions.
\end{IEEEbiography}
\vfill 
%%}

\end{document}

%% file: fig/catora/catora.tex
% \usetikzlibrary{shapes.geometric,backgrounds,calc,arrows}
\pgfdeclarelayer{backbackground}
\pgfdeclarelayer{background}
\pgfdeclarelayer{foreground}
\pgfsetlayers{backbackground,background,main,foreground}

\definecolor{arrow_red}{HTML}{A30D00}
% \definecolor{arrow_red}{HTML}{F37878}
\definecolor{arrow_green}{rgb}{0, .522, .243}
% \definecolor{arrow_blue}{rgb}{0, 0, .9}

% \definecolor{arrow_red}{HTML}{FF3333}
\definecolor{arrow_blue}{HTML}{22559C}

\definecolor{color_grey}{HTML}{606060}
\definecolor{color_blue}{HTML}{22559C}
\definecolor{color_red}{HTML}{FF3333}
\definecolor{color_green}{HTML}{D9F8C4}

\pgfmathsetmacro{\vshift}{0.12em} % shift in vertical direction
\pgfmathsetmacro{\hshift}{0.35em} % shift in vertical direction
\pgfmathsetmacro{\hshifta}{0.33em} % shift in vertical direction
\pgfmathsetmacro{\hshiftb}{0.43em} % shift in vertical direction
\pgfmathsetmacro{\voffset}{0.09em} % shift in vertical direction
\pgfmathsetmacro{\hoffset}{0.00em} % shift in vertical direction
\pgfmathsetmacro{\folloffset}{0.325em} % shift in vertical direction
\pgfmathsetmacro{\folloffsetn}{0.12em} % shift in vertical direction
\pgfmathsetmacro{\operatoroffset}{0.055em} % shift in vertical direction
\pgfmathsetmacro{\postdeployoffset}{0.124em} % shift in vertical direction
\pgfmathsetmacro{\postdeployvoffset}{0.118em} % shift in vertical direction
\pgfmathsetmacro{\nthuavvoffset}{-0.1em} % shift in vertical direction
\pgfmathsetmacro{\predeployoffset}{-0.27em} % shift in vertical direction
\pgfmathsetmacro{\predeployvoffset}{-0.066em} % shift in vertical direction
\pgfmathsetmacro{\safetyoperatoroffset}{-0.258em} % shift in vertical direction
\pgfmathsetmacro{\legendoffset}{-0.08em} % shift in vertical direction
\pgfmathsetmacro{\sensorsshift}{0.43} % shift in vertical direction
\pgfmathsetmacro{\trackingshift}{-0.06} % shift in vertical direction
\pgfmathsetmacro{\harrowshift}{0.1} % shift in vertical direction
\pgfmathsetmacro{\harrowoffset}{0.12} % shift in vertical direction
\pgfmathsetmacro{\varrowoffset}{0.18} % shift in vertical direction
\pgfmathsetmacro{\varrowshift}{0.18} % shift in vertical direction
\pgfmathsetmacro{\harrowshiftpredeploy}{1.69} % shift in vertical direction
\pgfmathsetmacro{\harrowshiftpostdeploy}{1.4} % shift in vertical direction
\pgfmathsetmacro{\harrowshiftoperator}{0.4} % shift in vertical direction
\pgfmathsetmacro{\operatorswitchshift}{0.5} % shift in vertical direction
\pgfmathsetmacro{\harrowbend}{0.1em} % horizontal shift of arrows
\pgfmathsetmacro{\safetyswitchshift}{0.092em} % horizontal shift of arrows
\pgfmathsetmacro{\harrowshiftscan}{0.14em} % horizontal shift of arrows
\pgfmathsetmacro{\predeployintershift}{0.52} % horizontal shift of arrows
\pgfmathsetmacro{\operatorintershift}{0.25} % horizontal shift of arrows
\pgfmathsetmacro{\postdeployintershift}{0.165} % horizontal shift of arrows
\pgfmathsetmacro{\representationshift}{0.04} % horizontal shift of arrows

% \tikzset{radiation/.style={{decorate,decoration={expanding waves,angle=90,segment length=4pt}}}}
\tikzstyle{block}=[draw, rounded corners, text centered, minimum height=2.0em, fill=white, fill opacity=1.0, text opacity=1.0]
\tikzstyle{arrow}=[draw, ->, thick]
\def\nodedst{2.0cm}

\begin{tikzpicture}[auto, node distance=1.0cm, >=latex, font=\scriptsize]

  \renewcommand{\arraystretch}{0.6}

  %%%{ nodes

  % deployment phase - operator
  \begin{pgfonlayer}{foreground}

  % \node[anchor=south west,inner sep=0] (a) at (0,0) {
  %     \def\arraystretch{0}%
  %     \includegraphics[width=0.2\columnwidth]{fig/catora/init_configuration.png}%
  %   }; 

  % deployment phase - uav
    \node [block] (initconfiguration) {
      \includegraphics[width=0.21\columnwidth]{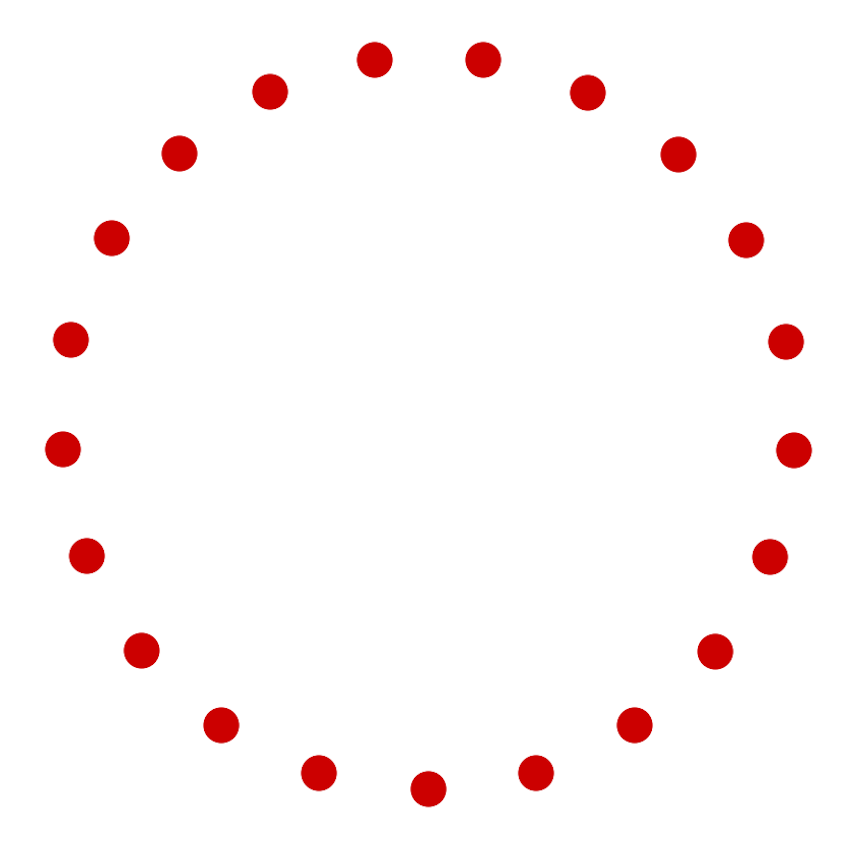}
    };
  
    \node [block, below of=initconfiguration, shift = {(-0.0, -\vshift)}] (goalconfiguration) {
        \includegraphics[width=0.21\columnwidth]{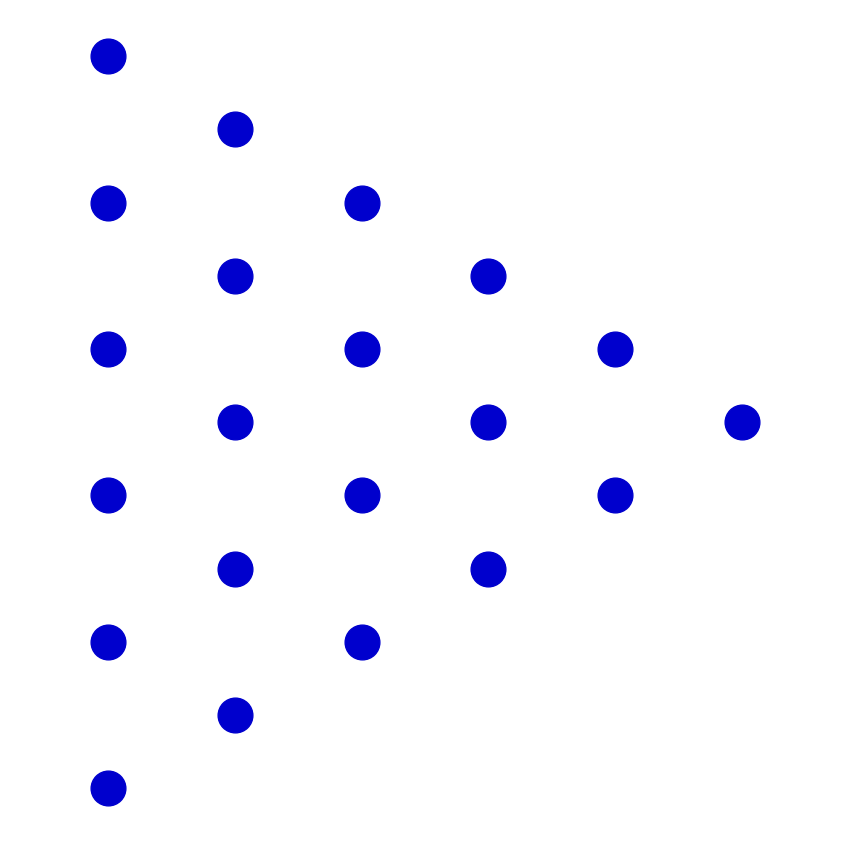}
    };

    \node [block, right of=initconfiguration, shift = {(\hshift, -\voffset)}] (assignment) {
        \includegraphics[width=0.4\columnwidth]{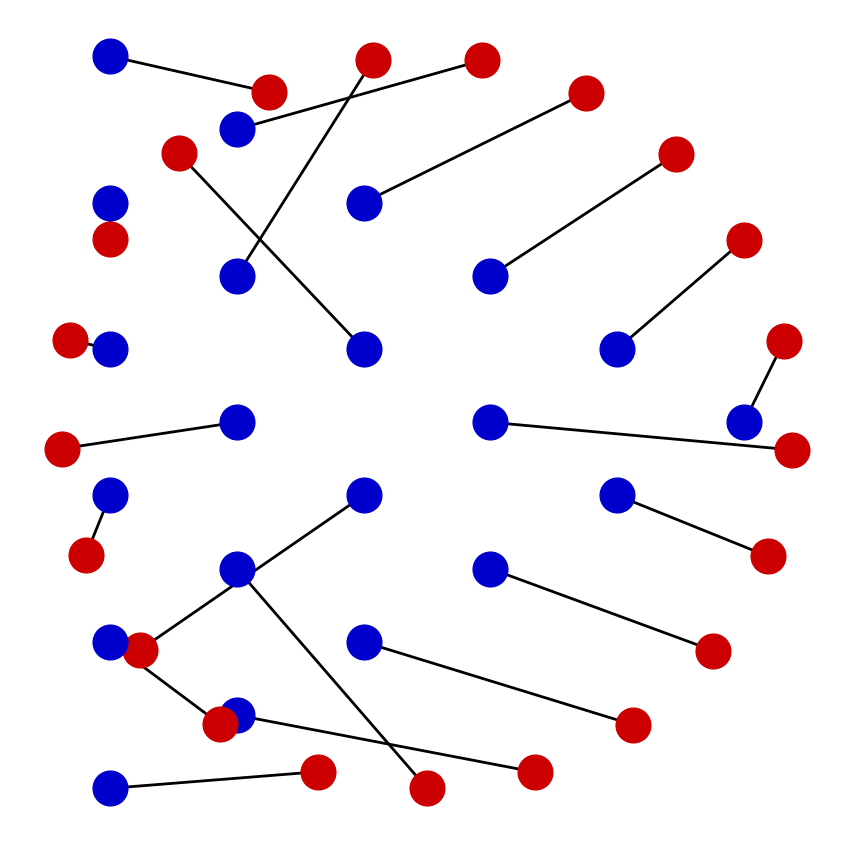}
    };

    \node [block, right of=assignment, shift = {(\hshifta, -0.0)}] (trajectorygeneration) {
        \includegraphics[width=0.4\columnwidth]{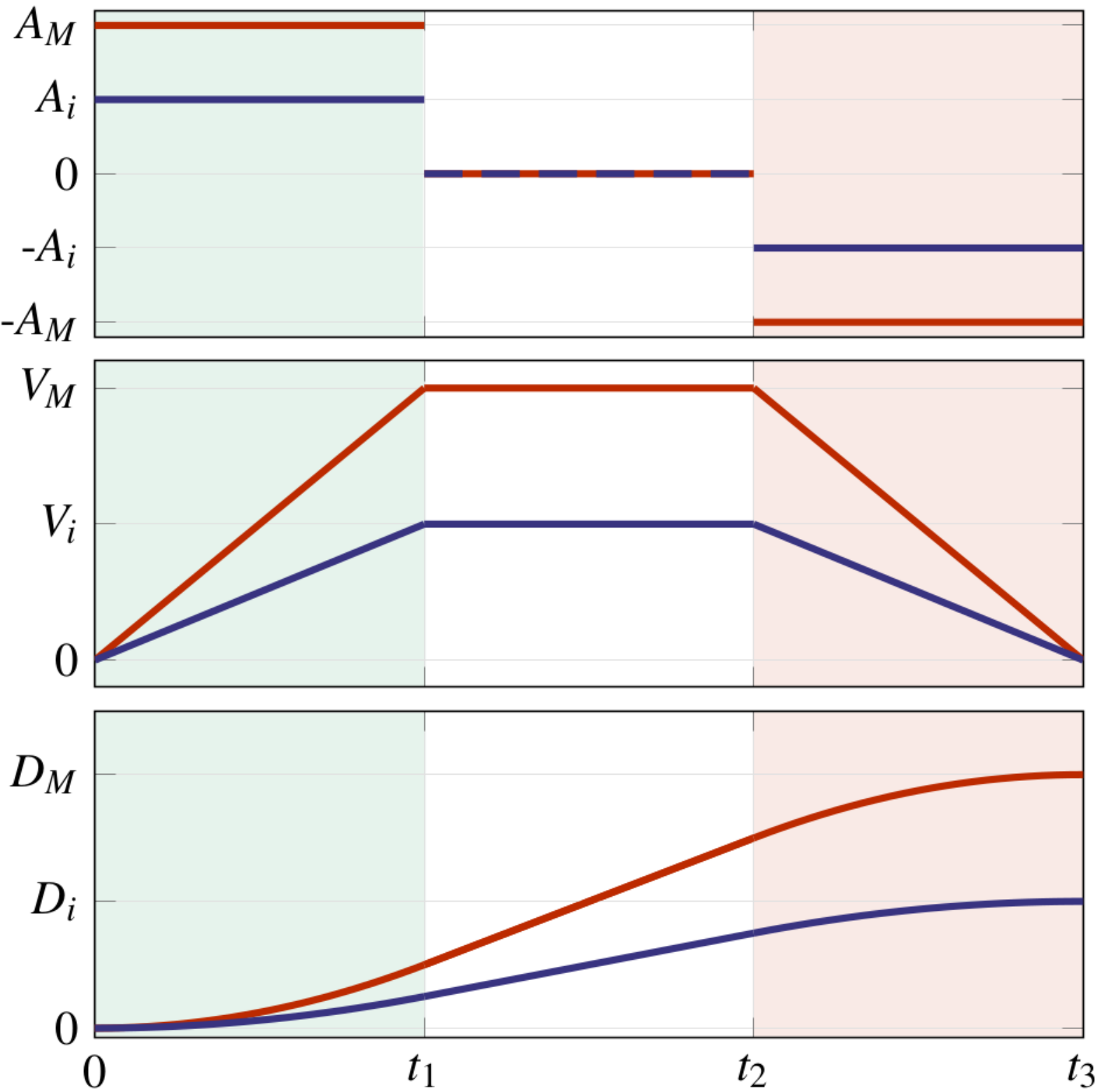}
    };

    \node [block, right of=trajectorygeneration, shift = {(\hshiftb, -0.0)}] (trajectories) {
        \includegraphics[width=0.4\columnwidth]{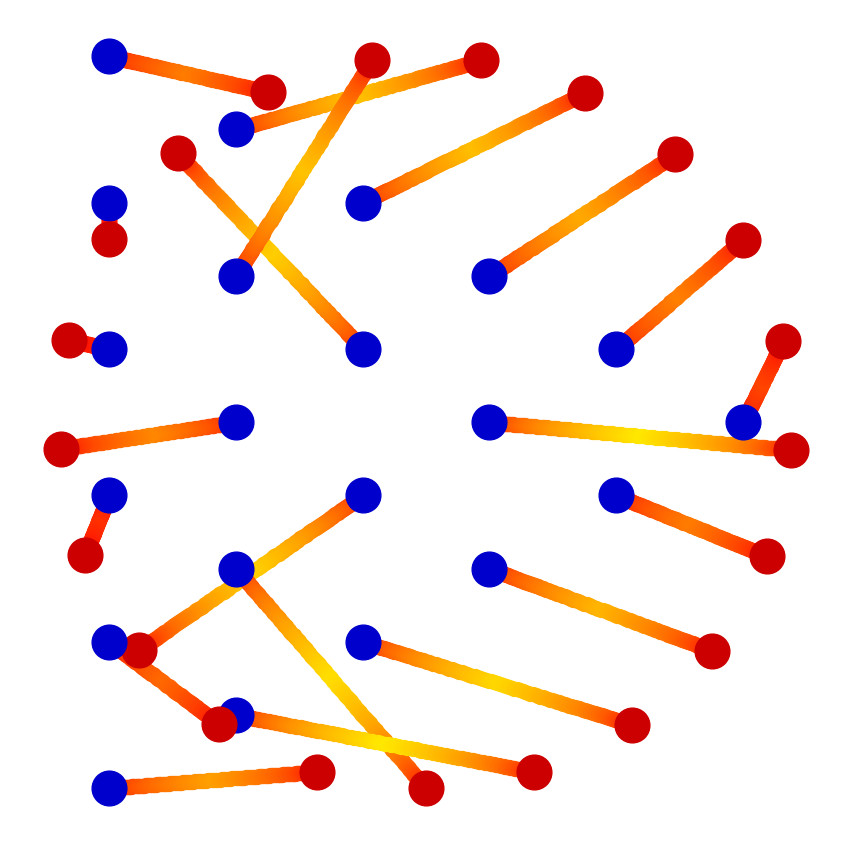}
    };

    % operators
    % \node [block, below of=experts,  shift = {(0.0, -\vshift+\safetyoperatoroffset)}, text width=2.6cm, rounded corners, fill=color_grey!35, draw=black, densely dotted] (human_operators) {\footnotesize \textbf{Human supervisor}};

  \end{pgfonlayer}

  %%%}

  %%%{ lines

  \draw [arrow, ultra thick] ($(trajectorygeneration.east) + (0.2, -0.2)$) -- ($(trajectorygeneration.east) + (1.4, -0.2)$);
  \draw [arrow, ultra thick] ($(assignment.east) + (0.0, 0.0)$) -- ($(trajectorygeneration.west) + (-0.00, 0.0)$);
  \draw [arrow, ultra thick] ($(assignment.west) + (-1.35, -0.2)$) -- ($(assignment.west) + (-0.11, -0.2)$);
  % \draw [arrow, color=arrow_red] ($(operators_computer.west) + (0.0, -\varrowoffset)$) -| node[above, shift={(0, 0)}] {\footnotesize} ($(l_mission_control.west) + (-\harrowshiftoperator, \varrowoffset)$) |- ($(l_mission_control.west) + (0.0, \varrowoffset)$);

  %%%}

  %%{ backgrounds
   \begin{pgfonlayer}{background}
     \path (initconfiguration.west |- initconfiguration.north)+(-0.45,0.45) node (a) {};
     \path (goalconfiguration.south -| goalconfiguration.east)+(+0.45,-0.45) node (b) {};
     \path[fill=color_grey!35,rounded corners, draw=none, densely dotted]
     (a) rectangle (b);
   \end{pgfonlayer}
   \node [rectangle, above of=initconfiguration, node distance=1.7em, shift={(0.0, 0.65)}] (text_init) {\footnotesize \textbf{Initial configuration}};
   \node [rectangle, below of=goalconfiguration, node distance=1.7em, shift={(0.00, -0.65)}] (text_init) {\footnotesize \textbf{Goal configuration}};

  \begin{pgfonlayer}{background}
    \path (assignment.west |- assignment.north)+(-0.45,0.5) node (a) {};
    \path (trajectorygeneration.south -| trajectorygeneration.east)+(+0.45,-0.95) node (b) {};
    \path[fill=color_grey!35,rounded corners, draw=none, densely dotted]
    (a) rectangle (b);
  \end{pgfonlayer}
  \node [rectangle, below of=assignment, node distance=1.7em, shift={(0.0,-1.75)}] (text_assignment) {\footnotesize \textbf{\begin{tabular}{c}Minimum-maximum-length \\ robot-to-goal assignment with \\ collision avoidance guarantees \end{tabular}}};
  \node [rectangle, below of=trajectorygeneration, node distance=1.7em, shift={(0.0,-1.65)}] (text_trajectory_generation) {\footnotesize \textbf{\begin{tabular}{c} Minimum-time collision-free\\ trajectory generation \end{tabular}}};

    \node [rectangle, above of=trajectorygeneration, node distance=1.7em, shift={(-2.1,1.55)}] (text_catora) { \normalsize \textbf{\ac{catora}}};

  \begin{pgfonlayer}{background}
    \path (trajectories.west |- trajectories.north)+(-0.45,0.5) node (a) {};
    \path (trajectories.south -| trajectories.east)+(+0.45,-0.95) node (b) {};
    \path[fill=color_grey!35,rounded corners, draw=none, densely dotted]
    (a) rectangle (b);
  \end{pgfonlayer}
  \node [rectangle, below of=trajectories, node distance=1.7em, shift={(0.0,-1.75)}] (text_trajectories) {\footnotesize \textbf{\begin{tabular}{c} Set of collision-free trajectories \\ for time-optimal formation \\ reshaping \end{tabular}}};

    \node [rectangle, above of=trajectories, node distance=1.7em, shift={(-0.0,1.55)}] (text_result) { \normalsize \textbf{Result}};

  %%%}

\end{tikzpicture}

%% file: fig/trapezoids/trapezoid_general_case.tex
\begin{tikzpicture}[line cap=round]
  \coordinate (si) at (0,0);
  \coordinate (gj) at ( 10:4.1);
  \coordinate (sj) at ( 80:1.4);
  \coordinate (gi) at ( 25:3.3);
  % \coordinate (A+B) at ($(A)+(B)$);

  \draw[ultra thick, dashed, mypurple!40] (sj) -- (gj) node[midway,above] {$Md$};
  \draw[ultra thick, dashed, mypurple!40] (si) -- (gj) node[midway,below] {$d$};

  \draw[myblue,fill=myblue] (si) circle (.3ex) node[midway,below] {$\mathbf{s}_j$};
  \draw[myblue,fill=myblue] (sj) circle (.3ex) node[midway,above, shift={(-0.1, 1.3)}] {$\mathbf{s}_i$};
  \draw[myred,fill=myred] (gi) circle (.3ex) node[midway,above, shift={(3.2, 1.5)}] {$\mathbf{g}_j$};
  \draw[myred,fill=myred] (gj) circle (.3ex) node[midway,above, shift={(4.2, 0.1)}] {$\mathbf{g}_i$};

  \draw[vector,arrow,myblue] (sj) -- (si) node[midway,above left] {$\mathbf{s}_{ij}$};
  \draw[vector,arrow,myred] (gj) -- (gi) node[midway,above right] {$\mathbf{g}_{ij}$};
  % \draw[vector,arrow,mypurple] (O) -- (A+B) node[above right=-3] {$\vb{a}+\vb{b}$};

  % \draw[vector,thin arrow,myblue!40] (A) -- (A+B) node[midway,below right=-2] {$\vb{b}$};
\end{tikzpicture}

%% file: fig/trapezoids/trapezoid_general_case_with_angles.tex
\begin{tikzpicture}[my angle/.style={font=\normalsize, draw, thick, angle eccentricity=1.48, angle radius=5mm}, line cap=round]
  \coordinate (si) at (0,0);
  \coordinate (gj) at (10:4.1);
  \coordinate (sj) at (80:1.4);
  \coordinate (gi) at (27:4.0);
  % \coordinate (A+B) at ($(A)+(B)$);

  \centerarc[black!30,thick, dashed](gj)(65:275:1.20)
  \draw[dashed, thick, mypurple!80, opacity=0.7] (sj) -- (gj) node[midway,above] {$Md$};

  \pic [my angle, "$\beta$"] {angle=gj--si--sj};
  \pic [my angle, "$\gamma$"] {angle=gi--gj--si};

  \draw[myblue,fill=myblue] (si) circle (.3ex) node[midway,below] {$\mathbf{s}_j$};
  \draw[myblue,fill=myblue] (sj) circle (.3ex) node[midway,above, shift={(-0.1, 1.3)}] {$\mathbf{s}_i$};
  \draw[myred,fill=myred] (gi) circle (.3ex) node[midway,above, shift={(3.4, 1.9)}] {$\mathbf{g}_j$};
  \draw[myred,fill=myred] (gj) circle (.3ex) node[midway,above, shift={(4.2, 0.1)}] {$\mathbf{g}_i$};

  \draw[vector,arrow,myred] (si) -- (gj) node[midway,below] {$d$};
  \draw[vector,arrow,myblue] (sj) -- (si) node[midway,above left] {$\mathbf{s}_{ij}$};
  \draw[vector,arrow,myred] (gj) -- (gi) node[midway,above right] {$\mathbf{g}_{ij}$};
  % \draw[vector,arrow,mypurple] (O) -- (A+B) node[above right=-3] {$\vb{a}+\vb{b}$};

  % \draw[vector,thin arrow,myblue!40] (A) -- (A+B) node[midway,below right=-2] {$\vb{b}$};
\end{tikzpicture}

%% file: fig/blap_algorithm_scheme.tex
\definecolor{color_red}{HTML}{A30D00}
\definecolor{color_green}{rgb}{0, .522, .243}
\definecolor{color_blue}{rgb}{0, 0, .9}

\pgfmathsetmacro{\vshift}{0.35} % shift in vertical direction
\pgfmathsetmacro{\hshift}{2.39} % shift in vertical direction
\pgfmathsetmacro{\hshiftsecondcol}{7.0} % shift in vertical direction
\pgfmathsetmacro{\harrowshift}{0.4} % shift in vertical direction
\pgfmathsetmacro{\varrowshift}{0.2} % shift in vertical direction
\pgfmathsetmacro{\varrowbend}{0.35} % shift in vertical direction
% \pgfmathsetmacro{\hoffset}{0.00em} % shift in vertical direction
% \pgfmathsetmacro{\folloffset}{0.32em} % shift in vertical direction
% \pgfmathsetmacro{\folloffsetn}{0.12em} % shift in vertical direction
% \pgfmathsetmacro{\operatoroffset}{0.07em} % shift in vertical direction
% \pgfmathsetmacro{\postdeployoffset}{0.30em} % shift in vertical direction
% \pgfmathsetmacro{\postdeployvoffset}{0.08em} % shift in vertical direction
% \pgfmathsetmacro{\predeployoffset}{-0.40em} % shift in vertical direction
% \pgfmathsetmacro{\predeployvoffset}{-0.02em} % shift in vertical direction
% \pgfmathsetmacro{\safetyoperatoroffset}{-0.18em} % shift in vertical direction
% \pgfmathsetmacro{\legendoffset}{-0.12em} % shift in vertical direction
% \pgfmathsetmacro{\sensorsshift}{1.03em} % shift in vertical direction
% \pgfmathsetmacro{\trackingshift}{0.4em} % shift in vertical direction
% \pgfmathsetmacro{\harrowoffset}{0.12em} % shift in vertical direction
% \pgfmathsetmacro{\varrowoffset}{0.18em} % shift in vertical direction
% \pgfmathsetmacro{\varrowshift}{0.18em} % shift in vertical direction
% \pgfmathsetmacro{\harrowshiftpredeploy}{2.5em} % shift in vertical direction
% \pgfmathsetmacro{\harrowshiftpostdeploy}{1.3em} % shift in vertical direction
% \pgfmathsetmacro{\harrowshiftoperator}{0.4em} % shift in vertical direction
% \pgfmathsetmacro{\operatorswitchshift}{0.2em} % shift in vertical direction
% \pgfmathsetmacro{\harrowbend}{0.1em} % horizontal shift of arrows
% \pgfmathsetmacro{\harrowshiftscan}{0.22em} % horizontal shift of arrows

\tikzstyle{block}=[draw, rounded corners, text centered, ultra thick, minimum height=2.5em, minimum width=5.8em, inner sep=1pt, fill=white, fill opacity=1.0, text opacity=1.0]
\tikzstyle{block_filter}=[draw, rounded corners, text centered, minimum height=2.0em,  minimum width=5.6em, fill=white, fill opacity=1.0, text opacity=1.0]
\tikzstyle{block_perf}=[draw, rounded corners, text centered, minimum height=2.0em, minimum width=5.6em, fill=white, fill opacity=1.0, text opacity=1.0]
\def\nodedst{2.0cm}

\begin{tikzpicture}[auto, node distance=1.0cm, >=latex, font=\scriptsize]

  % \renewcommand{\arraystretch}{0.6}

  %%%{ nodes

  % deployment phase - uav
    \node [block] (branch_solution) {
        \begin{tabular}{c}
          \footnotesize Branch solution \\
          \footnotesize and solve assignment \\
        \end{tabular}};
  
    \node [block, above of=branch_solution, shift = {(0.0cm, \vshift)}] (update_threshold) {
        \begin{tabular}{c}
          \footnotesize Update \\
          \footnotesize threshold \\
    \end{tabular}};

    \node [block, left of=branch_solution, fill=red!10, shift = {(-\hshift, 0.0cm)}] (set_initial) {
        \begin{tabular}{c}
          \footnotesize Set initial \\
          \footnotesize threshold \\
    \end{tabular}};

    \node [block, right of=branch_solution, shift = {(\hshift, 0.0cm)}] (check_collisions) {
        \begin{tabular}{c}
          \footnotesize Check \\
          \footnotesize collisions \\
    \end{tabular}};

    \node [block, above of=check_collisions, shift = {(0.0cm, \vshift)}] (check_completeness) {
        \begin{tabular}{c}
          \footnotesize Check \\
          \footnotesize completeness \\
    \end{tabular}};

    \node [block, above of=set_initial, shift = {(0.0cm, \vshift)}] (solve_assignment) {
        \begin{tabular}{c}
          \footnotesize Solve \\
          \footnotesize assignment \\
    \end{tabular}};

    \node [block, fill=green!10, below of=branch_solution, minimum width=25.0em, shift = {(0.0cm, -\vshift)}] (result) {
        \begin{tabular}{c}
          \footnotesize \textbf{Collision-free robot-to-goal assignment minimizing makespan}
    \end{tabular}};

  \draw [->, color=color_red, ultra thick] (branch_solution.north) -- (update_threshold.south); % align
  \draw [->, color=color_red, ultra thick] (check_completeness.west) -- (update_threshold.east); % align
  \draw [->, color=color_red, ultra thick] (check_collisions.west) -- (branch_solution.east); % align
  \draw [->, color=color_green, ultra thick] (check_completeness.south) -- (check_collisions.north); % align
  \draw [->, color=color_green, ultra thick] (branch_solution.south) -- (result.north); % align
  \draw [->, ultra thick] (set_initial.north) -- (solve_assignment.south); % align
  \draw [->, ultra thick] (update_threshold.west) -- (solve_assignment.east); % align
  \draw [->, ultra thick] (solve_assignment.north) |- ($(solve_assignment.north) + (0.0, \varrowbend)$) -| (check_completeness.north); % align
  \draw [->, ultra thick, color=color_green] (check_collisions.south) -- ($(check_collisions.south) + (0.0, -\vshift-0.07)$); % align

\end{tikzpicture}

%% file: fig/graphs/minimum_time_trajectories.tex
\begin{tikzpicture}[font=\normalsize]

  \pgfplotsset{compat=newest, 
    emphasize/.code args={#1:#2with#3}{
      \pgfplotsextra{
              \draw[color=#3, fill=#3] ({axis cs:#1,0.00} |- {axis description cs:0,0.006}) 
              rectangle ({axis cs:#2,0} |- {axis description cs:0,0.994});
      }
    }
  }

  \def\A{1.0}
  \def\ta{2.0}
  \def\tc{4.0}
  \def\td{6.0}
  \def\V{\A*\ta}
  \def\D{\A*\ta*\ta+\A*\ta*(\tc-\ta)}
  \def\Am{\A*0.5}
  \def\Vm{\Am*\ta}

  %%{ BACKGROUND EMPHASIZE
  
    \begin{groupplot}[
      group style={
            % set how the plots should be organized
            group size=1 by 3,
            % only show ticklabels and axis labels on the bottom
            x descriptions at=edge bottom,
            % set the `vertical sep' to zero
            vertical sep=5pt,
      },
      name=top,
      width=1.00\columnwidth,
      height=0.35\columnwidth,
      xmin=0.0, xmax=1.0,
      scaled x ticks = false,
      xtick=\empty,
      axis line style=transparent,
      ytick=\empty,
      xmin=0.0, xmax=\tc+\ta,
      ymin=-0.3, ymax=6.0,
      ]
  
      \nextgroupplot[
        height=0.35\columnwidth,
        % ylabel= $v$ $\left(\text{m.s}^\text{-1}\right)$,
      ]
  
  \addplot[draw=color_red, line width=1.5pt, opacity=0.0, emphasize=0.0:\ta with color_green!10][domain=0:\ta]{\A*x};
  \addplot[draw=color_red, line width=1.5pt, opacity=0.0, emphasize=\ta:\tc with color_blue!0][domain=\ta:\tc]{\A*\ta};
  \addplot[draw=color_red, line width=1.5pt, opacity=0.0, emphasize=\tc:\td with color_red!10][domain=\tc:\td]{\A*\ta-\A*(x-\tc)};

    \nextgroupplot[
      % ylabel= d (m),
      ymin=-0.3, ymax=10.0,
    ]
  
    \addplot[draw=color_red, line width=1.5pt, opacity=0.0, emphasize=0.0:\ta with color_green!10][domain=0:\ta]{0.5*\A*x^2};
      \addplot[draw=color_red, line width=1.5pt, opacity=0.0, emphasize=\tc:\td with color_red!10][domain=\tc:\td]{0.5*\A*\ta^2+\V*(\tc-\ta) + \V*(x-\tc) - 0.5*\A*(x-\tc)^2};
  
    \nextgroupplot[
      % ylabel= d (m),
      ymin=-0.3, ymax=10.0,
    ]
  
    \addplot[draw=color_red, line width=1.5pt, opacity=0.0, emphasize=0.0:\ta with color_green!10][domain=0:\ta]{0.5*\A*x^2};
      \addplot[draw=color_red, line width=1.5pt][domain=\ta:\tc]{0.5*\A*\ta^2+\V*(x-\ta)};
      \addplot[draw=color_red, line width=1.5pt, opacity=0.0, emphasize=\tc:\td with color_red!10][domain=\tc:\td]{0.5*\A*\ta^2+\V*(\tc-\ta) + \V*(x-\tc) - 0.5*\A*(x-\tc)^2};
  
    \end{groupplot}
  
  %%}

    %%{ POSITION AND VELOCITY PLOTS
    
      \begin{groupplot}[
        group style={
            % set how the plots should be organized
            group size=1 by 3,
            % only show ticklabels and axis labels on the bottom
            x descriptions at=edge bottom,
            % set the `vertical sep' to zero
            vertical sep=5pt,
        },
        name=top,
        width=1.00\columnwidth,
        height=0.35\columnwidth,
        grid=major,
        grid style={draw=black!12,line width=.1pt},
        xlabel={Time [s]},
        xtick={0, 2.0, 4.0, 6.0},
        xticklabels={0, $t_1$, $t_2$, $t_3$},
        scaled x ticks = false,
        xmin=0.0, xmax=\tc+\ta,
        ymin=-0.3, ymax=3.0,
        ]
    
        %% | ---------------------- ACCELERATION ---------------------- |
        \nextgroupplot[scaled y ticks=false,
          ylabel={$a$ $\left[\text{m.s}^\text{-2}\right]$},
          ymin=-1.1, ymax=1.1,
          ytick={-1.0, -0.5, 0, 0.5, 1.0},
          yticklabels={-$A_M$, -$A_i$, 0, $A_i$, $A_M$},
          ylabel shift = -3.7 pt,
        ]
      
        \addplot[draw=color_red, line width=1.5pt][domain=0:\ta]{\A};
        \addplot[draw=color_red, line width=1.5pt][domain=\ta:\tc]{0.0};
        \addplot[draw=color_red, line width=1.5pt][domain=\tc:\td]{-\A};
      
        \addplot[draw=color_blue, line width=1.5pt][domain=0:\ta]{\Am};
        \addplot[draw=color_blue, line width=1.5pt, dash pattern = on 8pt off 4pt][domain=\ta:\tc]{0.0};
        \addplot[draw=color_blue, line width=1.5pt][domain=\tc:\td]{-\Am};

      %% | ------------------------ VELOCITY ------------------------ |
        \nextgroupplot[
          ylabel={$v$ $\left[\text{m.s}^\text{-1}\right]$},
          ymin=-0.2, ymax=2.2,
          ytick={0, 1.0, 2.0},
          yticklabels={0, $V_i$, $V_M$},
          ylabel shift = 0.6 pt,
        ]
    
        \addplot[draw=color_red, line width=1.5pt][domain=0:\ta]{\A*x};
        \addplot[draw=color_red, line width=1.5pt][domain=\ta:\tc]{\A*\ta};
        \addplot[draw=color_red, line width=1.5pt][domain=\tc:\td]{\A*\ta-\A*(x-\tc)};
        
        \addplot[draw=color_blue, line width=1.5pt][domain=0:\ta]{\Am*x};
        \addplot[draw=color_blue, line width=1.5pt][domain=\ta:\tc]{\Am*\ta};
        \addplot[draw=color_blue, line width=1.5pt][domain=\tc:\td]{\Am*\ta-\Am*(x-\tc)};
    
        %% | ------------------------ POSITION ------------------------ |
        \nextgroupplot[
          ylabel={$p$ [m]},
          ymin=-0.3, ymax=10.0,
          grid style={draw=black!12,line width=.1pt},
          ytick={0, 4.0, 8.0},
          yticklabels={0, $D_i$, $D_M$},
        ]
      
        \addplot[draw=color_red, line width=1.5pt][domain=0:\ta]{0.5*\A*x^2};
          \addplot[draw=color_red, line width=1.5pt][domain=\ta:\tc]{0.5*\A*\ta^2+\V*(x-\ta)};
          \addplot[draw=color_red, line width=1.5pt][domain=\tc:\td]{0.5*\A*\ta^2+\V*(\tc-\ta) + \V*(x-\tc) - 0.5*\A*(x-\tc)^2};
      
        \addplot[draw=color_blue, line width=1.5pt][domain=0:\ta]{0.5*\Am*x^2};
          \addplot[draw=color_blue, line width=1.5pt][domain=\ta:\tc]{0.5*\Am*\ta^2+\Vm*(x-\ta)};
          \addplot[draw=color_blue, line width=1.5pt][domain=\tc:\td]{0.5*\Am*\ta^2+\Vm*(\tc-\ta) + \Vm*(x-\tc) - 0.5*\Am*(x-\tc)^2};
    
      \end{groupplot}
    
    %%}

  \end{tikzpicture}

%% file: fig/graphs/length_suboptimality.tex
\definecolor{color_blue}{rgb}{0.22, 0.2, 0.502}
\definecolor{color_red}{rgb}{0.737,0.165,0}
\definecolor{color_green}{rgb}{0, .522, .243}

  \begin{tikzpicture}[font=\normalsize]

    \pgfplotstableread[col sep=comma]{./fig/graphs/tikz_data/suboptimality/length_suboptimality.txt}{\tablea}

  \begin{axis}[
    axis y line*=right,
    width=0.9\columnwidth,
    height=0.50\columnwidth,
    % grid=major,
    % grid style={draw=gray!12,line width=.1pt},
    % xlabel= N (-),
    % ylabel= Suboptimality (-),
    ytick={1.00, 1.05, 1.10, 1.15, 1.20},
    yticklabels={1.00, 1.05, 1.10, 1.15, 1.20},
    % ymode=log,
    xmin=5, xmax=178,
    y axis line style = {color_red},
    y tick label style= {color_red},
    y tick style= {color_red},
    ylabel style = {color_red},
    ymin=0.98, ymax=1.22,
    % max space between ticks=20,
  ]
      \addplot[name path=O, smooth, color=color_red, line width=1.0pt, opacity=1.0, mark options={color=color_red, scale=0.7}] table[y expr={\thisrow{max}}, x=n_robots] {\tablea};

    \end{axis}

    \begin{axis}[
      width=0.9\columnwidth,
      height=0.50\columnwidth,
      grid=major,
      grid style={draw=gray!12,line width=.1pt},
      xlabel={Number of robots [-]},
      ylabel={Suboptimality [-]},
      % ymode=log,
      xmin=5, xmax=178,
      ytick={1.0, 1.005, 1.010, 1.015, 1.020},
      yticklabels={1.000, 1.005, 1.010, 1.015, 1.020},
      ymin=0.998, ymax=1.022,
      % max space between ticks=20,
      legend style={at={(0.950,0.98)}, legend columns = 4}
      ]

      \addplot[name path=M, smooth, color=black, line width=1.0pt, opacity=1.0, mark options={color=black, scale=0.5}] table[y expr={\thisrow{q95}}, x=n_robots] {\tablea};
      \addplot[name path=N, smooth, color=color_blue, line width=1.0pt, opacity=1.0, mark options={color=color_blue, scale=0.7}] table[y expr={\thisrow{q975}}, x=n_robots] {\tablea};
      %legend image for graph from previous axis 
      \addplot[smooth, color=color_red, line width=1.0pt, opacity=1.0, mark options={color=color_red, scale=0.7}] coordinates {(-10.0, 0) (-5.0, 0)};

      \addplot[name path=P, smooth, color=color_green, line width=1.0pt, opacity=1.0, mark options={color=green, scale=0.7}] table[y expr={\thisrow{mean}}, x=n_robots] {\tablea};

      \addplot[color=color_red, opacity=0.2] fill between [of=O and M];
      \addlegendentry{$Q_{0.95}$}
      \addlegendentry{$Q_{0.975}$}
      \addlegendentry{$Q_{1.0}$}
      \addlegendentry{$\mu$}
    \end{axis}

  \end{tikzpicture}

%% file: fig/graphs/lsap_vs_lbap_length.tex
\begin{tikzpicture}
  \pgfplotstableread[col sep=comma]{./fig/graphs/tikz_data/length_boxplots/density_01/length_ratio_boxplot_data.txt}{\tablea}
  \pgfplotstableread[col sep=comma]{./fig/graphs/tikz_data/length_boxplots/density_001/length_ratio_boxplot_data.txt}{\tableb}
  \pgfplotstableread[col sep=comma]{./fig/graphs/tikz_data/length_boxplots/density_0001/length_ratio_boxplot_data.txt}{\tablec}
  % \pgfplotstabletranspose[colnames from=Team]\transposeddata{\datatable}

  \begin{axis}[
      boxplot/draw direction=y,
      ylabel={$D_{m,CAT-ORA}/D_{m,LSAP}$ [-]},
      height=5.3cm,
      width=\columnwidth,
      ymin=0.35,ymax=1.08,
      cycle list={{white!20!blue},{white!20!red},{white!40!black}},
      restrict y to domain=0.0:1.0,
      yticklabel style={inner sep=0pt, align=right, xshift=-0.1cm},
      boxplot={
        %
        % Idea:
        %  place the
        %  group 1 at 0.3333 and 0.6666
        %  group 2 at 1.3333 and 1.6666
        %  group 3 at 2.3333 and 2.6666
        %  ...
        % in a formular:
        draw position={1/5 + floor(\plotnumofactualtype/3) + 1/5*fpumod(\plotnumofactualtype,3)},
        %
        % that means the box extend must be at most 0.33333 :
        box extend=0.15
      },
      % ... it also means that 1 unit in x controls the width:
      x=1.05cm,
      % ... and it means that we should describe intervals:
      xmin=-0.1,xmax=7.1,
      xtick={0,1,2,...,20},
      x tick label as interval,
      xlabel={Number of robots [-]},
      % xticklabels from table={\tablea}{name},
      xticklabels={%
        {4},%
        {8},%
        {16},%
        {32},%
        {64},%
        {128},%
        {256},%
        },
        x tick label style={
          text width=2.5cm,
          align=center
        },
        every axis plot/.append style={fill,fill opacity=0.4},
        legend style={at={(0.145,0.22)},anchor=north west, legend columns=-1},
        legend image code/.code={\draw [#1] (0cm,-0.1cm) rectangle (0.2cm,0.25cm); },
      ]

      \addplot table [y index = 0] {\tablea};
      \addplot table [y index = 0] {\tableb};
      \addplot table [y index = 0] {\tablec};

      \addplot table [y index = 1] {\tablea};
      \addplot table [y index = 1] {\tableb};
      \addplot table [y index = 1] {\tablec};

      \addplot table [y index = 2] {\tablea};
      \addplot table [y index = 2] {\tableb};
      \addplot table [y index = 2] {\tablec};

      % \addplot+[boxplot={draw position=3.25}] table [y index = 3] {\tablea};
      \addplot table [y index = 3] {\tablea};
      \addplot table [y index = 3] {\tableb};
      \addplot table [y index = 3] {\tablec};

      \addplot table [y index = 4] {\tablea};
      \addplot table [y index = 4] {\tableb};
      \addplot table [y index = 4] {\tablec};

      \addplot table [y index = 5] {\tablea};
      \addplot table [y index = 5] {\tableb};
      \addplot table [y index = 5] {\tablec};

      % \addplot+[boxplot={draw position=6.25}] table [y index = 6] {\tablea};
      \addplot table [y index = 6] {\tablea};
      \addplot table [y index = 6] {\tableb};
      \addplot table [y index = 6] {\tablec};

      % \setcounter{iloop}{0}
      % \foreach \Color/\Text in {red/density,black/density}
      % {\edef\temp{\noexpand\addlegendimage{boxplot legend=\Color}
      % \noexpand\addlegendentry[\Color]{\Text}}
      % \temp}
      \legend{$d_r=0.1$, $d_r=0.01$, $d_r=0.001$}
      % \addlegendentry{density = 0.01}

      %%------------------------------%

          % \addplot
      % table[row sep=\\,y index=0] {
          % data\\
          % 1.405\\
          % 1.188\\
          % 4.330\\
          % 3.665\\
          % 1.439\\
      % };

      % \addplot
      % table[row sep=\\,y index=0] {
          % data\\
          % 2.937\\
          % 1.320\\
          % 1.357\\
          % 1.852\\
          % 1.215\\
      % };
      %----------------------%

    \end{axis}
  \end{tikzpicture}

%% file: fig/graphs/lsap_vs_lbap_time.tex
\begin{tikzpicture}
  \pgfplotstableread[col sep=comma]{./fig/graphs/tikz_data/time_boxplots/density_0001/time_lsap_boxplot_data.txt}{\tablea}
  \pgfplotstableread[col sep=comma]{./fig/graphs/tikz_data/time_boxplots/density_0001/time_lbap_boxplot_data.txt}{\tableb}
  % \pgfplotstableread[col sep=comma]{./fig/graphs/tikz_data/time_boxplots/density_0001/time_ratio_boxplot_data.txt}{\tablec}
  % \pgfplotstabletranspose[colnames from=Team]\transposeddata{\datatable}

  \begin{axis}[
      boxplot/draw direction=y,
      ylabel={Computational time [ms]},
      height=5cm,
      width=\columnwidth,
      cycle list={{blue},{red}},
      % restrict y to domain=0.0:1.0,
      ymode=log,
      max space between ticks=20,
      yticklabel style={inner sep=0pt, align=right, xshift=-0.1cm},
      boxplot={
        %
        % Idea:
        %  place the
        %  group 1 at 0.3333 and 0.6666
        %  group 2 at 1.3333 and 1.6666
        %  group 3 at 2.3333 and 2.6666
        %  ...
        % in a formular:
        draw position={1/3 + floor(\plotnumofactualtype/2) + 1/3*fpumod(\plotnumofactualtype,2)},
        %
        % that means the box extend must be at most 0.33333 :
        box extend=0.25
      },
      % ... it also means that 1 unit in x controls the width:
      x=1.03cm,
      % ... and it means that we should describe intervals:
      xmin=-0.1,xmax=7.1,
      ymin=0.005,ymax=2000,
      xtick={0,1,2,...,20},
      x tick label as interval,
      xlabel={Number of robots [-]},
      % xticklabels from table={\tablea}{name},
      xticklabels={%
        {4},%
        {8},%
        {16},%
        {32},%
        {64},%
        {128},%
        {256},%
        },
        x tick label style={
          text width=2.5cm,
          align=center
        },
        every axis plot/.append style={fill,fill opacity=0.4},
        legend style={at={(0.525,0.24)},anchor=north west, legend columns=-1},
        legend image code/.code={\draw [#1] (0cm,-0.1cm) rectangle (0.2cm,0.25cm); },
      ]

      \addplot table [y index = 0] {\tablea};
      \addplot table [y index = 0] {\tableb};
      % \addplot table [y index = 0] {\tablec};

      \addplot table [y index = 1] {\tablea};
      \addplot table [y index = 1] {\tableb};
      % \addplot table [y index = 1] {\tablec};

      \addplot table [y index = 2] {\tablea};
      \addplot table [y index = 2] {\tableb};
      % \addplot table [y index = 2] {\tablec};

      % \addplot+[boxplot={draw position=3.25}] table [y index = 3] {\tablea};
      \addplot table [y index = 3] {\tablea};
      \addplot table [y index = 3] {\tableb};
      % \addplot table [y index = 3] {\tablec};

      \addplot table [y index = 4] {\tablea};
      \addplot table [y index = 4] {\tableb};
      % \addplot table [y index = 4] {\tablec};

      \addplot table [y index = 5] {\tablea};
      \addplot table [y index = 5] {\tableb};
      % \addplot table [y index = 5] {\tablec};

      \addplot+[boxplot={draw position=6.25}] table [y index = 6] {\tablea};
      % \addplot table [y index = 6] {\tablea};
      \addplot table [y index = 6] {\tableb};
      % \addplot table [y index = 6] {\tablec};

      % \setcounter{iloop}{0}
      % \foreach \Color/\Text in {red/density,black/density}
      % {\edef\temp{\noexpand\addlegendimage{boxplot legend=\Color}
      % \noexpand\addlegendentry[\Color]{\Text}}
      % \temp}
      \legend{LSAP, CAT-ORA}
      % \addlegendentry{density = 0.01}

      %%------------------------------%

          % \addplot
      % table[row sep=\\,y index=0] {
          % data\\
          % 1.405\\
          % 1.188\\
          % 4.330\\
          % 3.665\\
          % 1.439\\
      % };

      % \addplot
      % table[row sep=\\,y index=0] {
          % data\\
          % 2.937\\
          % 1.320\\
          % 1.357\\
          % 1.852\\
          % 1.215\\
      % };
      %----------------------%

    \end{axis}
  \end{tikzpicture}

%% file: fig/graphs/formation_reshaping.tex
\definecolor{color_blue}{rgb}{0.22, 0.2, 0.502}
\definecolor{color_red}{rgb}{0.737,0.165,0}
\definecolor{color_green}{rgb}{0, .522, .243}

\begin{tikzpicture}[font=\normalsize]

    \begin{groupplot}[
      group style={
            % set how the plots should be organized
            group size=5 by 2,
            % only show ticklabels and axis labels on the bottom
            x descriptions at=edge bottom,
            y descriptions at=edge left,
            % set the `vertical sep' to zero
            vertical sep=4pt,
            horizontal sep=4pt,
      },
      name=top,
      width=0.27\textwidth,
      height=0.27\textwidth,
      grid=major,
      grid style={draw=gray!12,line width=1pt},
      % xlabel= x (m),
      % ylabel= y (m),
      xmin=-60.0,xmax=60.0,
      ymin=-60.0,ymax=60.0,
      axis equal image,
      yticklabels={,,},
      xticklabels={,,},
      xlabel shift=-5pt
      ]

      \nextgroupplot[
        % ylabel= mut. dist. (m),
        % ymin=2.2, ymax=5.2,
        % ylabel shift = 7.7 pt,
        % ytick={3.0,4.0,5.0},
        % yticklabels={3.0,4.0,5.0}
        ylabel= CAT-ORA,
        % xlabel= $T = 0 s$,
      ]

      % \addplot[smooth, color=black, line width=1.0pt, opacity=0.2] table[y=mutual_distance, x=ts_zero] {data/mutual_distances.txt};
      % \addplot[thick, color=blue, on layer=axis background]
      % graphics[xmin=-60,ymin=-60,xmax=60,ymax=60] {./fig/graphs/py_images/reshaping_lbap_5.png};
      % \draw (800,150) [scale bar];

% Adding a scale bar

      % \nextgroupplot[]
      \addplot[thick, color=blue, on layer=axis background]
      graphics[xmin=-80,ymin=-60,xmax=80,ymax=60] {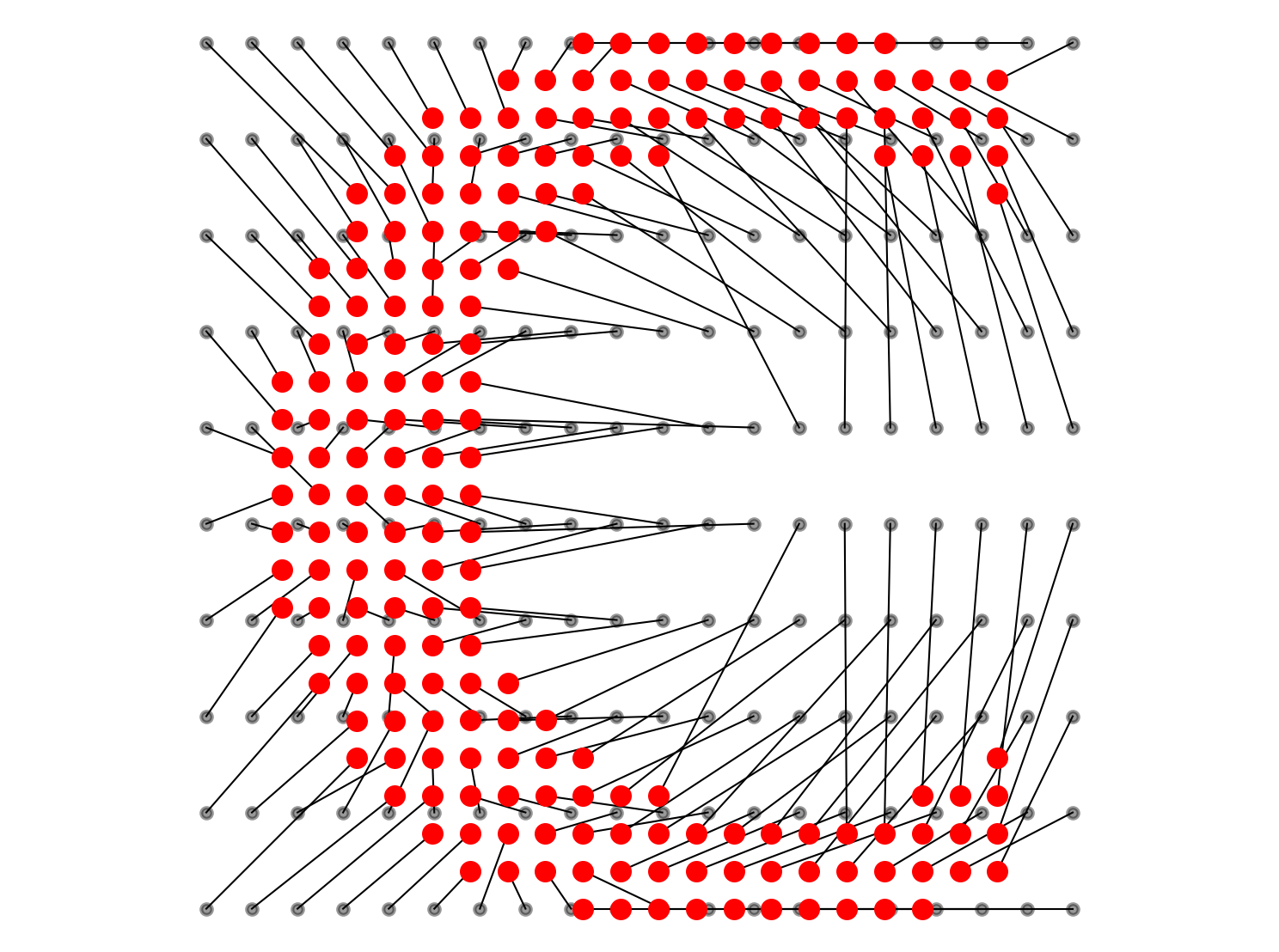};

      \nextgroupplot[]
      \addplot[thick, color=blue, on layer=axis background]
      graphics[xmin=-80,ymin=-60,xmax=80,ymax=60] {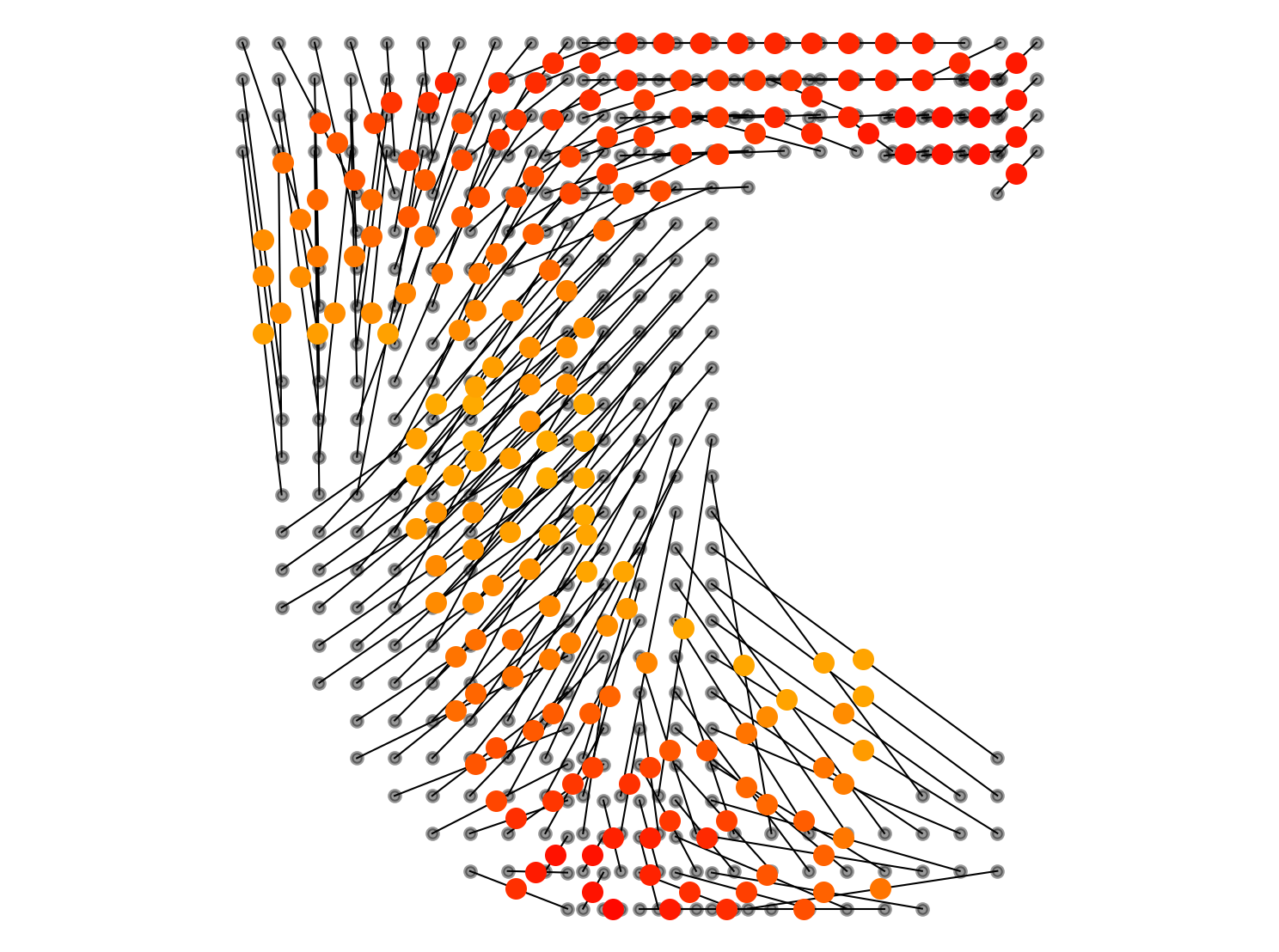};

      \nextgroupplot[]
      \addplot[thick, color=blue, on layer=axis background]
      graphics[xmin=-80,ymin=-60,xmax=80,ymax=60] {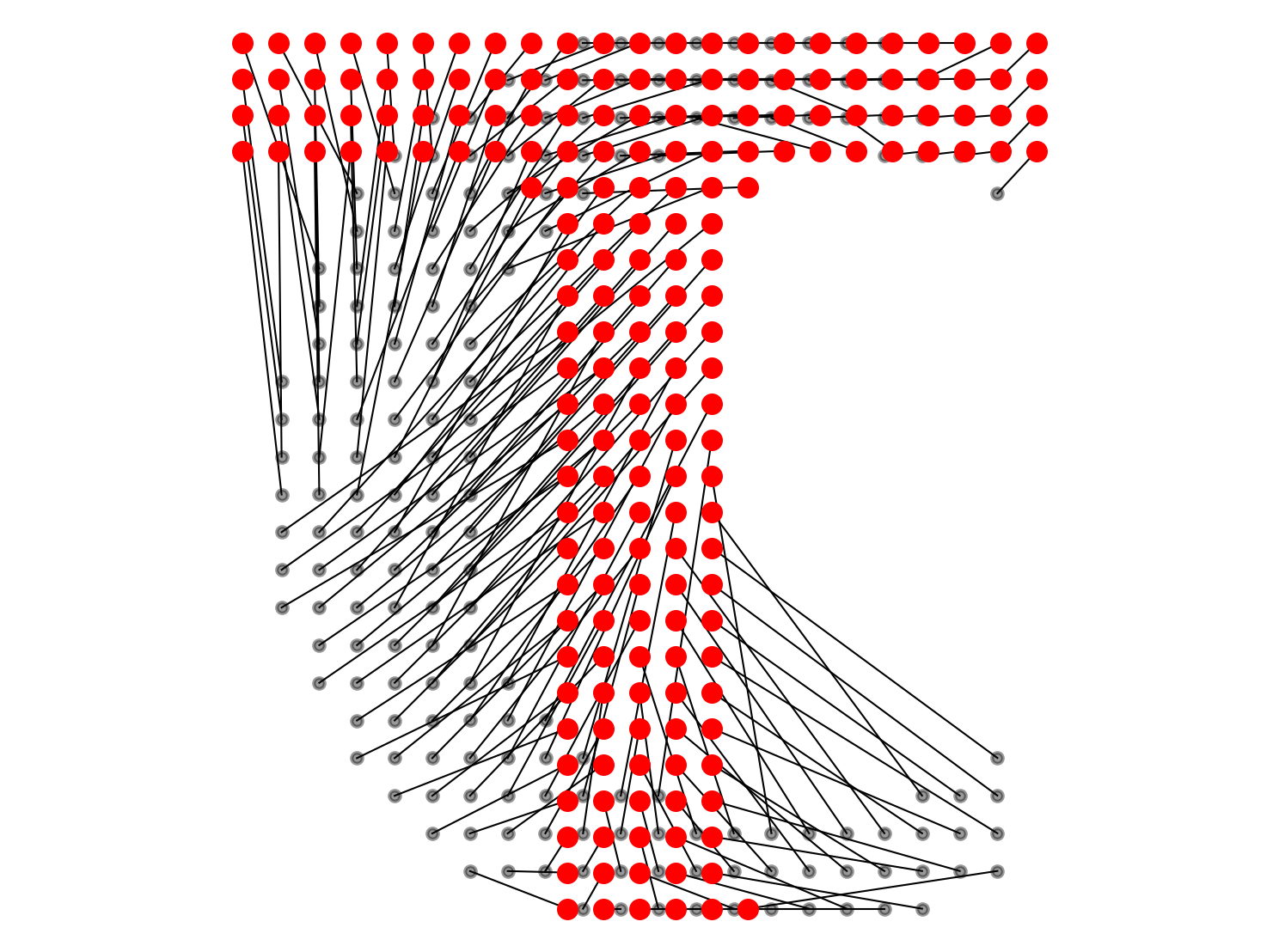};

      \nextgroupplot[]
      \addplot[thick, color=blue, on layer=axis background]
      graphics[xmin=-80,ymin=-60,xmax=80,ymax=60] {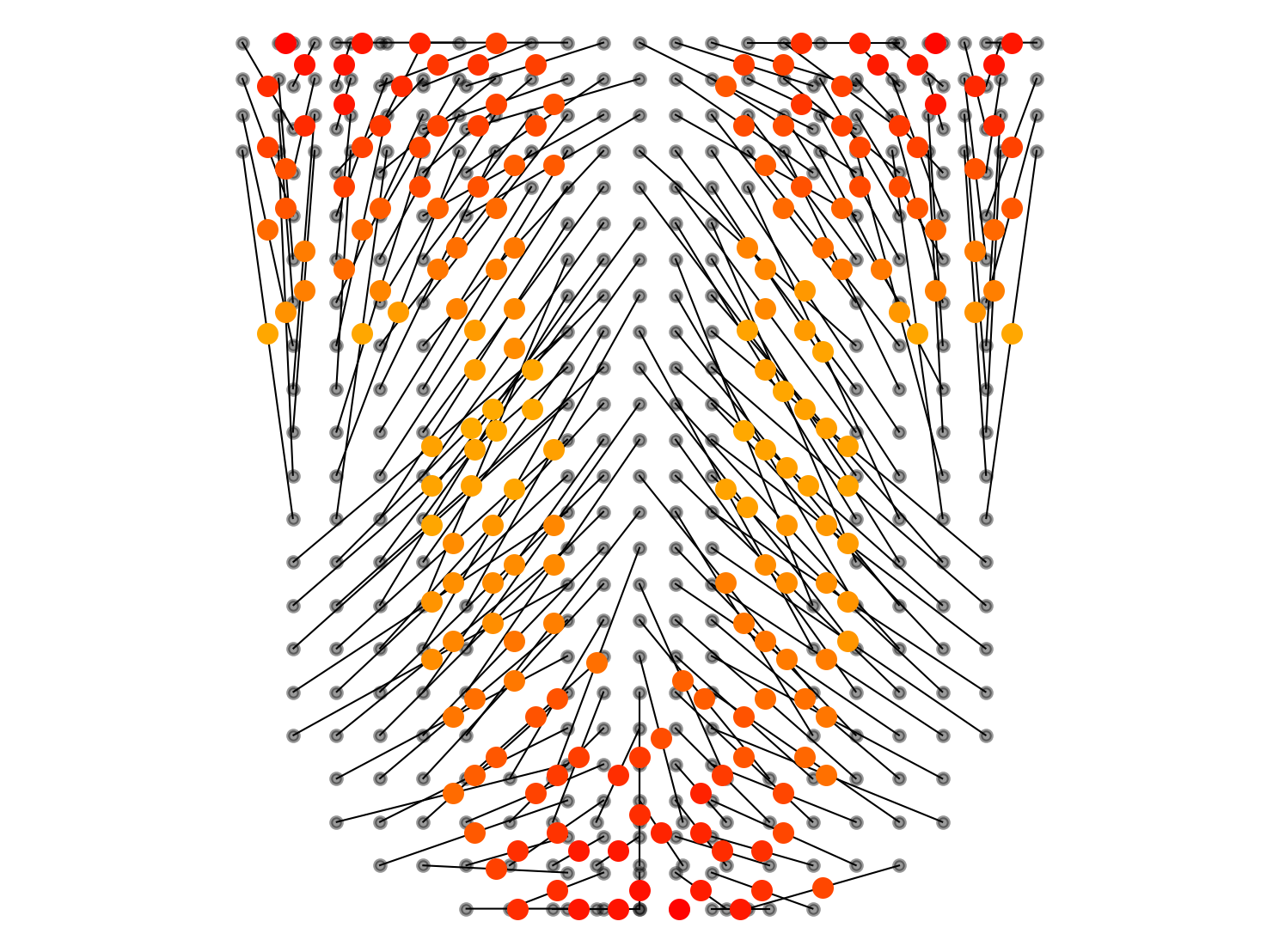};

      \nextgroupplot[]
      \addplot[thick, color=blue, on layer=axis background]
      graphics[xmin=-80,ymin=-60,xmax=80,ymax=60] {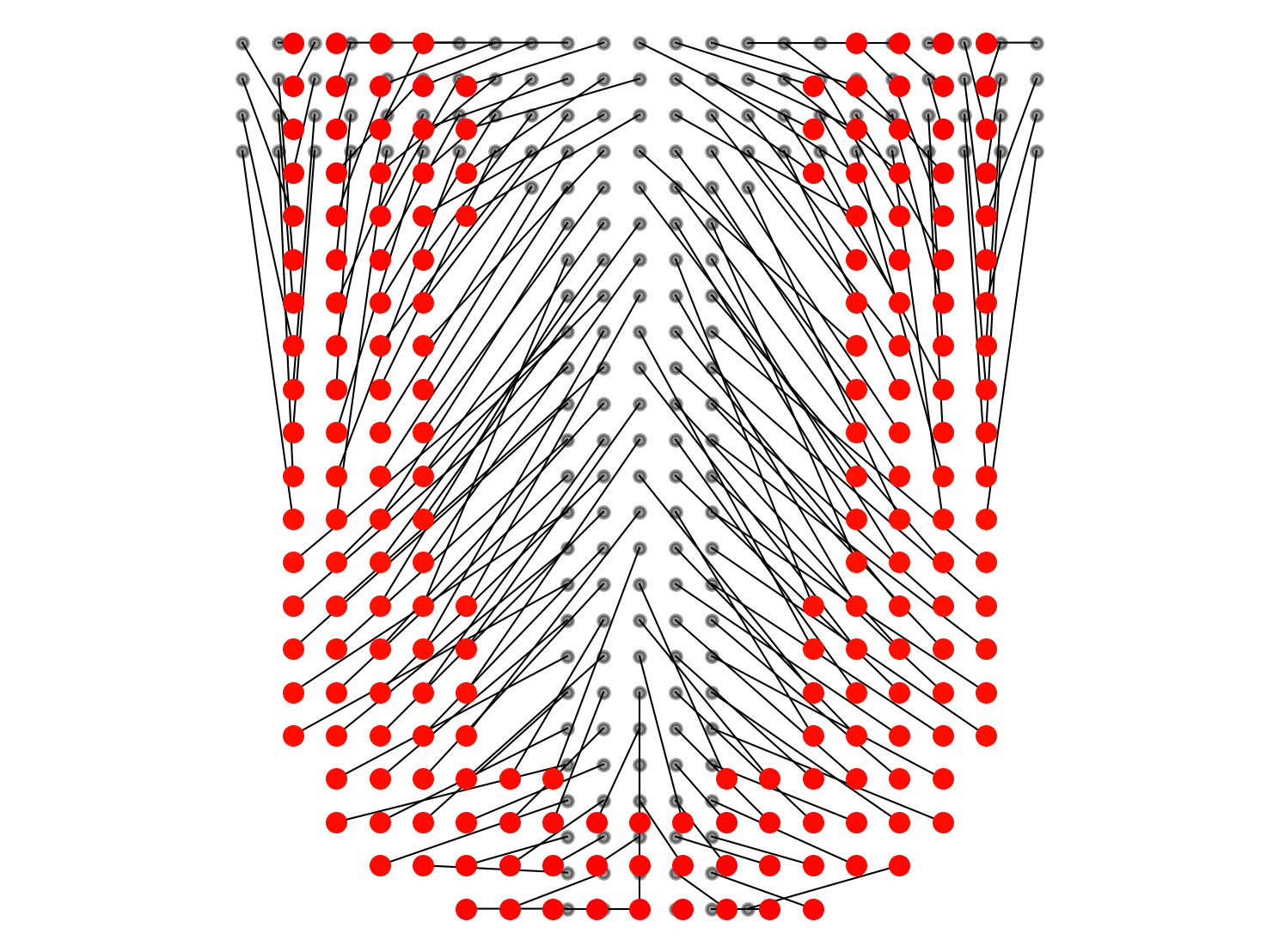};
      
      % begin lsap part
      \nextgroupplot[
        ylabel={LSAP, min. time},
        xlabel={$T = 11.0\,s$},
        ]

      % \addplot[thick, color=blue, on layer=axis background]
      % graphics[xmin=-60,ymin=-60,xmax=60,ymax=60] {./fig/graphs/py_images/reshaping_lsap_5.png};

      % \nextgroupplot[
      %   xlabel={$T = 10\,s$},
      %   ]
      \addplot[thick, color=blue, on layer=axis background]
      graphics[xmin=-80,ymin=-60,xmax=80,ymax=60] {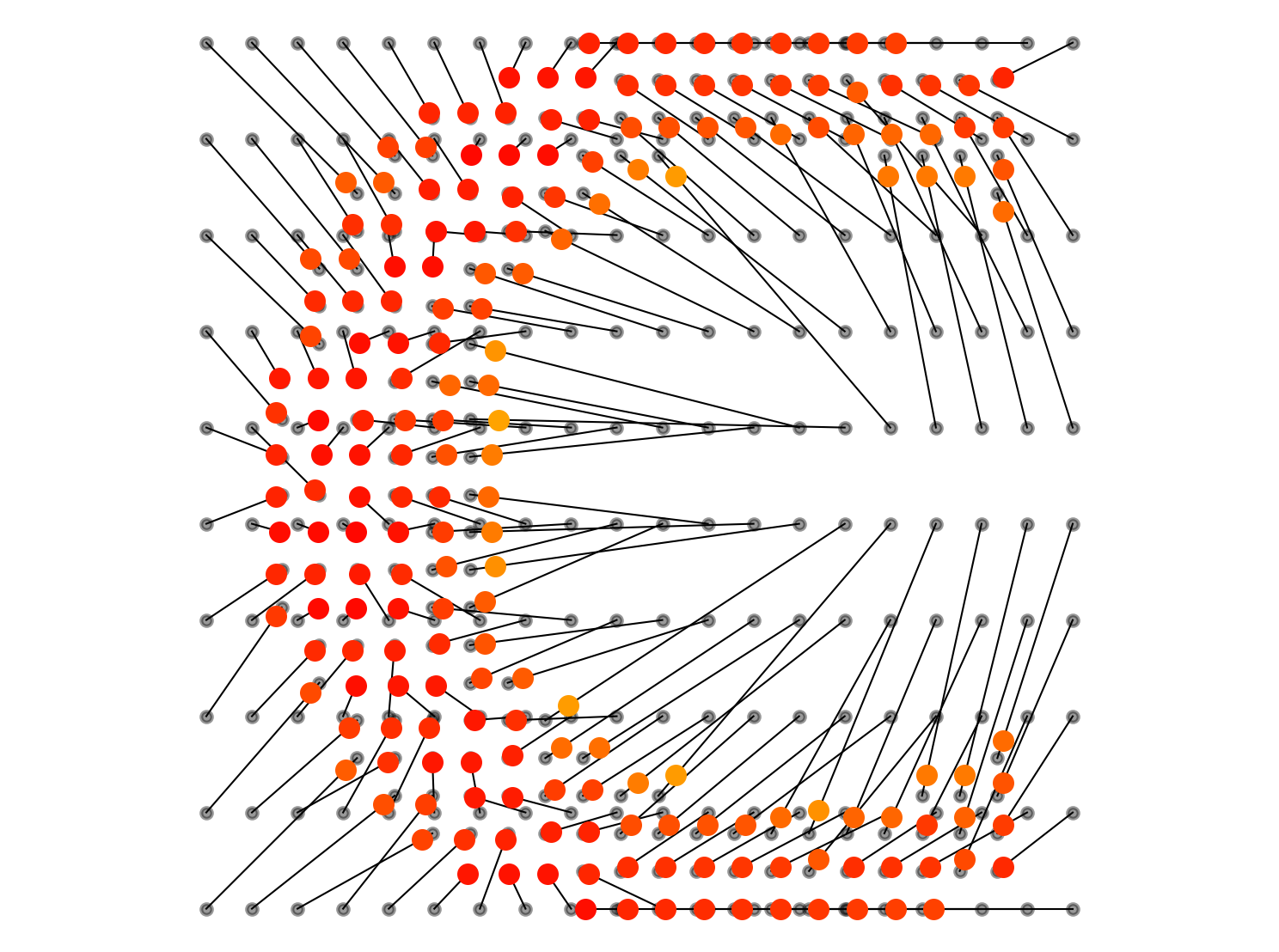};

      \nextgroupplot[
        xlabel={$T = 17.0\,s$},
        ]
      \addplot[thick, color=blue, on layer=axis background]
      graphics[xmin=-80,ymin=-60,xmax=80,ymax=60] {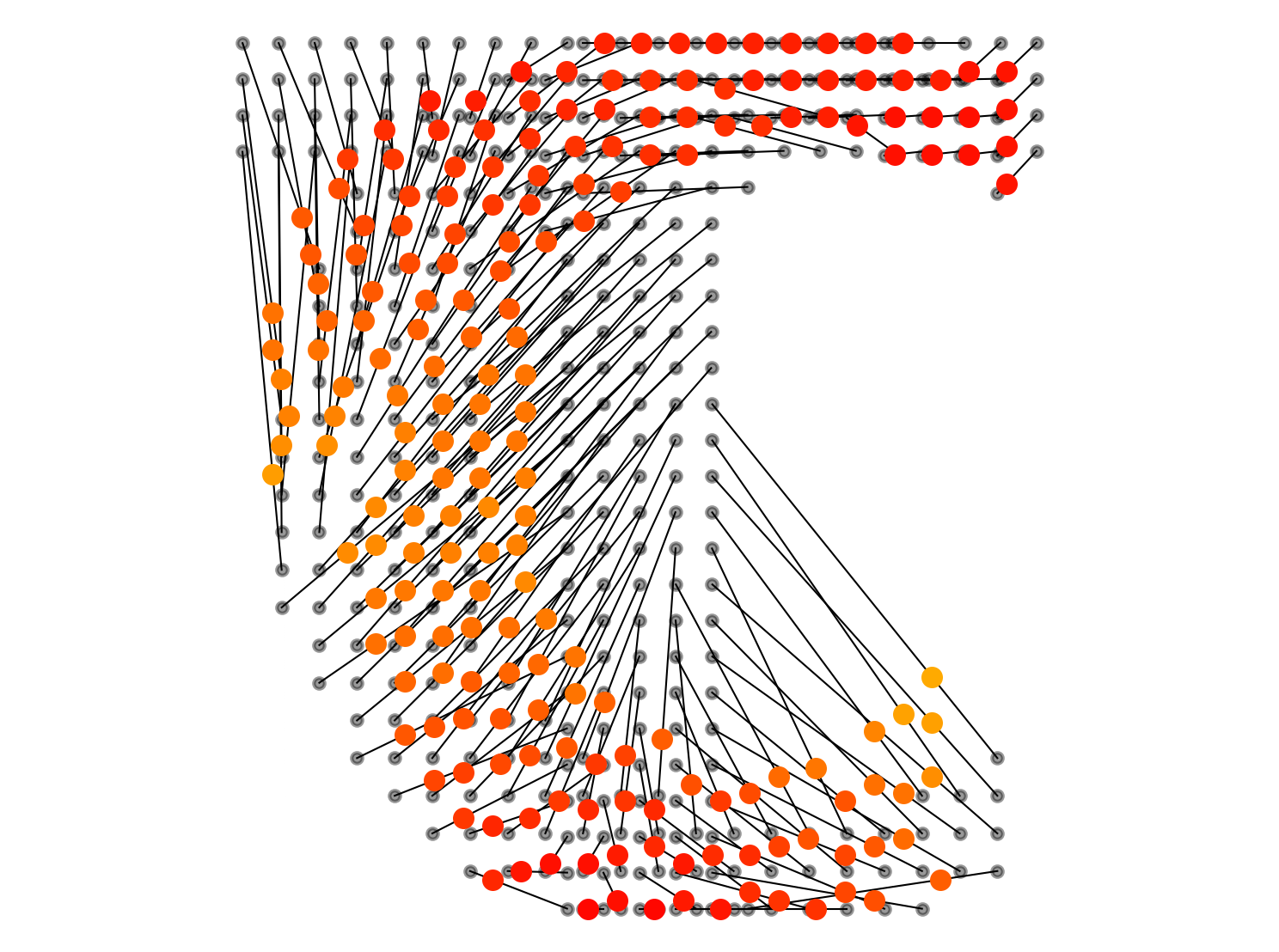};

      \nextgroupplot[
        xlabel={$T = 23.8\,s$},
        ]
      \addplot[thick, color=blue, on layer=axis background]
      graphics[xmin=-80,ymin=-60,xmax=80,ymax=60] {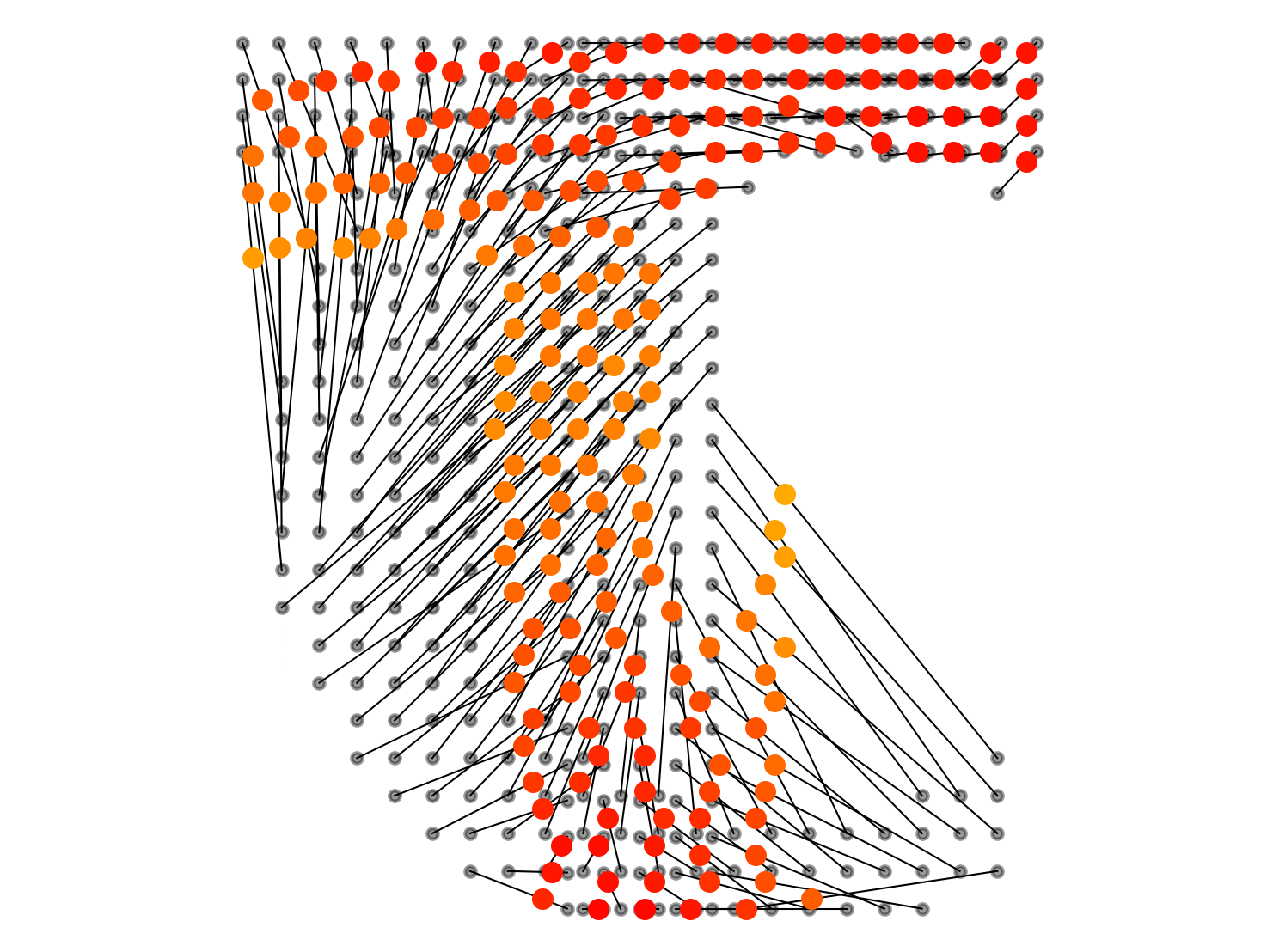};

      \nextgroupplot[
        xlabel={$T = 30.2\,s$},
        ]
      \addplot[thick, color=blue, on layer=axis background]
      graphics[xmin=-80,ymin=-60,xmax=80,ymax=60] {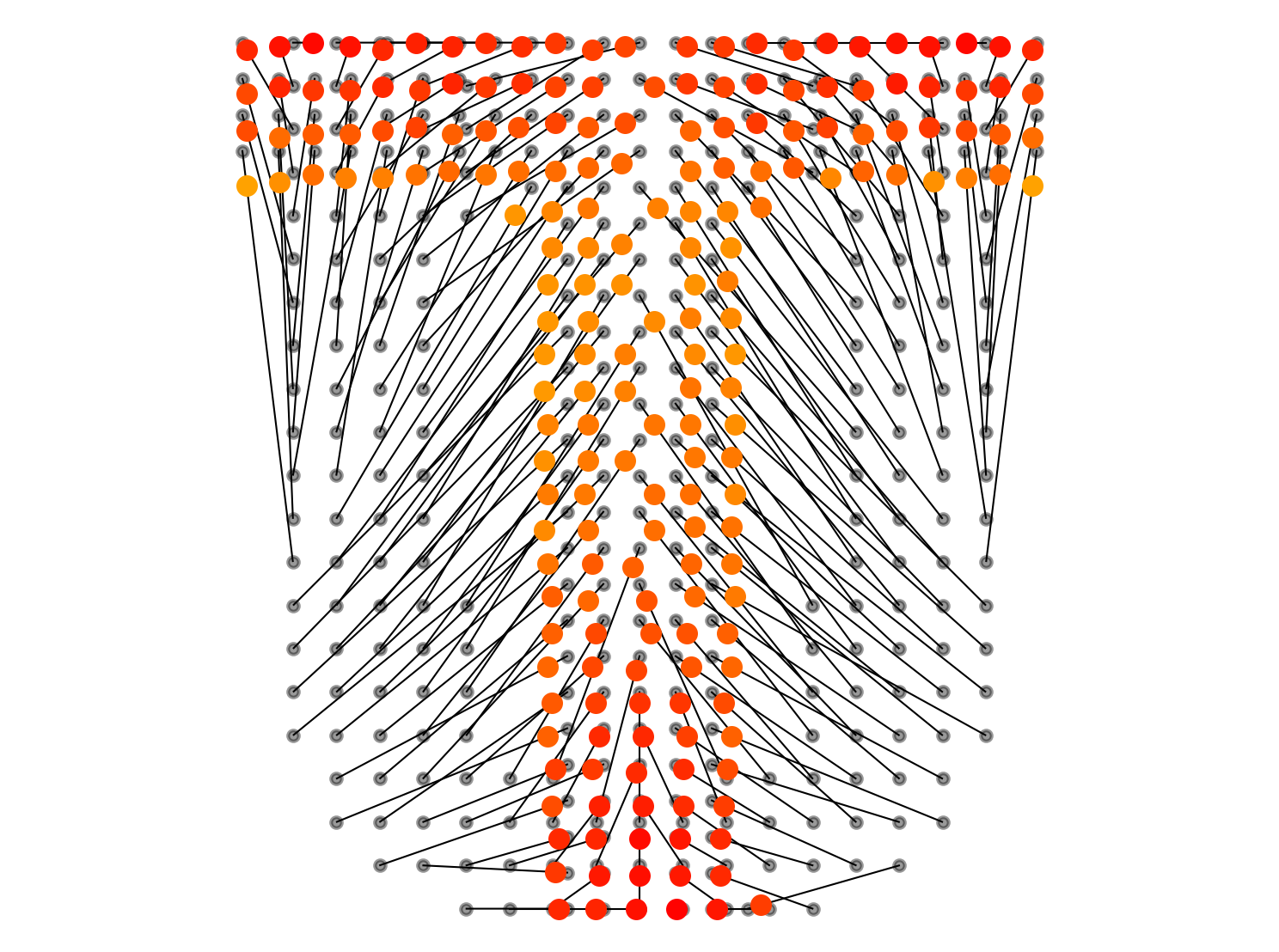};

      \nextgroupplot[
        xlabel={$T = 36.6\,s$},
        ]
      \addplot[thick, color=blue, on layer=axis background]
      graphics[xmin=-80,ymin=-60,xmax=80,ymax=60] {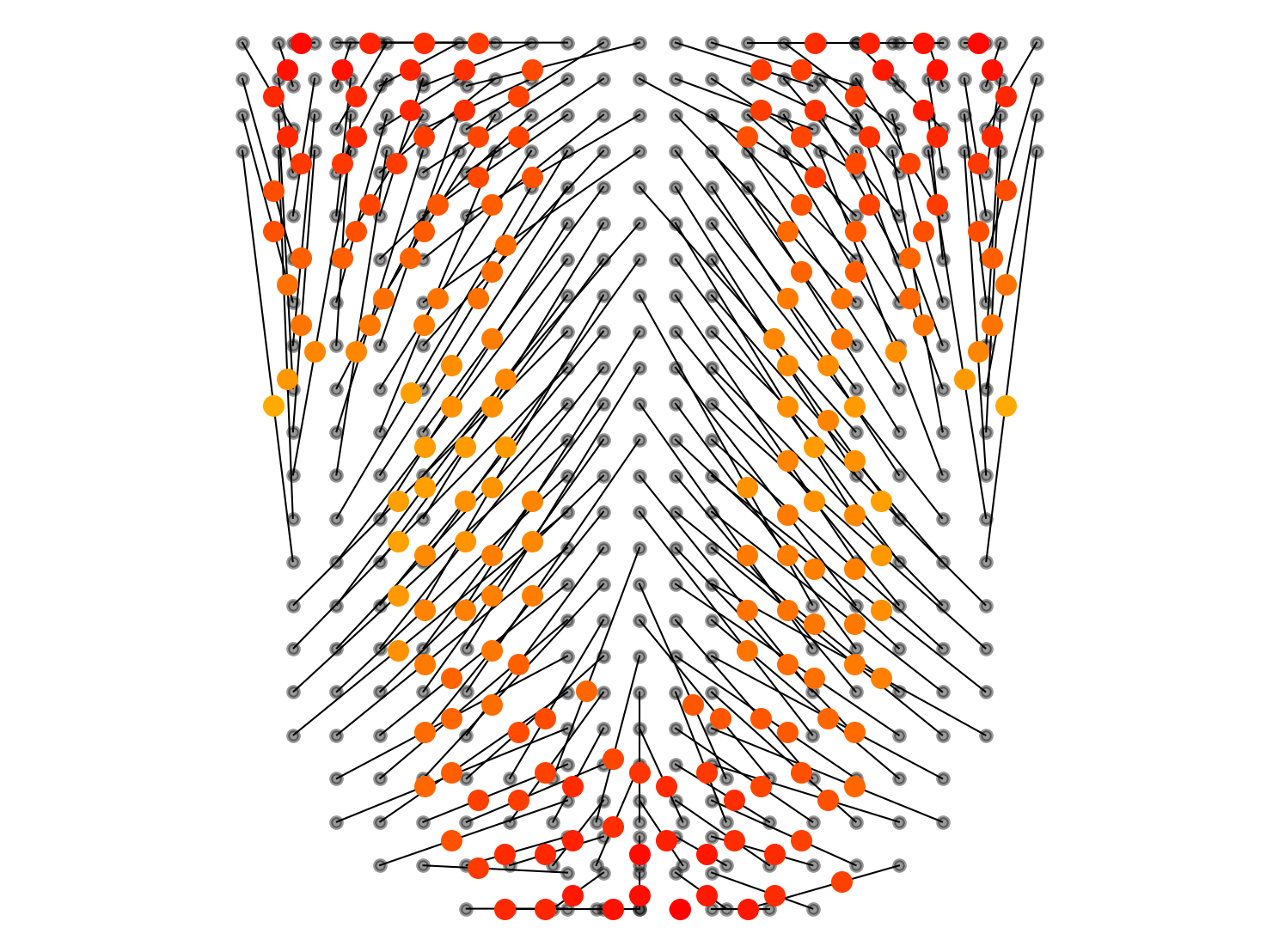};

    \end{groupplot}

  \end{tikzpicture}

%% file: fig/trapezoids/isosceles_trapezoid.tex
\begin{tikzpicture}[my angle/.style={font=\normalsize, draw, thick, angle eccentricity=1.42, angle radius=5mm}, line cap=round]
  \coordinate (si) at (0,0);
  \coordinate (gj) at (6.0, 0);
  \coordinate (gi) at ($(gj) + (110:2.2)$);
  \coordinate (sj) at (70:2.2);
  \coordinate (p1) at ($(si)!0.5!(gi)$);
  \coordinate (p2) at ($(sj)!0.5!(gj)$);
  % \coordinate (A+B) at ($(A)+(B)$);

  % \centerarc[black!30,thick, dashed](gj)(65:275:1.20)
  % \draw[dashed, thick, mypurple!80, opacity=0.7] (sj) -- (gj) node[midway,above] {$Md$};
  \draw[dashed, thick, black, opacity=0.7] (si) -- (gj) node[midway,above] {};
  \draw[dashed, thick, black, opacity=0.7] (sj) -- (gi) node[midway,above] {};
  \draw[dashed, thick, mygreen, opacity=0.7] (p1) -- (p2) node[midway,above] {};

  \pic [my angle, angle radius=10mm, "$\theta$"] {angle=gj--si--gi};
  \pic [my angle, "$\rho$"] {angle=si--sj--gi};
  % \pic [my angle, angle radius=12mm, "$\pi-\rho$"] {angle=gj--si--sj};
  % \pic [my angle, "$\gamma$"] {angle=gi--gj--si};

  \draw[myblue,fill=myblue] (si) circle (.3ex) node[midway,below, shift={(0.0, 0.0)}] {$\mathbf{s}_i\equiv A$};
  \draw[myblue,fill=myblue] (sj) circle (.3ex) node[midway,above, shift={(0.75, 2.14)}] {$\mathbf{s}_j \equiv D$};
  \draw[myred,fill=myred] (gi) circle (.3ex) node[midway,above, shift={(5.25, 2.1)}] {$\mathbf{g}_i \equiv C$};
  \draw[myred,fill=myred] (gj) circle (.3ex) node[midway,above, shift={(6.0, -0.6)}] {$\mathbf{g}_j \equiv B$};

  \draw[vector,arrow,mypurple] (sj) -- (gj) node[midway,above, shift={(-0.33, 0.2)}] {$E$};
  \draw[vector,arrow,mypurple] (si) -- (gi) node[midway,below] {};
  \draw[vector,arrow,myblue] (sj) -- (si) node[midway,above left] {$\mathbf{s}_{ji}$};
  \draw[vector,arrow,myred] (gj) -- (gi) node[midway,above right] {$\mathbf{g}_{ji}$};
  % \draw[vector,arrow,mypurple] (O) -- (A+B) node[above right=-3] {$\vb{a}+\vb{b}$};

  \draw[mygreen,fill=mygreen] (p1) circle (.3ex) node[midway,below, shift={(3.0, 1.1)}] {$d_{ij,min}$};
  \draw[mygreen,fill=mygreen] (p2) circle (.3ex) node[midway,above, shift={(0.0, -0.0)}] {};
  
  % \draw[vector,thin arrow,myblue!40] (A) -- (A+B) node[midway,below right=-2] {$\vb{b}$};
\end{tikzpicture}